\makeatletter
\@namedef{ver@everyshi.sty}{}
\makeatother
\documentclass[10pt,twocolumn,letterpaper]{article}

\usepackage[pagenumbers]{cvpr} 

\usepackage{graphicx}
\usepackage{amsmath}
\usepackage{amssymb}
\usepackage{booktabs}
\usepackage{xcolor}
\usepackage{amsthm}
\usepackage{url}
\usepackage{graphicx, subcaption}
\usepackage[export]{adjustbox}
\usepackage{pgffor}
\usepackage[accsupp]{axessibility}
\usepackage{titletoc}

\usepackage{amsmath,amsfonts,bm}

\def\eqref#1{equation~\ref{#1}}

\def\1{\bm{1}}

\def\vtheta{{\bm{\theta}}}

\def\mW{{\bm{W}}}

\DeclareMathAlphabet{\mathsfit}{\encodingdefault}{\sfdefault}{m}{sl}
\SetMathAlphabet{\mathsfit}{bold}{\encodingdefault}{\sfdefault}{bx}{n}

\newcommand{\R}{\mathbb{R}}

\usepackage[pagebackref,breaklinks,colorlinks]{hyperref}

\usepackage[capitalize]{cleveref}
\crefname{section}{Sec.}{Secs.}
\Crefname{section}{Section}{Sections}
\Crefname{table}{Table}{Tables}
\crefname{table}{Tab.}{Tabs.}
\crefname{theorem}{Theorem}{Theorems}
\crefname{theorem}{Thm.}{Thms.}

\usepackage[nomargin,inline,draft]{fixme}
\newcommand\inner[2]{\left\langle #1, #2 \right\rangle}
\newcommand{\bW}{{\bm{W}}}
\newtheorem{theorem}{Theorem}
\newtheorem*{theorem*}{Theorem}

\newtheorem{lemma}{Lemma}

\newtheorem*{example*}{Example}

\begin{document}

\title{A Structured Dictionary Perspective on Implicit Neural Representations}

\author{Gizem~Yüce\thanks{The first two authors contributed equally to this work. G. Yüce did this work during an internship at EPFL.}\\
  ETH Zurich\\
  \tt\small{gyuece@ethz.ch} \\
  \and
  Guillermo~Ortiz-Jim\'enez\footnotemark[1]
  \qquad Beril~Besbinar
  \qquad Pascal~Frossard \\
  École Polytechnique Fédérale de Lausanne (EPFL)\\
  \tt\small{\{name.surname\}@epfl.ch}
}

\maketitle

\everypar{\looseness=-1}

\begin{abstract}
Implicit neural representations (INRs) have recently emerged as a promising alternative to classical discretized representations of signals. Nevertheless, despite their practical success, we still do not understand how INRs represent signals. We propose a novel unified perspective to theoretically analyse INRs. Leveraging results from harmonic analysis and deep learning theory, we show that most INR families are analogous to structured signal dictionaries whose atoms are integer harmonics of the set of initial mapping frequencies. This structure allows INRs to express signals with an exponentially increasing frequency support using a number of parameters that only grows linearly with depth. We also explore the inductive bias of INRs exploiting recent results about the empirical neural tangent kernel (NTK). Specifically, we show that the eigenfunctions of the NTK can be seen as dictionary atoms whose inner product with the target signal determines the final performance of their reconstruction. In this regard, we reveal that meta-learning has a reshaping effect on the NTK analogous to dictionary learning, building dictionary atoms as a combination of the examples seen during meta-training. Our results permit to design and tune novel INR architectures, but can also be of interest for the wider deep learning theory community.
\end{abstract}

\section{Introduction}
\label{sec:intro}
Implicit neural representations (INRs) have recently emerged as a powerful alternative to classical, discretized, representations of multimedia signals \cite{chen2019learning, mescheder2019occupancy, park2019deepsdf, sitzmann2019scene, sitzmann2019siren, tancik2020fourfeat,  mildenhall2020nerf, dupont2021coin, tewari2021advances}. 
In contrast to traditional methods, INRs parameterize the continuous mapping between coordinates and signal values using neural networks. This allows for an efficient and compact representation of signals that can be easily integrated into modern differentiable learning pipelines.

\begin{figure}[!t]
    \centering
    \includegraphics[width=\columnwidth]{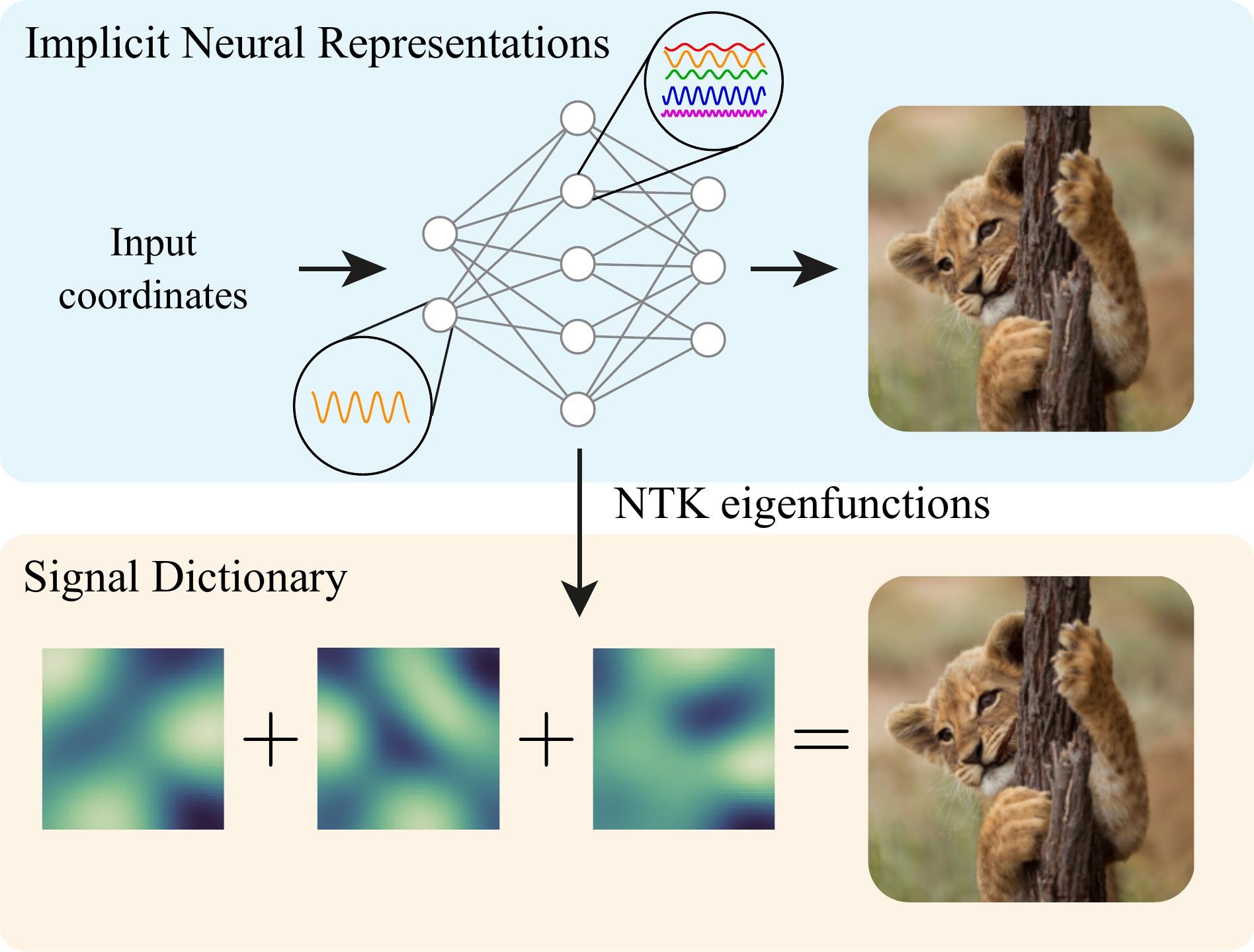}
    \caption{Conceptual illustration of our main theoretical contributions: i) Each layer of an INR increases the frequency support of the representation by splitting a signal into higher order harmonics. ii) INRs can be interpreted as signal dictionaries whose atoms are the eigenfunctions of their NTK at initialization.}
    \label{fig:teaser}
    \vspace{-1em}
\end{figure}

The recent success of INRs in many applications, such as surface representation \cite{sitzmann2019siren}, volume rendering \cite{mildenhall2020nerf, brualla2021nerf, srinivasan2021nerv, park2021nerfies} or generative modelling \cite{chan2021pi, dupont2021generative} can be largely attributed to the development of new periodic representations that can circumvent the spectral bias of standard neural networks. Indeed, there is ample evidence that the use of periodic representations~\cite{klocek2019hypernetwork, sitzmann2019scene, mildenhall2020nerf, tancik2020fourfeat} can mitigate the bias of standard architectures towards low frequency~\cite{rahaman2018spectral}. 

Nevertheless, even if INRs have become widely adopted in practice, the theoretical understanding of their principles and properties is rather limited. For example, there is no clear criterion to select between different INR families, their parameters are mostly based on heuristics, and their limitations are not well understood.
These shortcomings are slowing down further research developments. 
In this work, we therefore take a step back and focus on understanding the mechanisms behind the success of modern INRs, but also their failure modes, in order to develop more informed design strategies. We provide a unified perspective with the aim to answer the following important questions:
\begin{enumerate}
    \setlength{\itemsep}{0pt}
    \setlength{\parskip}{0pt}
    \setlength{\parsep}{0pt}
    \item \emph{What is the expressive power of INRs?}
    \item \emph{How does initialization affect their inductive bias?}
\end{enumerate}

Specifically, we first leverage results from harmonic analysis and deep learning theory, and we discover that the expressive power of most INRs is equivalent to that of a structured signal dictionary whose atoms are integer harmonics of the frequencies that define their initial input mapping (see \cref{fig:teaser}). This unifies many INR architectures under a single perspective, and can serve to understand them better and mitigate some of their common problems.

Then, we delve deeper on the inductive bias of INRs. We build upon the foundational work in~\cite{tancik2020fourfeat}, and exploit recent results in deep learning theory to give a new unifying framework to analyse the inductive bias of any INR architecture in terms of its empirical neural tangent kernel (NTK)~\cite{jacot2018neural}. In particular, we reveal the existence of a close analogy between the eigenfunctions of the empirical NTK and the atoms of a signal dictionary, and show that the difficulty of learning a signal with an INR is intimately connected to how efficiently it can be encoded by this dictionary.

Finally, we use our novel perspective to explain the role of meta-learning in improving the performance of INRs. INRs are known to be notoriously inefficient, requiring long training times, and a large sample exposure to achieve good results, especially in 3D settings~\cite{reiser2021kilonerf, Garbin21arxiv_FastNeRF, hedman2021baking}. However, recent works have shown that using meta-learning algorithms to initialize INRs can greatly improve their speed of convergence and sample complexity~\cite{tancik2020meta, sitzmann2020metasdf}. In this work, we show that meta-learning works as a dictionary learning algorithm, transforming the NTK of an INR into a rich signal dictionary whose atoms are formed by combinations of the examples seen during meta-training. This increases the representation efficiency of the target signals by the NTK~\cite{ortiz2021can}, thus improving performance and training speed.

In summary, the main contributions of our work are:
\begin{itemize}
    \setlength{\itemsep}{1pt}
    \setlength{\parskip}{1pt}
    \setlength{\parsep}{1pt}
    \item We provide a unified perspective to theoretically analyze the expressive power and inductive bias of INRs.
    \item We show that the frequency support of INRs grows exponentially with depth, as each layer splits its input into higher order harmonics, demonstrating their efficiency in representing wide spectrum signals. 
    \item We use this theory to explain the main failure modes of INRs: imperfect recovery and aliasing.
    \item We show that the inductive bias of INRs can be characterized by the ability of their empirical NTKs to encode different target signals efficiently.
    \item Finally, we discover that meta-learning greatly increases the encoding efficiency of the NTK by constructing a rich signal dictionary using different combinations of  the meta-training tasks.
\end{itemize}

Overall, we believe that our findings can impact the future research in INRs and their applications, and contribute to speeding up the development of new principled algorithms in the field. It gives a fresh perspective to understand and alleviate the drawbacks of the current architectures, as well as new intuitions to design better INR algorithms.
Finally, our analysis on the effect of meta-learning can also be of broader interest for the deep learning theory community\footnote{Code to reproduce this work: \href{https://github.com/gortizji/inr_dictionaries}{github.com/gortizji/inr\_dictionaries}}.

\section{Implicit Neural Representations}
\label{sec:INR}

The goal of an implicit neural representation is to encode a continuous target signal $g:\R^D\to\R^C$  using a neural network $f_{\bm \theta}:\R^D\to\R^C$, parameterized by a set of weights $\bm \theta\in\R^N$,
by representing the mapping between input coordinates $\bm r\in\R^D$, e.g., pixels, and signal values $g(\bm r)\in\R^C$, e.g., RGB colors. This is achieved minimizing a distortion measure, like mean-squared error, during training using some form of (stochastic) gradient descent (SGD). 

The continuous parameterization of INRs allows to store signals at a constant memory cost regardless of the spatial resolution, which makes INRs standout for reconstructing high-dimensional signals, such as videos or 3D scenes. The main challenge for INRs, though, is to reconstruct high frequency details present in most multimedia signals, e.g., textures in images. Classical neural network architectures are well-known for their strong spectral bias towards lower frequencies~\cite{rahaman2018spectral}, and this has made them traditionally useless for implicit representation tasks. Recently, however, few works \cite{tancik2020fourfeat, sitzmann2019siren} have come up with different solutions to circumvent the spectral bias of neural networks, allowing faster convergence and greater fidelity of INRs.

In what follows, we provide an overview of the main solutions under a unified \textit{architecture} formulation. Specifically, we note that most INR architectures can be decomposed into a mapping function $\gamma:\R^D\to\R^T$ followed by a multilayer perceptron (MLP), with weights $\bm W^{(\ell)}\in\R^{F_{\ell-1}\times F_\ell}$, bias $\bm b^{(\ell)}\in\R^{F_\ell}$, and activation function $\rho^{(\ell)}:\R\to\R$, applied elementwise; at each layer $\ell=1,\dots, L-1$. That is, if we denote by $\bm z^{(\ell)}$ each layers post activation, most INR architectures compute
\vspace{-2pt}
\begin{align}
    \bm z^{(0)} &= \gamma(\bm r), \label{eq:INR} \\
    \bm z^{(\ell)} &= \rho^{(\ell)}\left(\bm W^{(\ell)}\bm z^{(\ell-1)}+\bm b^{(\ell)}\right),\; \ell=1,\dots,L-1 \nonumber \\
    f_{\bm\theta}(\bm r) &= \bm W^{(L)}\bm z^{(L-1)} + \bm b^{(L)}. \nonumber
\end{align}
We now examine the two most popular INR architectures:

\vspace{-8pt}
\paragraph{Fourier feature networks (FFNs)} In \cite{tancik2020fourfeat}, Tancik \etal proposed to use a Fourier mapping $\gamma (\bm r) = \sin(\bm\Omega \bm r+\bm\phi)$, with parameters $\bm \Omega \in \R^{T \times D}$ and $\bm \phi \in \R^T$ followed by an MLP with $\rho^{(\ell)}=\operatorname{ReLU}$. Specifically, they showed that initializing $\bm\Omega_{i,j}\sim\mathcal{N}(0,\sigma^2)$ with random Fourier features~\cite{rahimi2008random} can modulate the spectral bias of an FFN, with larger values of $\sigma$ biasing these networks towards higher frequencies. Alternative formulations with deterministic initialization, commonly used for neural rendering algorithms~\cite{mildenhall2020nerf} can be considered as a special category of these networks where the frequencies in $\bm \Omega$ are taken to be powers of 2 and the frequencies in $\bm \phi$ alternate between $\{0, \pi/2\}$.

\vspace{-8pt}
\paragraph{Sinusoidal representation networks (SIRENs)} In \cite{sitzmann2019siren}, Sitzmann \etal proposed to use MLP with sinusoidal activations, i.e., $\rho^{(\ell)}=\sin$, where the first layer post activation, $\bm z^{(0)} = \sin \left(\omega_0 (\bm W^{(0)} \bm r + \bm b^{(0)}) \right)$ can be interpreted as $\gamma(\bm r) = \sin(\bm\Omega \bm r+\bm\phi) $. They showed that, by rescaling the parameters at initialization by the constant factor $\omega_0$, SIRENs can also modulate the spectral bias, with larger $\omega_0$ biasing these networks towards higher frequencies.

Nonetheless, despite the ample empirical evidence that shows that these architectures are effective at representing natural images or other visual signals, there is little theoretical understanding of how they do so. Moreover, since the design of each of these networks is guided by very different principles, the sheer diversity in the structure of these architectures makes their analysis very involved. 

In the next sections, we provide a unified perspective to analyze the expressive power and inductive bias of INRs and show that all modern INRs are intrinsically guided by the same fundamental principles, which let them express a wide range of signals. However, it also makes them prone to the same type of failure modes. Our novel framework can be used to design new principled solutions to address these shortcomings, but also simplify the tuning of current INRs.

\section{Expressive Power of INRs}

We now provide an integrated analysis of the expressive power of INRs. To that end, we will follow the formulation in \cref{eq:INR}, where, to simplify our derivations, we will restrict ourselves to polynomial activation functions, i.e., non-linearities of the form $\rho(x) = \sum_{k=0}^K \alpha_k x^k$. Note that this is a very mild assumption, as all analytic activation functions, e.g., sinusoids, can be approximated using polynomials with a na\"ive Taylor expansion; and that even the non-differentiable ReLUs can be effectively approximated using Chebyshev polynomials~\cite{mehmeti-gopel2021ringing}. Note, also, that the sequence of coefficients of the polynomial expansion of most activation functions used in practice decays very rapidly~\cite{mehmeti-gopel2021ringing}.

Now, without loss of generality, let $D=1$ and consider what happens when a single-frequency mapping, i.e. $\gamma(r)=e^{j\omega r}$, goes through such a polynomial activation: The output of the activation consists of a linear combination of the integer harmonics of the input frequency, i.e.,
\vspace{-2pt}
\begin{equation}
    \rho\left(\gamma(r)\right)=\rho\left(e^{j\omega r}\right) = \sum_{k=0}^K \alpha_k e^{j k \omega r}.\label{eq:poly}
\end{equation}
\vspace{-2pt}
This harmonic expansion is precisely the mechanism that controls the frequency representation in INRs. More generally, the mapping $\gamma(\bm r)$ acts as a collection of single frequency basis, whose spectral support is expanded after each non-linear activation into a collection of higher order harmonics. This particular structure is shared among all FFNs and SIRENs and it gives rise to the following result regarding their expressive power, i.e. the class of functions that can be represented with these architectures. 

\begin{theorem}
\label{thm:expressive}
	Let $f_{\bm\theta}:\R^D\to\R$ be an INR of the form of \cref{eq:INR} with $\rho^{(\ell)}(z)=\sum_{k=0}^K\alpha_k z^k$ for $\ell>1$. Furthermore, let $\bm\Omega=[\bm\Omega_0,\dots,\bm \Omega_{T-1}]^\top\in\R^{T \times D}$ and $\bm\phi\in\R^T$ denote the matrix of frequencies and vector of phases, respectively, used to map the input coordinate $\bm r\in\R^D$ to $\gamma(\bm r)=\sin(\bm\Omega \bm r + \bm \phi)$. This architecture can only represent functions of the form
	\begin{equation}
		f_\vtheta(r) = \sum_{\omega'\in\mathcal{H}(\bm\Omega)}c_{\omega'}\sin{(\langle\bm\omega', \bm r\rangle + \phi_{w'})},\label{eq:expressive}
	\end{equation}
	where
	\begin{equation}
		\mathcal{H}(\bm\Omega)\subseteq\left\{\sum_{t=0}^{T-1} s_t \bm\Omega_t\; \Bigg{|}\; s_t\in\mathbb{Z} \wedge \sum_{t=0}^{T-1} |s_t| \leq K^{L-1} \right\}.
	\end{equation}
\end{theorem}
\begin{proof}
See Appendix.
\end{proof}

\cref{thm:expressive} shows that the expressive power of FFNs and SIRENs is restricted to functions that can be expressed as a linear combination of certain harmonics of the feature mapping $\gamma(\bm r)$. That is, these architectures have the same expressive power as a structured signal dictionary whose atoms are sinusoids with frequencies equal to sums and differences of the integer harmonics of the mapping frequencies\footnote{We will refer to these components as the harmonics of $\gamma(\bm r)$.}. Interestingly, an analogous result was also proven for the Multiplicative Filter Networks (MFNs)~\cite{fathony2021multiplicative}, a proof-of-concept architecture based on a multiplicative connection between layers instead of the usual compositional structure of MLPs. In particular, it can be shown that MFNs, although very different in structure, are also only able to express linear combinations of certain harmonics of their sinusoidal filters~\cite{fathony2021multiplicative}, which means that they have the same expressive power as FFNs and SIRENs.

Besides this unification, \cref{thm:expressive} also highlights that the way all these architectures encode different signals is very similar. Indeed, instead of representing a signal by directly learning the coefficients of the linear combination,  which would require to store $\mathcal{O}(TK^L)$ coefficients $c_{\omega'}$; the multilayer structure of all INRs imposes a certain low rank structure over the coefficients -- akin to the sparsity assumption in classical dictionaries~\cite{tovsic2011dictionary} -- which can greatly save on memory as it only requires to store $\mathcal{O}(T^2L)$ parameters. This is better understood through an illustrative example\footnote{Similar examples for other architectures can be found in the Appendix.}.

\begin{example*}
    Let $f_\vtheta$ be a three-layer SIREN defined as
    \begin{equation}
        f_\vtheta(r)={\bm w^{(2)}}^\top\sin\left(\mW^{(1)}\sin\left( \bm\Omega r \right) \right),
    \end{equation}
    where $\bm\Omega\in\R^{T}$, $\mW^{(1)}\in\R^{F \times T}$, and $\bm w^{(2)}\in\R^{F}$. The output of this network can equivalently be represented as
    \begin{equation}
    	f_\vtheta(r) = \sum_{m=0}^{F-1} \sum_{s_1, \dots, s_T=-\infty}^{\infty} c_{m,s_1,\dots,s_T}\sin{ \left( \sum_{t=0}^{T-1} s_t \omega_t r \right)},\label{eq:bessel}
    \end{equation}
    where
    \begin{equation}
    	c_{m,s_1,\dots,s_T} = \left(\prod_{t=0}^{T-1} J_{s_t}\left(W^{(1)}_{m,t}\right)\right) w^{(2)}_m,\label{eq:coeffs_siren}
    \end{equation}
    and $J_s$ denotes the Bessel function of first kind of order $s$.
\end{example*}
\begin{proof}
See Appendix
\end{proof}
As we can see, the harmonic expansion introduced by the nested sinusoids of this simple SIREN can be developed into a signal with a very large bandwidth. Indeed, the few coefficients of this network are enough to represent a signal supported by an infinite number of frequency harmonics. 

On the other hand, note that composing sinusoids is a common operation in communication theory as it defines the basis of frequency modulation (FM) technology~\cite{proakis_salehi_2014}. Interestingly, drawing analogies between FM signals and SIRENs is a good source of inspiration to intuitively understand how these networks modulate their spectral bias: Recall that for FM signals, such as $\sin(\beta\sin(\omega_0 r))$, the parameter $\beta$ controls the bandwidth of the modulation, which is generally limited by the decreasing nature of the Bessel coefficients $J_n(\beta)$ in $n$. Increasing $\beta$ has the effect of expanding the spectral support of the modulation, as the arguments of the Bessel functions increase. 

The analogous phenomenon can be observed in \cref{eq:bessel} for this simple SIREN, but can be extended to more general architectures. In general, we see that due to the decreasing nature of the Bessel functions $J_{s_t}(W^{(1)}_{m,t})$, the high order harmonics in \cref{eq:bessel} tend to have smaller weights than the lower ones. This specific parameterization acts as an implicit bias mechanism, which focuses most of the energy of the output signal in a narrow band around the input frequencies $\bm \Omega$. Nevertheless, we can also see that increasing the scale of the coefficients in the inner layer, e.g., $\bm W^{(1)}$, makes the coefficients of higher order terms in \cref{eq:coeffs_siren} larger, thus increasing the power of the higher order harmonics, and allowing the network to learn a wider range of frequencies.

The fact that all modern INRs encode information in a similar way can explain why all these architectures are as powerful, in practice. However, it may also explain why they all suffer from the same failure modes. In \cref{sec:failure_modes}, we study these in more detail.

\section{Failure modes of INRs}\label{sec:failure_modes}

We now move on to study of the main failure modes of INRs. In particular, we will see how the specific harmonic expansion from \cref{thm:expressive} can sometimes lead to very recognizable artifacts in the learned reconstructions. Specifically, imperfect signal recovery and aliasing.

\subsection{Imperfect recovery}
\label{sec:imperfect}

One of the main consequences of \cref{thm:expressive} is that the set of frequencies that define the base embedding $\gamma(\bm r)$ completely determines the frequency support of the reconstruction $f_\vtheta(r)$. In this sense, it is fundamental to guarantee that the set $\mathcal{H}(\bm\Omega)$ permits to properly cover the spectrum of $g(\bm r)$. When this is not the case, the reconstructed representations can exhibit severe artifacts in the spatial domain stemming from an incorrect choice of fundamental frequencies determined by the INR mapping in \cref{eq:INR}.

\def \figsizetwo{7cm}
\def \heightsizetwo{2.1cm}
\def \textboxsizefigtwodouble{4.2cm}
\def \textboxsizefigtwotight{2cm}

\begin{figure}
    \centering
    \rotatebox[origin=c]{90}{
        \parbox{\textboxsizefigtwotight}{
        \centering \footnotesize{Ground Truth \newline \vspace{2pt}}
        }
    }
    \begin{subfigure}{\figsizetwo}
        \raggedright
        \includegraphics[height=\heightsizetwo]{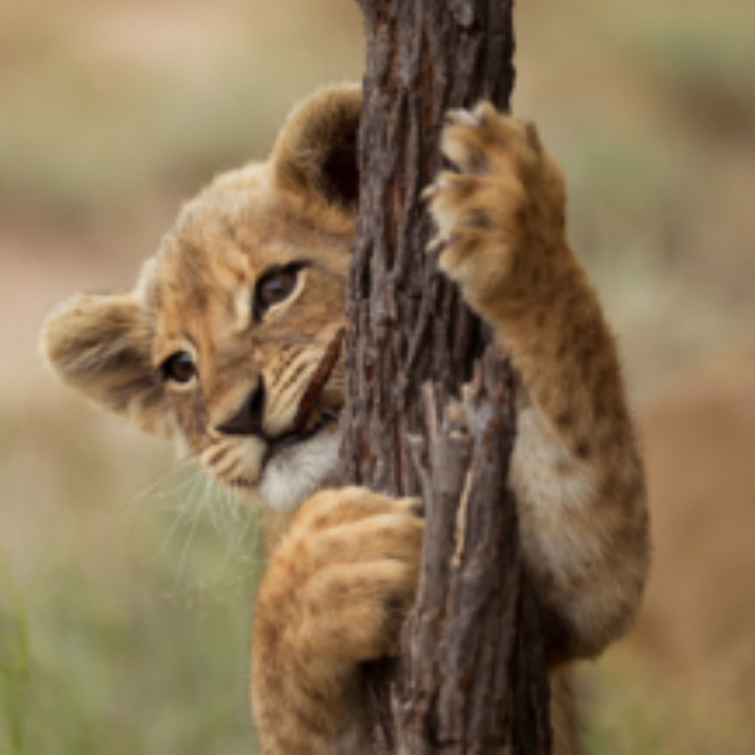}
        \includegraphics[height=\heightsizetwo]{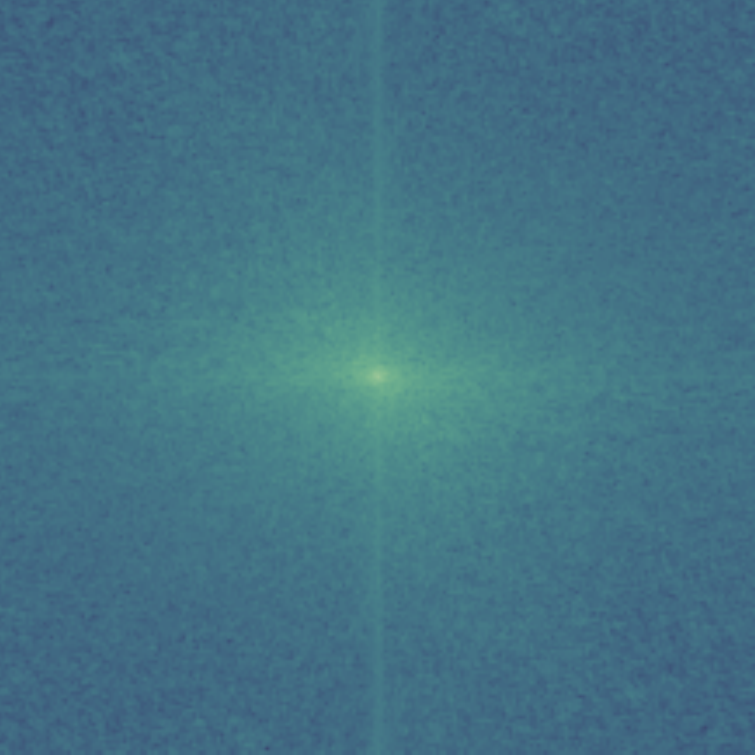}
    \end{subfigure}\\
    \vspace{-3pt}
    \rotatebox[origin=c]{90}{
        \parbox{\textboxsizefigtwodouble}{
        \centering 
        \footnotesize{Single frequency mapping \\
        \parbox{\textboxsizefigtwotight}{\centering\hspace{4pt}\footnotesize{($f_0=0.5$)}}
        \hfill
        \parbox{\textboxsizefigtwotight}{\centering \footnotesize{($f_0=1$)}}}}
    }
    \begin{subfigure}{\figsizetwo}
        \raggedright
        \includegraphics[height=\heightsizetwo]{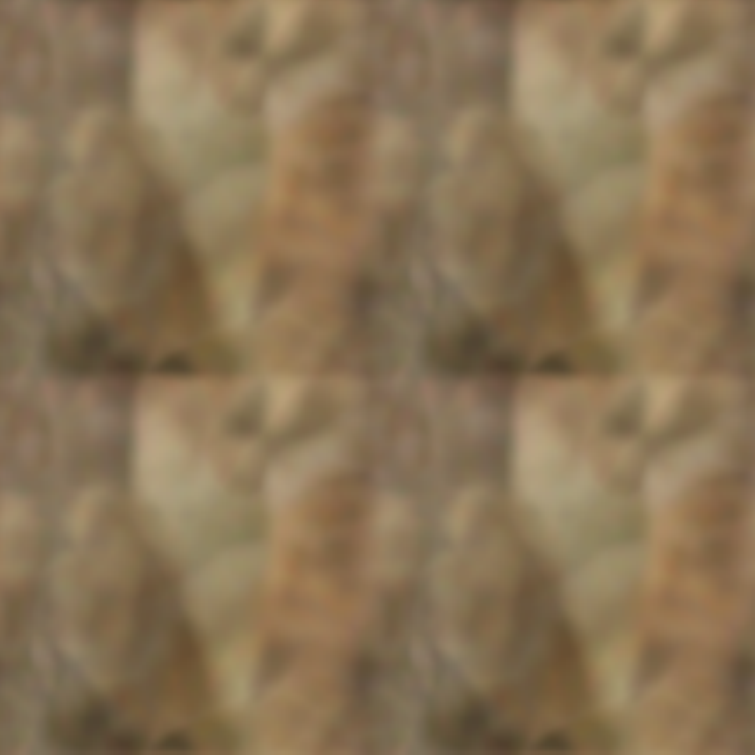}
        \includegraphics[height=\heightsizetwo]{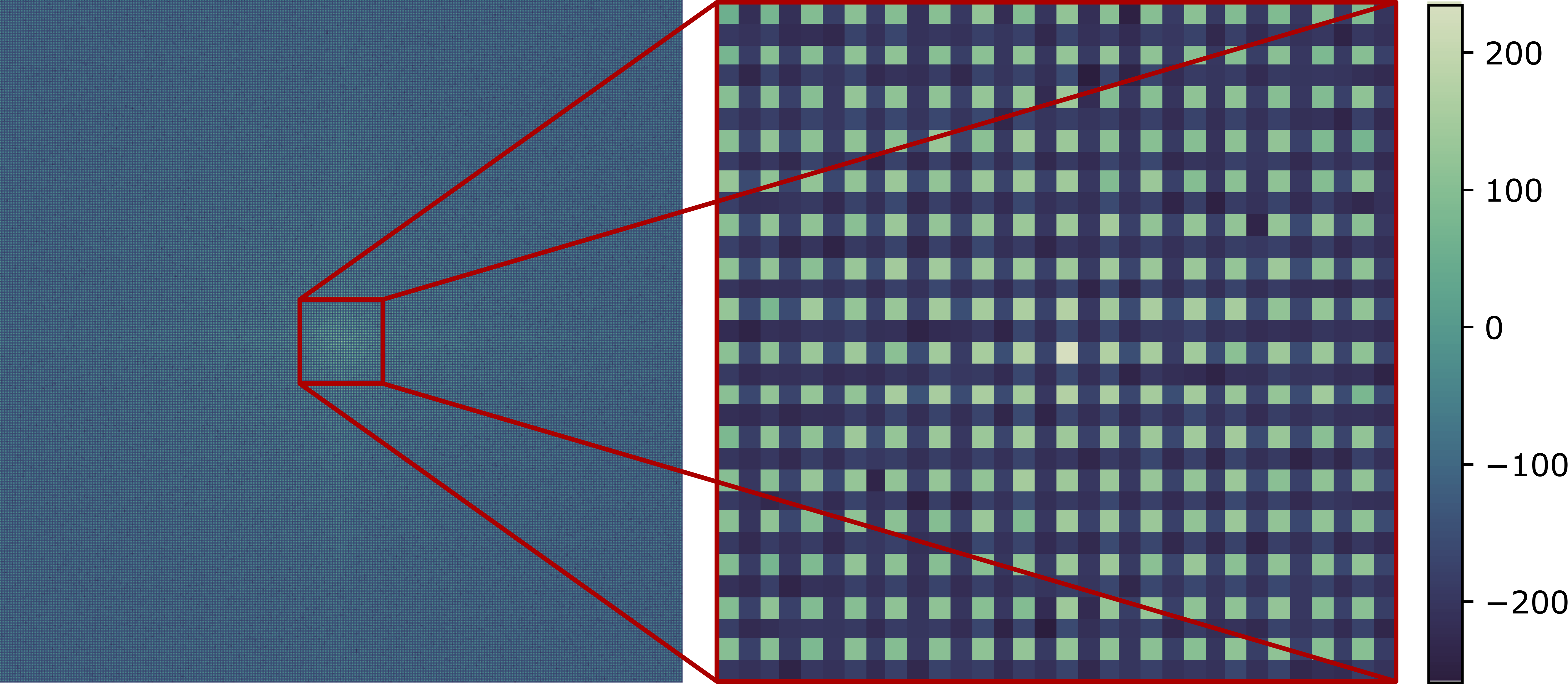}
        \\
        \includegraphics[height=\heightsizetwo]{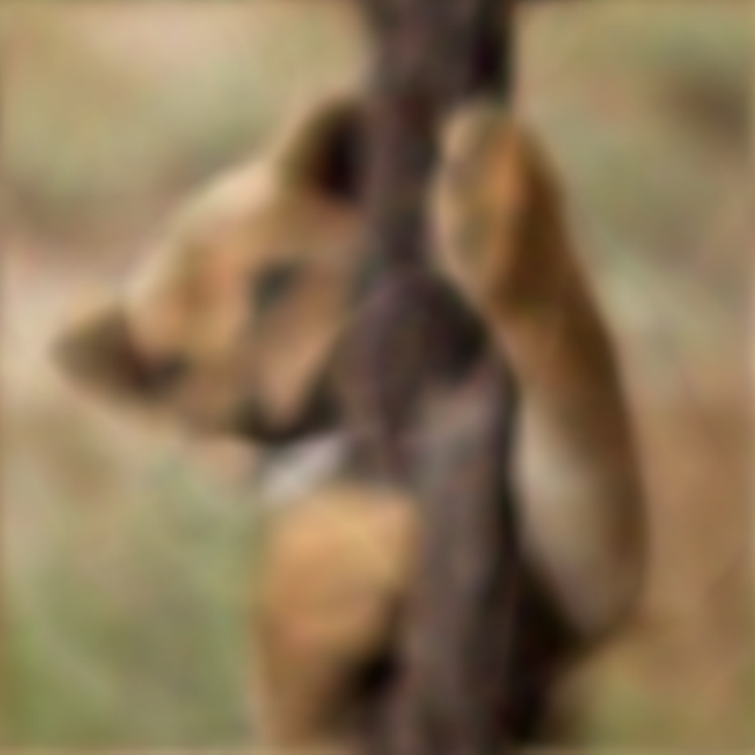}
        \includegraphics[height=\heightsizetwo]{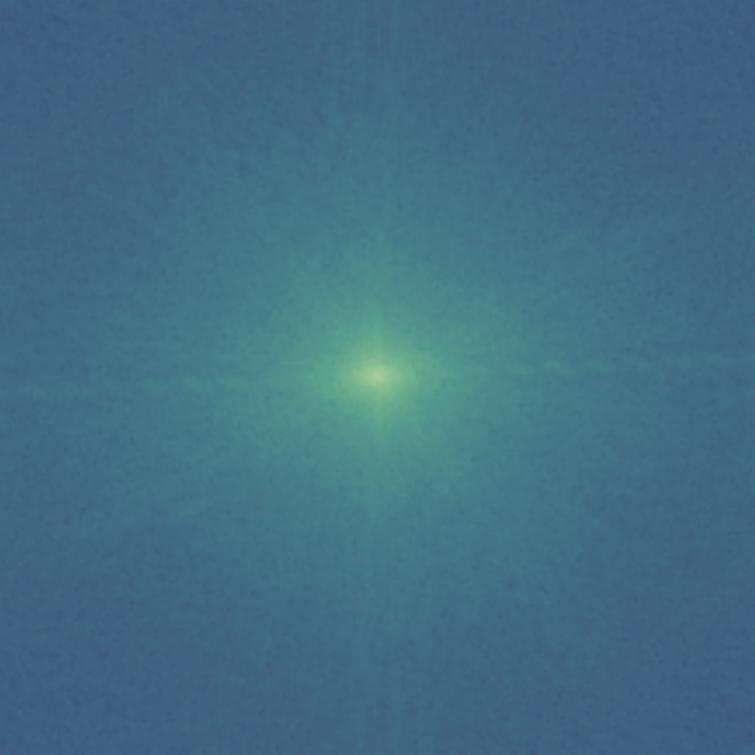}
    \end{subfigure}\\
    \vspace{-3pt}
    \rotatebox[origin=c]{90}{
        \parbox{\textboxsizefigtwotight}{
        \centering \footnotesize{FFN} \\\footnotesize{($\sigma=10$)}
        }
    }
    \begin{subfigure}{\figsizetwo}
        \raggedright
        \includegraphics[height=\heightsizetwo]{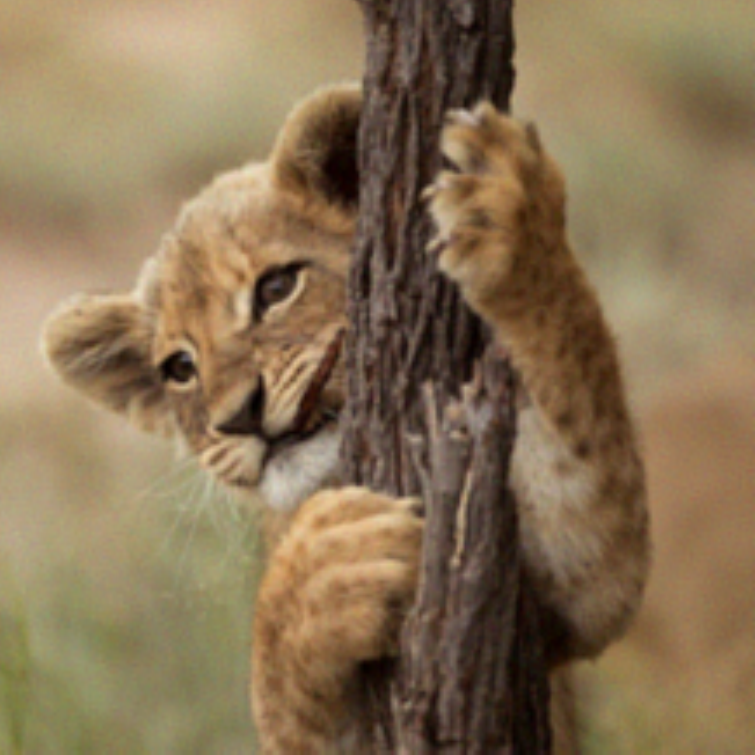}
        \includegraphics[height=\heightsizetwo]{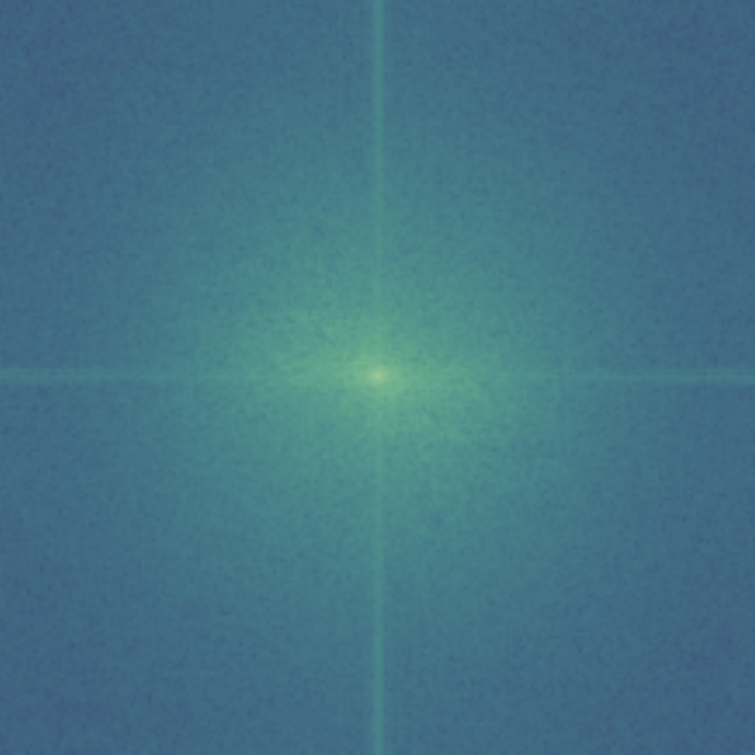} 
    \end{subfigure}\\
    \caption{\textbf{Left} Image reconstruction with different mappings of the input coordinates. \textbf{Right:} Magnitude of the DFT of the reconstruction. The \textbf{FFN} uses random Fourier encodings as defined in \cref{sec:INR}, and the \textbf{single frequency mappings} correspond to $\gamma(\bm r) = [\cos (2\pi f_0 \bm r), \sin (2\pi f_0 \bm r)]^T$.}
    \label{fig:aliasing_spatial}
    \vspace{-8pt}
\end{figure}

Let us illustrate this phenomenon for FFNs\footnote{The details of all our experiments can be found in the Appendix.}, but note that other types of architectures, such as SIRENs, also can suffer from spatial artifacts\footnote{We replicate our experiments for other networks in the Appendix.}. To that end, let us take the extreme case of an FFN $f_{\bm\theta}:\R^2\to\R^3$, with a deterministic single frequency Fourier encoding $\gamma(\bm r)=[\sin(2\pi f_0 \bm r), \cos(2\pi f_0 \bm r)]^\top$, reconstructing an image $f:[-1, 1]^2\to\R^3$, from samples in a grid of $N\times N$.

Now, note that, in light of \cref{thm:expressive}, this network can only represent signals with a frequency support in $\mathcal{H}(\bm\Omega)\subseteq \{2k\cdot\pi f_0 |k\in\mathbb{Z}\}$, i.e., containing only even multiples of $\pi f_0$. This means that if we choose $f_0=1$, the discrete Fourier transform (DFT) of the reconstruction will only have non-zero coefficients at frequencies corresponding to $2k\cdot 2\pi /N$, for $k=0,\dots, \lfloor (N-1)/2\rfloor$. This frequency covering is certainly not enough to completely represent images, as it misses all odd multiples of $2\pi/N$.

As shown in \cref{fig:aliasing_spatial}, reconstructing an image with such network produces severe artifacts. The learned representation with $f_0=1$ is highly distorted. That is, we see multiple displaced versions of the target image imposed over each other. The nature of this artifact is much more clear when we inspect the DFT of the reconstruction, which is supported on a perfect grid in the spectral domain, missing all the values of the spectrum at the odd coefficients.

Strikingly, setting $f_0=0.5$ is enough to completely get rid of this type of artifact. Indeed, when $f=0.5$ the set $\mathcal{H}(\bm\Omega)\subseteq \{\pi k|k\in\mathbb{Z}\}$, which means that the DFT of the reconstruction can have energy in all spectral coefficients. Nonetheless, we also observe that the resulting image is quite blurry. As we will see, this is due to the fast decay of the polynomial coefficients in \cref{eq:poly} for most activation functions, including ReLUs~\cite{mehmeti-gopel2021ringing}, which causes the weights of the high frequency harmonics in \cref{eq:expressive} to be very small. This phenomenon can be greatly alleviated, however, by increasing the frequency cover of the initial mapping $\gamma(\bm r)=\sin(\bm \Omega \bm r+\bm \phi)$ and sampling $\bm \Omega\in\R^{D\times T}$ using $\Omega_{i,j}\sim\mathcal{N}(0, \sigma^2)$. Indeed, using a large $T$ with a large $\sigma$ can reduce the probability of having a limited representation of the frequency spectrum of the target signal. Nevertheless, as we will see in \cref{sec:aliasing_freq}, setting $\sigma$ too large can introduce other problems.

\subsection{Aliasing}\label{sec:aliasing_freq}
It has been empirically shown that INRs with high fundamental frequencies in $\gamma(\bm r)$ converge faster, and achieve higher performances in the training set~\cite{sitzmann2019siren, tancik2020fourfeat}; even for targets with high frequency details. Nevertheless, it has also been reported that initializing these frequencies too high leads to poor performance outside the exact support of the training set, and produces aliasing artifacts~\cite{barron2021mip}. To the best of our knowledge, this behavior is still poorly understood.

\cref{thm:expressive} can, however, shed new light on this phenomenon. To that end, it is useful to see INRs as digital-to-analog converters (DAC), since INRs do little more than reconstruct a continuous signal from a set of discrete training samples. Classical sampling theory~\cite{oppenheim1999discrete} guarantees that one can reconstruct a bandlimited signal from its samples provided the sampling frequency is above the Nyquist rate. Nevertheless, it also states that without this prior knowledge, the problem of reconstructing a continuous signal from its samples is, in general, ill-posed -- there are many continuous functions that can lead to the exact same samples. Since INRs do not have an explicit knowledge of the bandwidth of the target, only their implicit bias can determine which of all these functions they reconstruct.

\def \subfigsizethree{4.1cm}
\def \textboxsizefigthreerotated{2cm}
\def \textboxsizefigthree{3.8cm}

\begin{figure}
\centering
\hspace{-6pt}
\rotatebox[origin=c]{90}{\makebox[\textboxsizefigthreerotated]{\centering \footnotesize{$w_0=300$} \vspace{-8pt}}}
\begin{subfigure}{\subfigsizethree}
    \centering 
    \includegraphics[width=\subfigsizethree]{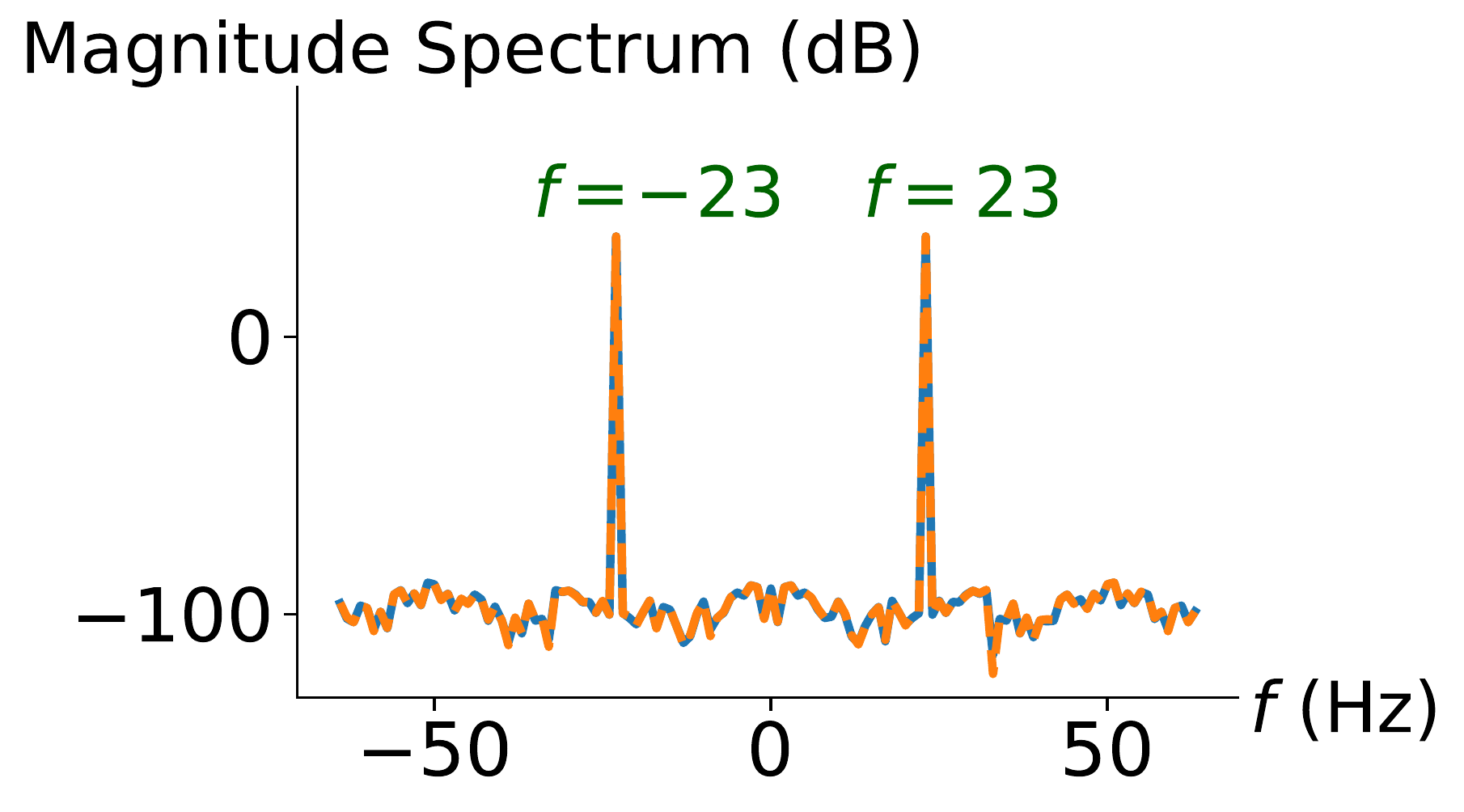}
\end{subfigure}
\hspace{-6pt}
\begin{subfigure}{\subfigsizethree}
    \centering
    \includegraphics[width=\subfigsizethree]{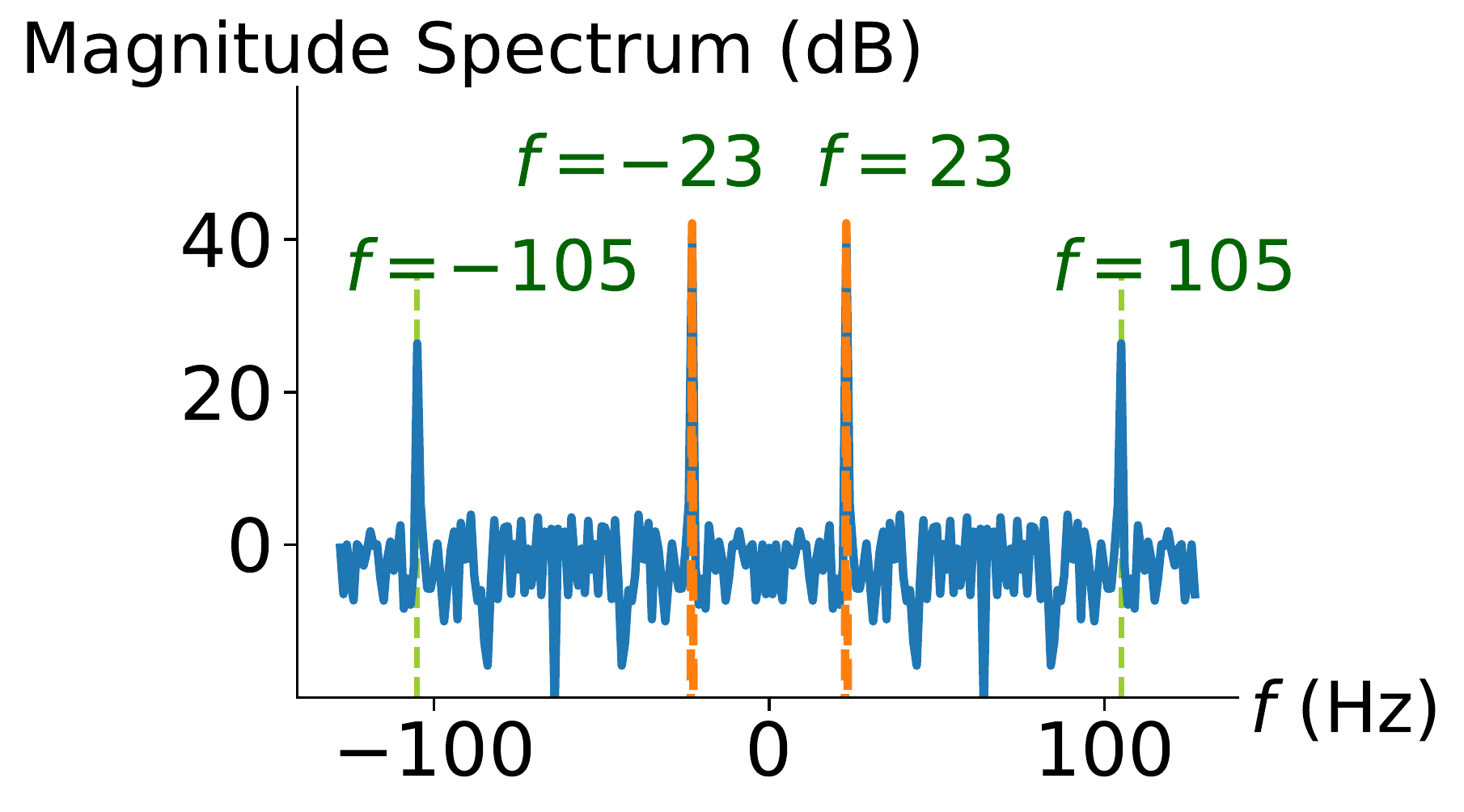}
\end{subfigure}\\
\hspace{-6pt}
\rotatebox[origin=c]{90}{\makebox[\textboxsizefigthreerotated]{\centering \footnotesize{$w_0=30$} \vspace{-8pt}}}
\begin{subfigure}{\subfigsizethree}
    \centering 
    \includegraphics[width=\subfigsizethree]{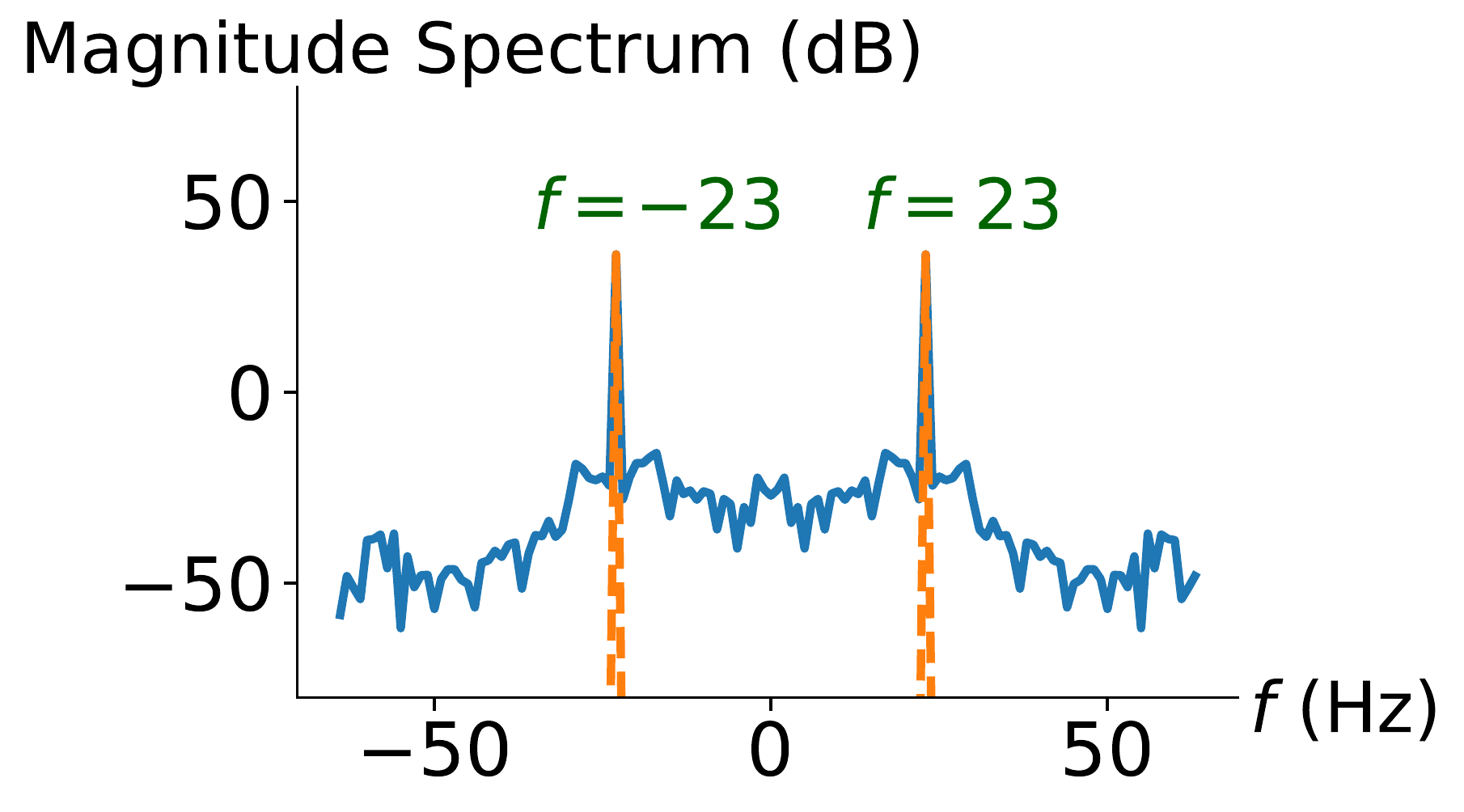}
\end{subfigure}
\hspace{-6pt}
\begin{subfigure}{\subfigsizethree}
    \centering
    \includegraphics[width=\subfigsizethree]{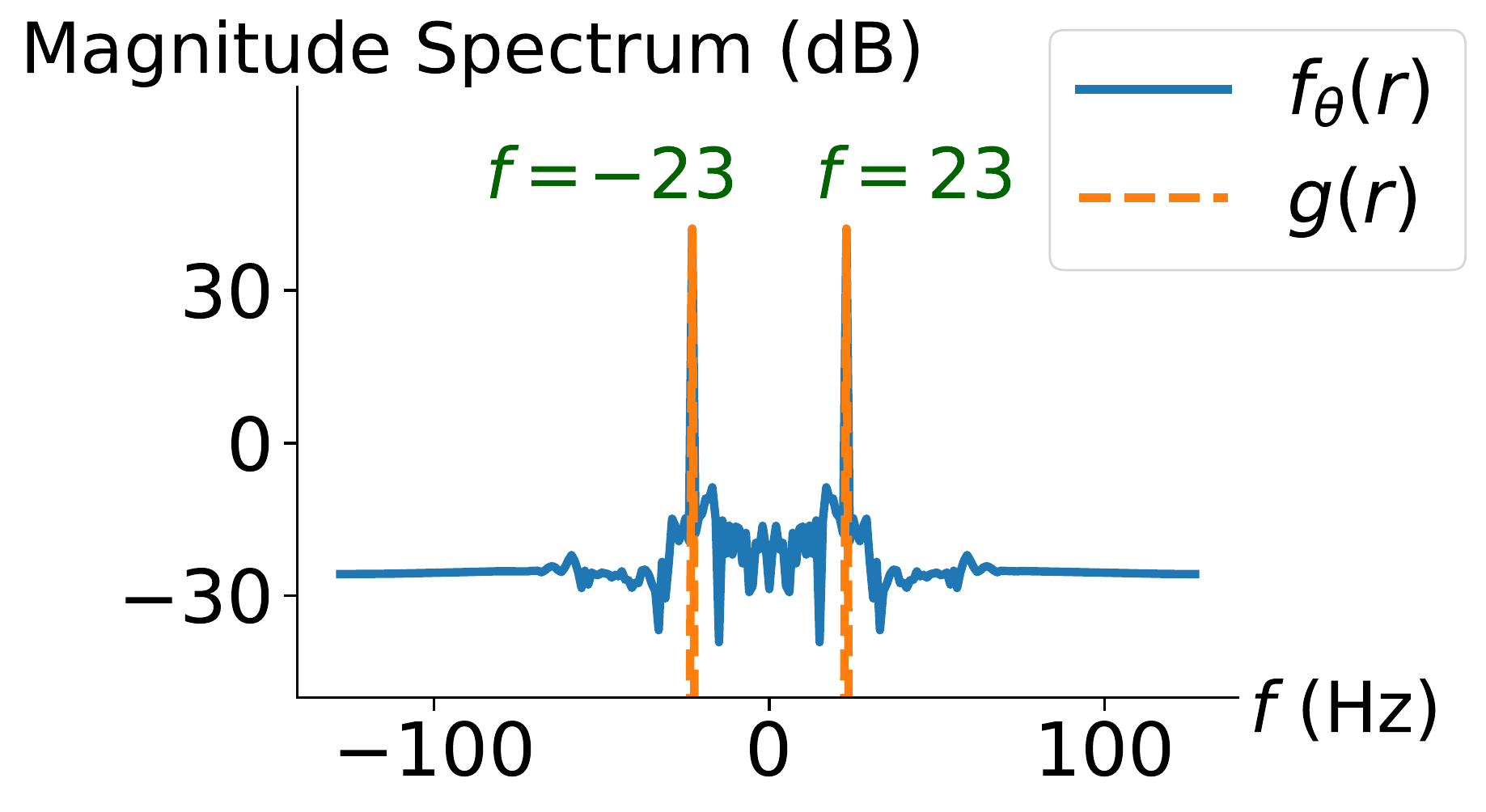}
\end{subfigure}
\\
\vspace{4pt}
\hspace{0.2cm}
\makebox[\textboxsizefigthree]{\centering\footnotesize{Sampling frequency $f_s = 128$}}
\hspace{0.1cm}
\makebox[\textboxsizefigthree]{\centering\footnotesize{Sampling frequency $f_s = 256$}}
\caption{Magnitude of the spectrum of $g(r)=\sin(2\pi\cdot 23r)$ and its SIREN reconstruction trained at $f_s = 128$ Hz. \textbf{Top row} shows $\omega_0=300$, and \textbf{bottom row} $\omega_0=30$. On the \textbf{left} the signals are sampled at $f_s=128$ Hz and on the \textbf{right} at $f_s=256$ Hz.}
\label{fig:freq_aliasing}
\vspace{-8pt}
\end{figure}

When the implicit bias does not match the nature of the signals, this can lead to reconstruction artifacts. Take for instance the problem of reconstructing a single-frequency signal $g(r)=\sin(2\pi\cdot23 r)$ using a SIREN ($\omega_0=300$ rad/s) trained on $128$ evenly spaced samples in the range $[0,1]$, i.e., sampled with a frequency of $f_s=128$ Hz. As we can see in \cref{fig:freq_aliasing}, the discrete-time Fourier transform of the reconstruction at the training points perfectly matches the target signal, i.e., the training loss is zero. Surprisingly, though, if one reconstructs the signal on a finer grid, e.g., $f_s=256$ Hz, which contains coordinates not seen during training, one can see that the spectrum of the reconstruction has an additional peak at $105$ Hz that is not present in the target signal. That is, the implicit bias of the network has ``chosen'' to reconstruct the signal using an aliased higher frequency component, as it had no way to discard this feasible solution. Interestingly, if one initializes the SIREN using $\omega_0=30$ rad/s, instead, this aliased copy disappears.

\cref{thm:expressive} gives the key to understand this behaviour. Specifically, note that most non-linearities used in INRs, e.g., ReLU or $\sin$, can be effectively approximated by polynomials of small order, or with rapidly decaying coefficients. As a result, even if the frequency support of the INRs can include harmonics of very high frequencies, theoretically, those components tend to be weighted with much smaller coefficients in practice. Increasing the value of the fundamental frequencies does help to include higher frequency components without relying in very high order harmonics. However, it does so, at the cost of introducing high frequency components with large weights in \cref{eq:expressive}, thus increasing the chances of yielding aliased reconstructions.

Reconstructing signals at low sampling rates makes the aliased high frequency components in \cref{eq:expressive} indistinguishable from lower frequency components. As we have seen this phenomenon stems from the underspecification~\cite{damour2020underspecification} of the reconstruction of the reconstruction problem in INRs, which can yield aliasing artifacts when testing at higher sampling rates. Solving this issues is crucial in application where a certain degree of generalization is required from the INRs. Applications such as super-resolution~\cite{chen2021learning, kim2021learning} or scene reconstruction~\cite{sitzmann2019siren} cannot rely on pure overfitting, and require INRs to generalize outside of their training support. Overall, we hope that our new insights can support the design of a new generation of INR architectures and algorithms that can mitigate this underspecification.

\section{Inductive bias of INRs}

All our results, so far, have only dealt with expressive power, i.e., the type of functions that can be represented by INRs. However, even if a network can express a signal, it does not mean that it can learn to represent it efficiently. MLPs, for instance, are widely known to be universal function approximators~\cite{cybenko1989approximation}, but still they have a hard time learning to high frequency functions~\cite{rahaman2018spectral}. To the best of our knowledge, the inductive bias of INRs is a largely unexplored topic. Besides the fact that INRs can circumvent the spectral bias~\cite{tancik2020fourfeat, sitzmann2019siren}, little is known of how different design choices influence the learnability of different signals.

In what follows, we will try to narrow this knowledge gap, as we will leverage recent results from deep learning theory to shed new light on the inductive bias of INRs, and how their initialization has a crucial role on what they learn.

\subsection{Overview of NTK theory}
Studying the inductive bias of deep learning is hard. This is mostly due to the non-linear nature of the mapping between parameters and functions specified by neural networks. Recent studies, however, have started arguing that studying learnability approximately is much more tractable. Notably, the neural tangent kernel (NTK) framework~\cite{jacot2018neural} proposes to approximate any neural network by its first order Taylor decomposition around the initialization $\bm \theta_0$, i.e.,
\begin{equation}
    f_{\bm\theta}(\bm r)\approx f_{\bm\theta_0}(\bm r)+(\bm\theta-\bm\theta_0)^\top\nabla_{\bm\theta}f_{\bm\theta_0}(\bm r), \label{eq:linearization}
\end{equation}
since using this approximation, the network is reduced to a simple linear predictor defined by the kernel
\begin{equation}
\bm\Theta(\bm r_1, \bm r_2)=\langle\nabla_{\bm\theta}f_{\bm\theta_0}(\bm r_1),\nabla_{\bm\theta}f_{\bm\theta_0}(\bm r_2)\rangle.\label{eq:ntk}
\end{equation}

Remarkably, while the understanding of deep learning is still in its infancy, the learning theory of kernels is much more developed~\cite{smola1998kernel}. Specifically, it can be shown that using the kernel in \cref{eq:ntk}, the sample complexity, and optimization difficulty, of learning a target function $g$ grows proportionally to its kernel norm~\cite{bartlettRademacherGaussianComplexities2001}, i.e.,

\begin{equation}
    \|g\|^2_{\bm\Theta}=\sum_{i=0}^\infty \cfrac{1}{\lambda_i}\left|\langle\phi_i,g\rangle\right|^2,\label{eq:kernel_norm}
\end{equation}
where $\langle \phi_i, g\rangle=\mathbb{E}_{\bm r}[\phi_i(\bm r)g(\bm r)]$, and $\{\lambda_i, \phi_i\}_{i=0}^\infty$ denote the eigenvalue, eigenfunction pairs of the kernel given by its Mercer's decomposition, i.e., $\bm\Theta(\bm r_1, \bm r_2)=\sum_{i=0}^\infty\lambda_i \phi_i(\bm r_1)\phi_i(\bm r_2)$. That is, those targets that are more concentrated in the span of the eigenfunctions associated with the largest eigenvalues of the kernel are easier to learn.

\cref{eq:linearization} holds with equality only if the neural network $f_{\bm\theta}$ is infinitely wide and has a specific structure~\cite{jacot2018neural, aroraCNTK2019}. For the finite-size neural networks used in practice, it only provides a rough approximation. Fortunately, recent studies have shown that even if finite-size neural networks and their kernel approximations do not have exactly the same dynamics, their sample complexity when learning a target $g$ scales in both cases with its kernel norm~\cite{ortiz2021can}, which makes \cref{eq:kernel_norm} a good proxy for learnability in deep learning. 

\begin{figure}[t]
    \centering
        \includegraphics[width=\linewidth]{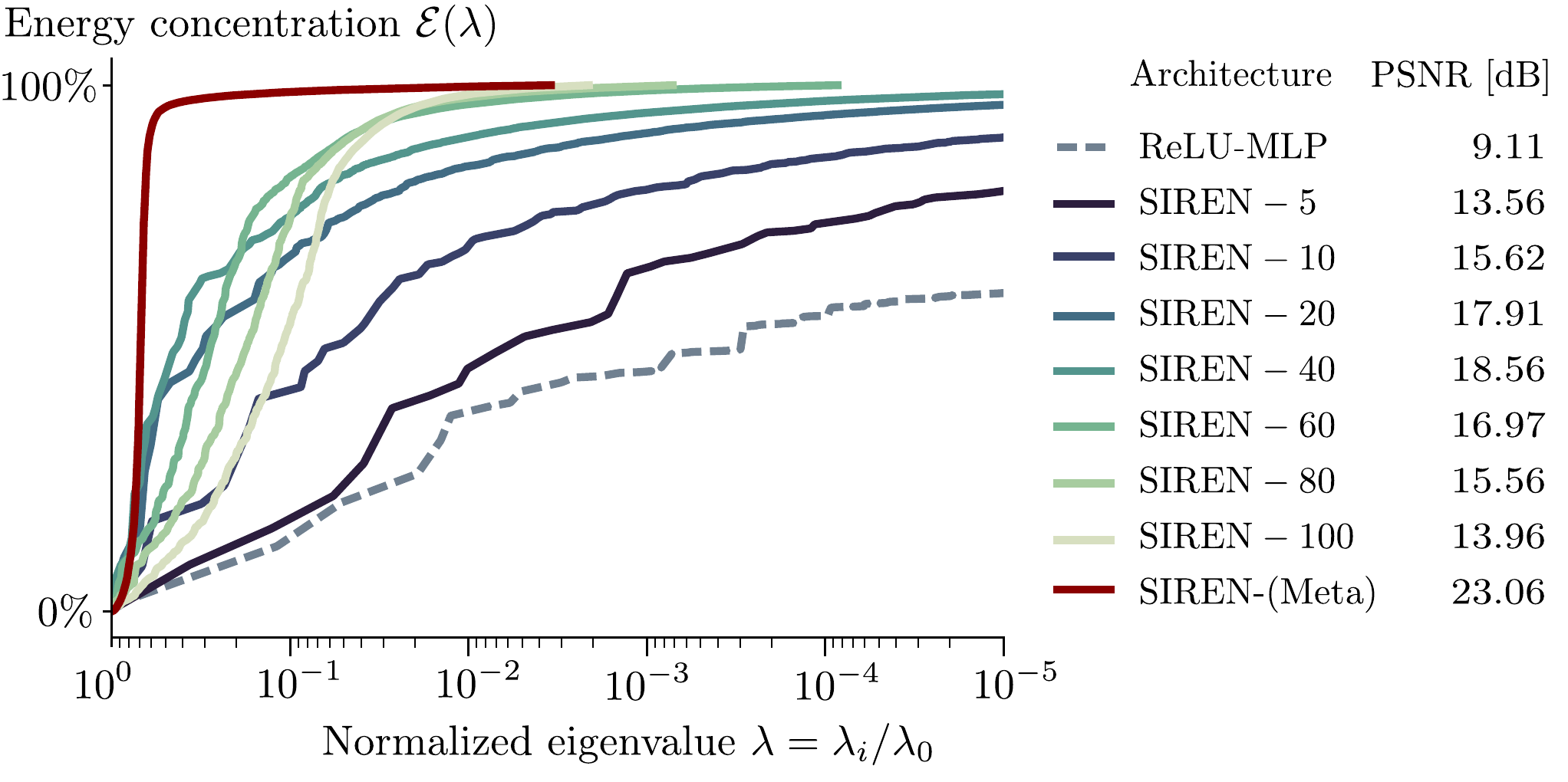}
    \vspace{-0.6cm}
    \caption{Average energy concentration of $100$ validation images from CelebA on subspaces spanned by the eigenfunctions of the empirical NTK associated to eigenvalues greater than a given threshold. Legend shows the average test PSNR after training to reconstruct those images from 50\% randomly selected pixels.}
    \label{fig:energy_vs_eigval}
    \vspace{-8pt}
\end{figure}
\subsection{NTK eigenfunctions as dictionary atoms}
\label{sec:ntk_init}

\def \subfigsizeone{1.4cm}
\def \textboxsizeone{1.4cm}
\def \textboxsizetwo{1.65cm}

\begin{figure*}[t!]
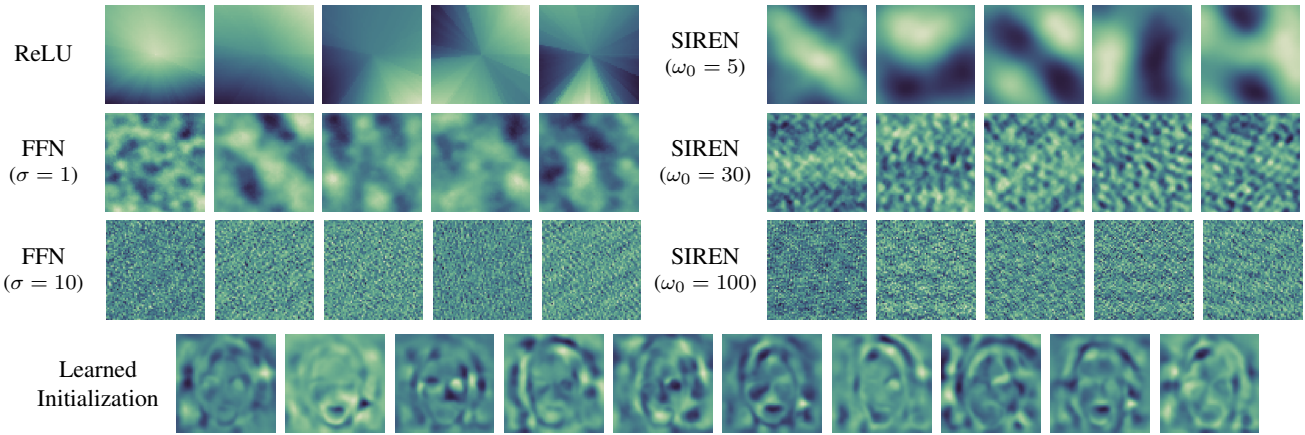

\begin{minipage}{0.5\textwidth}
    \centering
    \rotatebox[origin=c]{0}{\makebox[\textboxsizeone]{\small{ReLU}}}\hfil
    \foreach \i in {0,...,4} {
        \begin{subfigure}{\subfigsizeone}\includegraphics[width=\linewidth]{figures/MLPeigvec\i}
        \end{subfigure}
        \hspace{-6pt}
    }\\
    \rotatebox[origin=c]{0}{\parbox{\textboxsizeone}{\centering \small{FFN} \footnotesize{($\sigma=1$)}}}\hfil
    \foreach \i in {0,...,4} {
        \begin{subfigure}{\subfigsizeone}\includegraphics[width=\linewidth]{figures/FFrandom1eigvec\i}
        \end{subfigure}
        \hspace{-6pt}
    }\\
    \rotatebox[origin=c]{0}{\parbox{\textboxsizeone}{\centering \small{FFN} \footnotesize{($\sigma=10$)}}}\hfil
    \foreach \i in {0,...,4} {
        \begin{subfigure}{\subfigsizeone}\includegraphics[width=\linewidth]{figures/FFrandom10eigvec\i}
        \end{subfigure}
        \hspace{-6pt}
    }
    \vspace{2pt}
\end{minipage}
\begin{minipage}{0.5\textwidth}
    \centering
    \rotatebox[origin=c]{0}{\parbox{\textboxsizeone}{\centering \small{SIREN} \footnotesize{($\omega_0=5$)}}}\hfil
    \foreach \i in {0,...,4} {
        \begin{subfigure}{\subfigsizeone}\includegraphics[width=\linewidth]{figures/Siren_foz5eigvec\i}
        \end{subfigure}
        \hspace{-6pt}
    }\\
    \rotatebox[origin=c]{0}{\parbox{\textboxsizeone}{\centering \small{SIREN} \footnotesize{($\omega_0=30$)}}}\hfil
    \foreach \i in {0,...,4} {
        \begin{subfigure}{\subfigsizeone}\includegraphics[width=\linewidth]{figures/Siren_foz30eigvec\i}
        \end{subfigure}
        \hspace{-6pt}
    }\\
    \rotatebox[origin=c]{0}{\parbox{\textboxsizeone}{\centering \small{SIREN} \footnotesize{($\omega_0=100$)}}}\hfil
    \foreach \i in {0,...,4} {
        \begin{subfigure}{\subfigsizeone}\includegraphics[width=\linewidth]{figures/Siren_foz100eigvec\i}
        \end{subfigure}
        \hspace{-6pt}
    }
    \vspace{2pt}
\end{minipage}
\begin{minipage}{\textwidth}
\centering
\rotatebox[origin=c]{0}{\parbox{\textboxsizetwo}{\centering \small{Learned Initialization}}}
\foreach \i in {0,...,9} {
    \begin{subfigure}{\subfigsizeone}\includegraphics[width=\linewidth]{figures/meta_64_30_eigvec\i}
    \end{subfigure}
    \hspace{-6pt}
}
\end{minipage}
\caption{First eigenfunctions of the empirical NTK of different INRs at initialization. The first six architectures are initialized as described in \cref{sec:INR}. The \textbf{learned initialization} row shows the eigenfunctions of a SIREN initialized after meta-learning on $1,000$ training images from the CelebA dataset~\cite{liu2015faceattributes} following the procedure described in \cite{tancik2020meta}. Details of this experiment can be found in the Appendix.}
\label{fig:eigvec}
\vspace{-0.4cm}
\end{figure*}

The fact that the empirical NTK can approximately capture learnability in deep learning leads to a new interpretation of INRs: we can view INRs as signal dictionaries whose atoms are given by the eigenfunctions of the NTK at initialization. In this view, the study of the inductive bias of an INR is equivalent to the study of the representation capabilities of its NTK dictionary, in the sense that the functions that can be efficiently encoded by this dictionary are the ones that will be easier to learn.

The simplicity of this analogy allows us to investigate phenomena that appear complex otherwise. For example, we can use this perspective to constructively characterize the effect of the parameter $\omega_0$ in the inductive bias of a SIREN, and compare different networks, or initializations.
To that end, we measure the average energy concentration\footnote{Details of the experiments can be found in the Appendix.} of $N=100$ validation images $\{g_n\}_{n=1}^N$ from the CelebA dataset~\cite{liu2015faceattributes} on the span of the eigenfunctions of the NTK associated to eigenvalues greater than a given treshold, i.e.,
\begin{equation}
    \mathcal{E}(\lambda)=\cfrac{1}{N}\sum_{n=1}^N\sum_{\lambda_i/\lambda_0\geq \lambda}\cfrac{\left|\langle \phi_i, g_n \rangle\right|^2}{\left|\langle g_n, g_n \rangle\right|^2}.
\end{equation}
This metric is intimately connected to the kernel norm in \cref{eq:kernel_norm}, and it can give us a convenient perspective of the region of the NTK spectrum that will represent an image. The results of this procedure applied to different networks are shown in \cref{fig:energy_vs_eigval}. Remarkably, for very low values of $\omega_0$, most of the energy of these images is concentrated on the eigenfuctions corresponding to small eigenvalues. However, as we increase $\omega_0$, the energy concentration gets more skewed towards the eigenfunctions associated with large eigenvalues. Interestingly, after some point ($\omega_0>40$), the energy profile starts receding to the right, again.

Comparing the energy profiles with the generalization performance of these networks, we observe a clear pattern: the more energy is concentrated on the eigenfunctions associated with larger eigenvalues, the better the test peak signal-to-noise ratio (PSNR)\footnote{Correlations with other training metrics are shown in the Appendix.}. To understand this phenomenon, we can inspect the eigenfunctions of the NTK. As it is shown in \cref{fig:eigvec}, the eigenfunctions of the SIRENs with larger $\omega_0$ have higher frequency content. This means that increasing $\omega_0$ can have a positive effect in generalization as it yields a dictionary that better spans the medium-high frequency spectrum of natural images. Increasing $\omega_0$ too much, on the other hand, yields atoms with an overly high frequency content that cannot span the space of natural images efficiently, which explains their poor reconstruction performance of these networks.

Overall, we see how interpreting learnability as encoding efficiency of the NTK dictionary is a powerful analogy that can explain diverse phenomena, and lets us study under a single framework all sorts of INR questions, including those which might not be readily understood from \cref{thm:expressive}. This is a very powerful tool that we further exploit in \cref{sec:meta} to provide novel insights on the role of meta-learning in INRs.

\subsection{Meta-learning as dictionary learning}\label{sec:meta}

Prior work has shown that a correct initialization is key to ensure a good performance for INRs~\cite{tancik2020fourfeat, sitzmann2019siren}. In this sense, recent studies~\cite{tancik2020meta, sitzmann2020metasdf} have shown that the use of learned initialization, such as the ones obtained from meta-learning algorithms~\cite{finn2017model}, can significantly boost the performance of INRs. Indeed, initializing with meta-learned weights is one of the most effective remedies against the slow speed of convergence, and high sample complexity of INRs. However, while there has been recently great progress in understanding traditional forms of deep learning, the role of meta-learning on the inductive bias of deep neural networks remains largely overlooked. Interestingly, we now show how using the connections between INRs and signal dictionaries can help us understand meta-learning in general.

To do so, we follow the same experimental protocol as in \cref{sec:ntk_init}, where instead of computing the eigenfunctions of the NTK at a random initialization point, we linearize the INRs using \cref{eq:linearization} at the meta-learned weights, after pre-training on $1,000$ training images from CelebA using model agnostic meta-learning (MAML)~\cite{finn2017model, tancik2020meta}.

As it is shown in \cref{fig:energy_vs_eigval}, the meta-learned weights yield an eigenstructure that concentrates most of the energy of the target images on a subspace spanned by the eigenfunctions of the NTK with the largest eigenvalues, with almost no energy concentrated on the eigenfunctions corresponding to smaller eigenvalues. Therefore, training this INR starting from the meta-learned weights, results in a very fast speed of convergence and superior generalization capacity.

As it happened with the role of $\omega_0$ in \cref{sec:ntk_init}, visually inspecting the eigenfunctions of the NTK can help to build an intuition around this phenomenon. In this regard, recall that the CelebA dataset consists of a collection of face images. Strikingly, as illustrated in \cref{fig:eigvec}, the first eigenfunctions of the meta-learned NTK also look like faces. Clearly, meta-learning has reshaped the NTK so that the eigenfunctions have a large correlation with the target images. 

To the best of our knowledge, we are the first to report the NTK reshaping behavior of meta-learning, which cannot be obviously explained by first order approximation theories (cf. \cref{eq:linearization}). This result is remarkable for deep learning theory, as it helps us undertand the high-order dynamics of the NTK during training, which remains one of the main open questions of the field. Prior work had observed that standard training procedures change the first few eigenfunctions of the NTK so that they look like the target task~\cite{kopitkov2020neural, paccolat2021geometric, baratinNeuralAlignment2021, ortiz2021can}, but our observations in \cref{fig:energy_vs_eigval} and \cref{fig:eigvec} go one step further, and show that meta-learning has the potential to reshape a much larger space of the NTK dictionary by combining many tasks together, thus increasing the capacity of the NTK to efficiently encode a full meta-distribution of signals\footnote{In the Appendix we provide a more detailed experimental discussion.}. In this sense, we believe that that drawing parallels between classical dictionary learning algorithms~\cite{tovsic2011dictionary} and meta-learning can be a strong abstraction which can simplify the complexity of this problem, thus leading to a promising avenue for future research. Delving deeper in this connection will not only improve our understanding of meta-learning as a whole, but it can also provide new insights for the design of more efficient INRs by leveraging data to construct richer dictionaries.

\section{Related work}
INRs are a very active research field in computer vision, as they have become integral parts of many applications such as volume reconstruction~\cite{park2019deepsdf, mescheder2019occupancy}, scene rendering~\cite{sitzmann2019scene, niemeyer2020differentiable, mildenhall2020nerf}, texture synthesis~\cite{oechsle2019texture, henzler2020learning}, generative modelling~\cite{chen2019learning, chan2021pi, niemeyer2021giraffe}, or compression~\cite{dupont2021coin}. Recent architectural advances have focused mostly on improving the inference and training cost~\cite{Liu20neurips_sparse_nerf, park2021hypernerf, neff2021donerf, tancik2020meta, deng2021depth, sitzmann2020metasdf} of INRs, as well as on mitigating aliasing and improving generalization~\cite{barron2021mip, mehta2021modulated}.

The theory behind INRs has attracted much less attention, however. Similar to our work, Fathony \etal studied the expressive power of INRs, but their results only apply to their proposed multiplicative filter network architecture~\cite{fathony2021multiplicative}. Zheng \etal~\cite{zheng2021rethinking}, on the other hand, studied the trade-off between the rank and distance-preserving properties of different activation functions on INRs. Most notably, however, Tancik \etal~\cite{tancik2020fourfeat} showed that precoding the input of an infinitely wide ReLU-network with random Fourier features~\cite{rahimi2008random} is equivalent to using a tunable shift-invariant kernel method. This gives a static intuition of how randomly initialized FFNs circumvent the spectral bias~\cite{rahaman2018spectral}. Our work goes one step further, and builds upon recent empirical results~\cite{ortiz2021can} to extend this NTK analysis to finite networks with arbitrary weights and activations, e.g., meta-learned SIRENs. This allows us to investigate dynamical aspects of INRs such as the role of pre-training.


Interestingly, Kopitkov and Indelman~\cite{kopitkov2020neural} also used the visualization of the eigenfunctions of the NTK during training to understand other high-order training effects, such as the increase of alignment of the NTK with the target signal~\cite{kopitkov2020neural, paccolat2021geometric, baratinNeuralAlignment2021, ortiz2021can}. Our experiments use a similar approach to show the complex dictionary learning behaviour of MAML~\cite{finn2017model} in the NTK, which to the best of our knowledge is the first time this has been reported in the literature.

Connected to \cref{thm:expressive}, other works have also used a similar harmonic expansion to analyze certain effects in deep learning, such as the increase in roughness of the loss landscape with respect to the weights for deeper layers~\cite{mehmeti-gopel2021ringing}, or how skip-connections can avoid shattered gradients~\cite{balduzzi2017shattered}.

Finally, we note that most of our work draws inspirations from the classical signal processing literature~\cite{oppenheim1999discrete}. Some of our derivations are intimately connected to standard techniques in communications~\cite{proakis_salehi_2014}, and most of our analogies are founded on the field of signal representation~\cite{mallat1999wavelet} and dictionary design~\cite{tovsic2011dictionary}. Moving forward, delving deeper on these connections will be a fruitful avenue for future work.


\section{Conclusion}

In this paper, we have analyzed the expressive power and inductive bias of modern INRs from a unified perspective. We have shown that the expressive power of a large class of INRs with sinusoidal encodings is given by the space of linear combinations of the integer harmonics of their input mapping. This allows INRs to encode signals with an exponentially large frequency support using a few coefficients, but also cause them to suffer from imperfect signal recovery or aliasing. We have also seen that the inductive bias of INRs is captured by the ability of the empirical NTK to encode signals efficiently, and we have revealed that meta-learning can modify the NTK and increase this efficieny.

A natural future extension would be to generalize \cref{thm:expressive} to input mappings beyond sinusoids~\cite{barron2021mip, fathony2021multiplicative} or include normalization layers~\cite{ba2016layer}. Similarly, one could also study the effect of out-of-distribution data on the alignment with the NTK after meta-training.

Finally, it is important to note that our insights should be readily extensible to higher dimensional settings, although most of our practical results were performed using one or two-dimensional signals. In this sense, designing methods to visualize the eigenfunctions of the NTK in higher dimensions would clearly help to inform practitioners about the inductive bias of different INRs.

\section*{Acknowledgements}
We thank Alessandro Favero, Apostolos Modas, Seyed-Mohsen Moosavi-Dezfooli and Arun Venkitaraman for their fruitful discussions and feedback. This work has been partially supported by a GCP Research Credit Award.

\small
\bibliographystyle{ieee_fullname}
\bibliography{main}
\appendix
\onecolumn
\section*{Appendix}
\startcontents[sections]
\printcontents[sections]{l}{1}{\setcounter{tocdepth}{2}}
\clearpage

\section{Deferred proofs}
\subsection{Proof of Theorem 1}

We provide here the proof of \Cref{thm:expressive} which gives an explicit expression to the expressive power of INRs. However, before we delve deeper in this proof we will prove a few useful lemmas.

\subsubsection{Preliminary lemmas}

\begin{lemma}
Let $\{\bm\omega^{(1)}_k\in\R^D\}_{k\in\mathcal{K}}$ and $\{\bm\omega^{(2)}_j\in\R^D\}_{j\in\mathcal{J}}$, and $\{\phi^{(1)}_k\in\R\}_{k\in\mathcal{K}}$ and $\{\phi^{(2)}_j\in\R\}_{j\in\mathcal{J}}$ be two collections of frequency vectors and scalar phases, respectively, indexed by the sets $\mathcal{K},\mathcal{J}\subseteq \mathbb{N}$. Furthermore, let $\{\beta^{(1)}_k\in\R\}_{k\in\mathcal{K}}$ and $\{\beta^{(2)}_j\in\R\}_{j\in\mathcal{J}}$ be two sets of scalar coefficients and $\bm r\in\R^D$. Then,
\begin{align}
    \left(\sum_{k\in \mathcal{K}} \beta_{k}^{(1)} \cos\left(\inner{\bm\omega_{k}^{(1)}}{\bm r} + \phi_{k}^{(1)}\right)\right)
    \left(\sum_{j \in \mathcal{J}} \beta_{j}^{(2)} \cos\left(\inner{\bm\omega_{j}^{(2)}}{\bm r} + \phi_{j}^{(2)}\right)\right)
    =& \sum_{\bm\omega'\in\mathcal{D}} \Tilde{\beta}_{\bm\omega'} \cos(\inner{\bm\omega'}{\bm r} + \Tilde{\phi}_{\bm\omega'})
\end{align}
where 
\begin{equation}
    		\mathcal{D}\left(\left\{ \bm\omega_{k}^{(1)} \right\}_{k\in \mathcal{K}},\left\{ \bm\omega_{j}^{(2)} \right\}_{j \in \mathcal{J}}\right)=\left\{\bm\omega'= \bm\omega_{k}^{(1)} \pm \bm\omega_{j}^{(2)}\Bigg{|} k\in\mathcal{K}, j\in\mathcal{J}\right\}
\end{equation}
for some $\left\{ \Tilde{\phi}_{\bm\omega'}\in\R \,\Bigg{|}\,\bm\omega'\in\mathcal D\right\}$, $\left\{ \Tilde{\beta}_{\bm\omega'}\in\R\,\Bigg{|}\,\bm\omega'\in\mathcal D\right\}$.
\end{lemma}

\begin{proof}
\begin{align}
    &\left(\sum_{k\in \mathcal{K}} \beta_{k}^{(1)} \cos\left(\inner{\bm\omega_{k}^{(1)}}{\bm r} + \phi_{k}^{(1)}\right)\right)
    \left(\sum_{j \in \mathcal{J}} \beta_{j}^{(2)} \cos\left(\inner{\bm\omega_{j}^{(2)}}{\bm r} + \phi_{j}^{(2)}\right)\right) \notag \\
    =& \sum_{k\in \mathcal{K}} \sum_{j \in \mathcal{J}} \beta_{k}^{(1)} \beta_{j}^{(2)} \cos\left(\inner{\bm\omega_{k}^{(1)}}{\bm r} + \phi_{k}^{(1)}\right) \cos\left(\inner{\bm\omega_{j}^{(2)}}{\bm r} + \phi_{j}^{(2)}\right)  \notag \\
    =& \sum_{k\in \mathcal{K}} \sum_{j \in \mathcal{J}} \beta_{k}^{(1)} \beta_{j}^{(2)} \frac{1}{2} 
    \left( \cos\left( \inner{\bm\omega_{k}^{(1)}}{\bm r} + \inner{\bm\omega_{j}^{(2)}}{\bm r} + \phi_{k}^{(1)}+\phi_{j}^{(2)} \right) + 
    \cos\left( \inner{\bm\omega_{k}^{(1)}}{\bm r} - \inner{\bm\omega_{j}^{(2)}}{\bm r} + \phi_{k}^{(1)} - \phi_{j}^{(2)} \right) \right) \notag \\
    =& \sum_{k\in \mathcal{K}} \sum_{j \in \mathcal{J}} \beta_{k}^{(1)} \beta_{j}^{(2)} \frac{1}{2} 
    \left( \cos\left( \inner{\bm\omega_{k}^{(1)}+\bm\omega_{j}^{(2)}}{\bm r} + \phi_{k}^{(1)}+\phi_{j}^{(2)} \right) +
    \cos\left( \inner{\bm\omega_{k}^{(1)}-\bm\omega_{j}^{(2)}}{\bm r} + \phi_{k}^{(1)} - \phi_{j}^{(2)} \right) \right) \notag \\
    =& \sum_{\bm\omega'\in\mathcal{D}} \Tilde{\beta}_{\bm\omega'} \cos(\inner{\bm\omega'}{\bm r} + \Tilde{\phi}_{\bm\omega'})
\end{align}
\end{proof}

\begin{lemma}\label{lemma2}
Let $\{\bm\omega_j\in\R^D\}_{j\in\mathcal{J}}$ and $\{\phi_j\in\R\}_{j\in\mathcal{J}}$ be a collection of frequency vectors and scalar phases, respectively, indexed by the set $\mathcal{J}\subseteq \mathbb{N}$. Furthermore,  $\{\beta_j\in\R\}_{j\in\mathcal{J}}$ be a set of scalar coefficients, and let $k\in \mathbb{N}$ Then, 
\begin{equation}
    \left(\sum_{j\in \mathcal{J}} \beta_{j} \cos\left(\inner{\bm\omega_{j}}{\bm r} + \phi_{j}\right)\right)^k =  \sum_{\bm\omega'\in\mathcal{H}_k} \Tilde{\beta}_{\bm\omega'} \cos(\inner{\bm\omega'}{\bm r} + \Tilde{\phi}_{\bm\omega'}) \label{eq:lemma2}
\end{equation}

where
\begin{equation}
    \mathcal{H}_k\left(\left\{ \bm\omega_{j} \right\}_{j\in \mathcal{J}}\right)\subseteq \Tilde{\mathcal{H}}_k\left(\left\{ \bm\omega_{j} \right\}_{j\in \mathcal{J}}\right):=\left\{\bm\omega'= \sum_{j \in \mathcal{J}} c_j\bm\omega_{j} \Bigg{|} c_j\in\mathbb{Z} \wedge \sum_{j\in \mathcal{J}}|c_j| \leq k   \right\}
\end{equation}
\end{lemma}

Note that we will often use the notation $\mathcal{H}_k$ and $\Tilde{\mathcal{H}}_k$ instead of explicitly writing the dependence on the set $\left(\left\{ \bm\omega_{j} \right\}_{j\in \mathcal{J}}\right)$ when it is clear from the context. 

\begin{proof}
The statement trivially holds for $k=1$. Assume it also holds for $k$, then 

\begin{align}
    \left(\sum_{j\in \mathcal{J}} \beta_{j} \cos\left(\inner{\bm\omega_{j}}{\bm r} + \phi_{j}\right)\right)^{k+1} 
    &= \left(\sum_{j\in \mathcal{J}} \beta_{j} \cos\left(\inner{\bm\omega_{j}}{\bm r} + \phi_{j}\right)\right)^k \left(\sum_{j\in \mathcal{J}} \beta_{j} \cos\left(\inner{\bm\omega_{j}}{\bm r} + \phi_{j}\right)\right) \\
    &= \left( \sum_{\bm\omega'\in\mathcal{H}_k} \Tilde{\beta}_{\bm\omega'} \cos(\inner{\bm\omega'}{\bm r} + \Tilde{\phi}_{\bm\omega'}) \right) 
    \left(\sum_{j\in \mathcal{J}} \beta_{j} \cos\left(\inner{\bm\omega_{j}}{\bm r} + \phi_{j}\right)\right) \label{eq:byassumption}\\
    &= \sum_{\bm\omega'\in\mathcal{D}\left\{ \mathcal{H}_k, \left\{ \bm\omega_{j} \right\}_{j \in \mathcal{J}}  \right\}} {\beta}'_{\bm\omega'} \cos(\inner{\bm\omega'}{\bm r} + {\phi}'_{\bm\omega'}) \label{eq:bylemma}\\
    &= \sum_{\bm\omega'\in\mathcal{H}_{k+1}} {\beta}'_{\bm\omega'} \cos(\inner{\bm\omega'}{\bm r} + {\phi}'_{\bm\omega'}) \\    
\end{align}
where \cref{eq:byassumption} holds by assumption and \cref{eq:bylemma} holds because of the previous lemma. Moreover we have

\begin{align}
    \mathcal{H}_{k+1} 
    = {D}\left\{ \mathcal{H}_k, \left\{ \bm\omega_{i} \right\}_{i \in \mathcal{J}}  \right\}
    &=\left\{   \bm\omega'= \bm\omega_{h} \pm \bm\omega_{i} \Bigg{|} \bm\omega_h\in\mathcal{H}_k, i\in\mathcal{J}    \right\} \\
    &\subseteq \left\{   \bm\omega'= \sum_{j \in \mathcal{J}} c_j\bm\omega_{j} \pm \bm\omega_{i} \Bigg{|}  c_j\in\mathbb{Z} \wedge \sum_{j\in \mathcal{J}}|c_j| \leq k , i\in\mathcal{J}    \right\} \\
    &\subseteq \left\{   \bm\omega'= \sum_{j \in \mathcal{J}} c_j\bm\omega_{j}\Bigg{|}  c_j\in\mathbb{Z} \wedge \sum_{j\in \mathcal{J}}|c_j| \leq k+1  \right\} \\
\end{align}
So \cref{eq:lemma2} holds for $k+1$ as well. Then by induction \cref{eq:lemma2} holds $\forall k \in \mathbb{N}$
\end{proof}

\subsubsection{Main proof}

Recall that we are interested in understanding the expressive power of INR architectures that can be decomposed into a mapping function $\gamma:\R^D\to\R^T$ followed by a multilayer perceptron (MLP), with weights $\bm W^{(\ell)}\in\R^{F_{\ell-1}\times F_\ell}$, bias $\bm b^{(\ell)}\in\R^{F_\ell}$, and activation function $\rho^{(\ell)}:\R\to\R$, applied elementwise; at each layer $\ell=1,\dots, L-1$. That is, if we denote by $\bm z^{(\ell)}$ each layers post activation, most INR architectures compute
\begin{align}
    \bm z^{(0)} &= \gamma(\bm r), \nonumber \\
    \bm z^{(\ell)} &= \rho^{(\ell)}\left(\bm W^{(\ell)}\bm z^{(\ell-1)}+\bm b^{(\ell)}\right),\; \ell=1,\dots,L-1 \label{eq:INR_supp}  \\
    f_{\bm\theta}(\bm r) &= \bm W^{(L)}\bm z^{(L-1)} + \bm b^{(L)}. \nonumber
\end{align}
Based on this architecture we can prove the following theorem.

\begin{theorem*}
	Let $f_{\bm\theta}:\R^D\to\R$ be an INR of the form of \cref{eq:INR_supp} with $\rho^{(\ell)}(z)=\sum_{k=0}^K\alpha_k z^k$ for $\ell>1$. Furthermore, let $\bm\Omega=[\bm\Omega_0,\dots,\bm\Omega_{T-1}]^\top\in\R^{T\times D }$ and $\bm\phi\in\R^T$ denote the matrix of frequencies and vector of phases, respectively, used to map the input coordinate $\bm r\in\R^D$ to $\gamma(\bm r)=\sin(\bm\Omega \bm r + \bm \phi)$. This architecture can only represent functions of the form
	\begin{equation}
		f_\vtheta(r) = \sum_{\bm\omega'\in\mathcal{H}(\bm\Omega)}c_{\bm\omega'}\sin{(\langle\bm\omega', \bm r\rangle + \phi_{\bm\omega'})}, \label{eq:expressive_supp}
	\end{equation}
	where
	\begin{equation}
		\mathcal{H}(\bm\Omega)\subseteq\left\{\bm\omega'=\sum_{t=0}^{T-1} s_t \bm\Omega_t\; \Bigg{|}\; s_t\in\mathbb{Z} \wedge \sum_{t=0}^{T-1} |s_t| \leq K^{L-1} \right\}.
	\end{equation}
\end{theorem*}

\begin{proof}
We will prove the statement by induction. To that end, let us denote the preactivation vector at each layer as $\bm v^{(\ell)}$, i.e. $\bm z^{(\ell)} = \rho^{(\ell)} \left( \bm v^{(\ell)}  \right)$. We will first derive the expressions for the base case.

\paragraph{Base case} Consider the preactivation of a node at the first layer of the neural network for any mapping of the form in \cref{eq:INR_supp}. Then
\begin{equation}
    \bm v^{(1)}_j =  \bm W^{(1)}_{j }\gamma(\bm r) = \sum_{t=0}^{T-1} b_{tj} \cos{(\inner{\bm\Omega_t}{\bm r} + \phi_{tj})}
\end{equation}
with some $b_{tj}\in\R$ and $\phi_{tj}\in\R$ depending on the first layer weights connected to that node. Also note that interchanging sines with cosines only affects the phase terms.

Therefore, using the result of Lemma \ref{lemma2}, and after applying the activation function, the output of each node at the first layer is given by
\begin{align}
    \bm z^{(1)}_j = \rho^{(1)} \left( \bm v^{(1)}_j  \right)
    = \sum_{k=0}^K\alpha_k \left(\bm v^{(1)}_j\right)^k
    =& \sum_{k=0}^K\alpha_k \left( \sum_{t=0}^{T-1} b_{tj} \cos{(\inner{\bm\Omega_t}{\bm r} + \phi_{tj})} \right)^k \\
    =& \sum_{k=0}^K\alpha_k \sum_{\bm\omega_k'\in\mathcal{H}_k} {\beta}_{\bm\omega_k'} \cos(\inner{\bm\omega_k'}{\bm r} + {\phi}_{\bm\omega_k'}) \\
    =& \sum_{\bm\omega_k\in\mathcal{H}_K'} \Tilde{\beta}_{\bm\omega_k} \cos(\inner{\bm\omega_k}{\bm r} + \Tilde{\phi}_{\bm\omega_k})
\end{align}
where  $\mathcal{H}_K' := \bigcup\limits_{j=1}^K \mathcal{H}_j$ and we use the definitions of $\mathcal{H}_k$ and $\Tilde{\mathcal{H}}_k$ in Lemma \ref{lemma2}. Therefore, since $\forall k  \mathcal{H}_k \subseteq \Tilde{\mathcal{H}}_k\; $ by construction, and  $ \forall j \leq k $  $ \mathcal{H}_j \subseteq \Tilde{\mathcal{H}}_k $; then it holds that $\mathcal{H}_K' \subseteq \Tilde{\mathcal{H}}_K$, i.e.,
\begin{equation}
	\mathcal{H}_K' \subseteq \Tilde{\mathcal{H}}_K = \left\{\bm\omega'=\sum_{t=0}^{T-1} s_t \bm\Omega_t\; \Bigg{|}\; s_t\in\mathbb{Z} \wedge \sum_{t=0}^{T-1}|s_t| \leq K^{} \right\}.
\end{equation}

\paragraph{Induction step} Assume the output of the nodes at layer $\ell$ satisfy the following expression:
\begin{equation}
    \bm z_j^{(\ell)} = \sum_{\omega'\in\mathcal{H}^{(\ell)}(\bm\Omega)}c_{\bm\omega', j}\sin{(\inner{\bm\omega'}{\bm r} + \phi_{\bm\omega', j})} 
\end{equation}
where
\begin{equation}
   \mathcal{H}^{(\ell)} \subseteq \Tilde{\mathcal{H}}_{K^{\ell}} = \left\{\bm\omega'=\sum_{t=0}^{T-1} s_t \bm\Omega_t\; \Bigg{|}\; s_t\in\mathbb{Z} \wedge \sum_{t=0}^{T-1}|s_t| \leq K^{\ell} \right\}. 
\end{equation}
Then, the preactivation of any node at the $(\ell+1)^{th}$ layer can be expressed as:
\begin{equation}
    v_j^{(\ell+1)} = \sum_{\bm\omega'\in\mathcal{H}^{(\ell)}(\bm\Omega)}b_{\bm\omega', j}\sin{(\inner{\bm\omega'}{\bm r} + \Tilde{\phi}_{\bm\omega', j})} 
\end{equation}
since the sum of cosines with the same frequency only result in a cosine with the same frequency but with a modified phase and amplitude. Hence, after applying the activation function the output of the $j^{th}$ node at the $(\ell+1)^{th}$ layer can be written as:
\begin{align}
    \bm z^{(\ell+1)}_j = \rho^{(\ell+1)} \left( \bm v^{(\ell+1)}_j  \right)
    = \sum_{k=0}^K\alpha_k \left(\bm v^{(\ell+1)}_j\right)^k
    =& \sum_{k=0}^K\alpha_k \left( \sum_{\bm\omega'\in\mathcal{H}^{(\ell)}(\bm\Omega)}b_{\bm\omega', j}\sin{(\inner{\bm\omega'}{\bm r} + \Tilde{\phi}_{\bm\omega', j})}
    \right)^k \label{eq:activationlplus1}
\end{align}
Let us inspect the term $
    \left( \sum_{\bm\omega'\in\mathcal{H}^{(\ell)}(\bm\Omega)}b_{\bm\omega', j}\sin{(\inner{\bm\omega'}{\bm r}+ \Tilde{\phi}_{\bm\omega', j})}
    \right)^k
$. Instead of directly applying \cref{lemma2}, we will leverage the fact that all the frequencies $\bm\omega'\in\mathcal{H}^{(\ell)}$ share a similar structure. More precisely, they all can be represented as a sum of the frequencies in the set $\bm\Omega$. To that end, let us show the following intermediate result:
\begin{equation}
    \left( \sum_{\bm\omega'\in\mathcal{H}^{(\ell)}(\bm\Omega)}b_{\bm\omega', j}\sin{(\inner{\bm\omega'}{\bm r} + \Tilde{\phi}_{\bm\omega', j})}
    \right)^k
    = \sum_{\bm\omega'\in\mathcal{H}_k^{(\ell)}(\bm\Omega)}\Tilde{b}_{\bm\omega', j}\sin{(\inner{\bm\omega'}{\bm r} + \Tilde{\Tilde{\phi}}_{\bm\omega', j})} 
    \label{eq:layerlpowerk}
\end{equation}
where $\mathcal{H}_k^{(\ell)} \subseteq \Tilde{\mathcal{H}}_{kK^{\ell}}$. The base case for $k=1$ holds trivially. Now assume \cref{eq:layerlpowerk} holds for $k$, then 
\small
\begin{align}
    \left( \sum_{\bm\omega'\in\mathcal{H}^{(\ell)}(\bm\Omega)}b_{\bm\omega', j}\sin{(\inner{\bm\omega'}{\bm r} + \Tilde{\phi}_{\bm\omega', j})}
    \right)^{k+1}
    =&  \left( \sum_{\bm\omega'\in\mathcal{H}^{(\ell)}(\bm\Omega)}b_{\bm\omega', j}\sin{(\inner{\bm\omega'}{\bm r} + \Tilde{\phi}_{\bm\omega', j})}
    \right)^{k}  \left( \sum_{\bm\omega'\in\mathcal{H}^{(\ell)}(\bm\Omega)}b_{\bm\omega', j}\sin{(\inner{\bm\omega'}{\bm r} + \Tilde{\phi}_{\bm\omega', j})}
    \right) \\
    =& \left( \sum_{\bm\omega'\in\mathcal{H}_k^{(\ell)}(\bm\Omega)}\Tilde{b}_{\bm\omega', j}\sin{(\inner{\bm\omega'}{\bm r} + \Tilde{\Tilde{\phi}}_{\bm\omega', j})} \right) \left( \sum_{\bm\omega'\in\mathcal{H}^{(\ell)}(\bm\Omega)}b_{\bm\omega', j}\sin{(\inner{\bm\omega'}{\bm r} + \Tilde{\phi}_{\bm\omega', j})}
    \right) \\
    =& \sum_{\bm\omega'\in \mathcal{D}\left\{ \mathcal{H}_k^{(\ell)}, \mathcal{H}^{(\ell)}   \right\} }\Tilde{\Tilde{b}}_{\bm\omega', j}\sin{(\inner{\bm\omega'}{\bm r} + \Tilde{\Tilde{\Tilde{\phi}}}_{\bm\omega', j})} 
\end{align} 
\normalsize
where last equality holds because of \cref{lemma2} and we have:
\begin{align}
    \mathcal{D}\left\{ \mathcal{H}_k^{(\ell)}, \mathcal{H}^{(\ell)}   \right\}
    &= \left\{ \bm\omega_1 \pm \bm\omega_2 \Bigg{|} \bm\omega_1 \in \mathcal{H}_k^{(\ell)}, \bm\omega_2 \in \mathcal{H}^{(\ell)}
    \right\} \\
    & \subseteq \left\{ \bm\omega_1 \pm \bm\omega_2 \Bigg{|} \bm\omega_1 \in \mathcal{\Tilde{H}}_{kK^{\ell}} , \bm\omega_2 \in \Tilde{\mathcal{H}}_{K^{\ell}}
    \right\} \\
    &= \left\{ \sum_{t=0}^{T-1} s_t^{(1)} \bm\Omega_t \pm\sum_{t=0}^{T-1} s_t^{(2)}  \bm\Omega_t  \Bigg{|} \sum_{t=0}^{T-1}|s_t^{(1)} | \leq {kK^{\ell}} , \sum_{t=0}^{T-1}|s_t^{(2)} | \leq {K^{\ell}}
    \right\} \\
    &= \left\{ \sum_{t=0}^{T-1} \left(s_t^{(1)} \pm s_t^{(2)} \right) \bm\Omega_t  \Bigg{|} \sum_{t=0}^{T-1}|s_t^{(1)} | \leq {kK^{\ell}} , \sum_{t=0}^{T-1}|s_t^{(2)} | \leq {K^{\ell}}
    \right\} \\
    &\subseteq \left\{ \sum_{t=0}^{T-1} s_t' \bm\Omega_t  \Bigg{|} \sum_{t=0}^{T-1}|s_t' | \leq {(k+1)K^{\ell}} 
    \right\} = \Tilde{\mathcal{H}}_{(k+1)K^{\ell}}\\        
\end{align}
where the last line follows from triangle inequality. This proves our intermediate result in \cref{eq:layerlpowerk}. 

Now, let us use this result to complete the proof of the inductive step. In particular, we can now write \cref{eq:activationlplus1} as
\begin{align}
    \bm z^{(\ell+1)}_j &= \sum_{k=0}^K\alpha_k \sum_{\bm\omega'\in \mathcal{H}_k^{(\ell)}(\bm\Omega)}\Tilde{b}_{\bm\omega', j, k}\sin{(\inner{\bm\omega'}{\bm r} + \Tilde{\Tilde{\phi}}_{\bm\omega', j, k})} \\
    &= \sum_{\bm\omega'\in\mathcal{H}^{(\ell+1)}(\bm\Omega)}c_{\bm\omega', j}\sin{(\inner{\bm\omega'}{\bm r} + \phi_{\bm\omega', j})} 
\end{align}
where $\mathcal{H}^{(\ell+1)} := \bigcup\limits_{k=1}^K \mathcal{H}_k^{(\ell)} \subseteq \bigcup\limits_{k=1}^K \Tilde{\mathcal{H}}_{kK^{\ell}} \subseteq \Tilde{\mathcal{H}}_{KK^{\ell}} 
= \Tilde{\mathcal{H}}_{K^{\ell+1}}$. This sequence of inclusions concludes the proof.

\end{proof}
\clearpage

\subsection{Two-layer SIREN example}
\begin{example*}
    Let $f_\vtheta$ be a three-layer SIREN defined as \footnote{Note that we have omitted the bias terms to simplify the notation. These biases only change the phase terms in the sinusoids of the sum. }
    \begin{equation}
        f_\vtheta(r)={\bm w^{(2)}}^\top\sin\left(\mW^{(1)}\sin\left( \bm\Omega r \right) \right),
    \end{equation}
    where $r \in \R$, $\bm\Omega\in\R^{T}$, $\mW^{(1)}\in\R^{F \times T}$, and $\bm w^{(2)}\in\R^{F}$. The output of this network can equivalently be represented as
    \begin{equation}
    	f_\vtheta(r) = \sum_{m=0}^{F-1} \sum_{s_1, \dots, s_T=-\infty}^{\infty} c_{m,s_1,\dots,s_T}\sin{ \left( \left({\sum_{t=0}^{T-1}   s_t \omega_t}\right) {r} \right)},\label{eq:bessel_supp}
    \end{equation}
    where $\omega_t^\top \in \R^D$ denotes the $t^{th}$ row of $\Omega$,
    \begin{equation}
    	c_{m,s_1,\dots,s_T} = \left(\prod_{t=0}^{T-1} J_{s_t}\left(W^{(1)}_{m,t}\right)\right) w^{(2)}_m,\label{eq:coeffs_siren_supp}
    \end{equation}
    and $J_s$ denotes the Bessel function of first kind of order $s$.
\end{example*}

\begin{proof}
    As we have discussed before, the first layer in SIREN plays the role of the frequency mapping, i.e.
    \begin{equation}
        \bm z^{(0)} = \sin \left(\bm W^{(0)} \bm r  \right)  = \sin(\bm\Omega \bm r).
    \end{equation}
    Hence the input of a node at the next layer is a linear combination of sinusoids at mapping frequencies. The output of a node at second layer can be written as: 
    
    \begin{align}
	z_m^{(1)} &= \sin\left( \bW^{(1)}_{m,:} \sin\left( \bm\Omega r  \right)  \right) \label{eq:z_k1} \\
	 &= \sin\left(\sum_{t=0}^{T-1} \bW^{(1)}_{m,t} \sin\left( {\omega_t}{r}  \right)  \right) \\	
	 &= \operatorname{Im} \left\{ \exp{ \left( j \left(\sum_{t=0}^{T-1} \bW^{(1)}_{m,t} \sin\left( {\omega_t}{r}  \right)  \right)  \right) }  \right\} \\
	 &= \operatorname{Im} \left\{ \prod_{t=0}^{T-1}  \exp{ \left( j \bW^{(1)}_{m,t} \sin\left( {\omega_t}{r}  \right)  \right) }  \right\} \\
	 &= \operatorname{Im} \left\{  \prod_{t=0}^{T-1} \sum_{s_t=-\infty}^{\infty} J_{s_t}(\bW^{(1)}_{m,t}) \exp{ \left( j s_t  {\omega_t}{r}   \right) }  \right\} \label{eq:zk_bessel}\\
	 &= \operatorname{Im} \left\{ \sum_{s_0=-\infty}^{\infty} \dots \sum_{s_{T-1}=-\infty}^{\infty}  \prod_{t=0}^{T-1} J_{s_t}(\bW^{(1)}_{m,t}) \exp{ \left( j s_t  {\omega_t}{r}   \right) }  \right\} \\
	 &=\sum_{s_1, \dots, s_T=-\infty}^{\infty}
	 \operatorname{Im} \left\{  \prod_{t=0}^{T-1} J_{s_t}(\bW^{(1)}_{m,t}) \exp{ \left( j s_t {\omega_t}{r}   \right) }  \right\} \\
	 &= \sum_{s_1, \dots, s_T=-\infty}^{\infty} \operatorname{Im} \left\{  \left( \prod_{t=0}^{T-1} J_{s_t}(\bW^{(1)}_{m,t}) \right) \exp{ \left( j \sum_{t=0}^{T-1} s_t {\omega_t}{r} \right) }  \right\} \\
	 &= \sum_{s_1, \dots, s_T=-\infty}^{\infty}  \left( \prod_{t=0}^{T-1} J_{s_t}(\bW^{(1)}_{m,t}) \right) \operatorname{Im} \left\{   \exp{ \left( j \sum_{t=0}^{T-1} s_t {\omega_t}{r} \right) }  \right\} \\	 
	 &= \sum_{s_1, \dots, s_T=-\infty}^{\infty} \left( \prod_{t=0}^{T-1} J_{s_t}(\bW^{(1)}_{m,t}) \right) \sin{ \left( \sum_{t=0}^{T-1} s_t {\omega_t}{r} \right) } \label{eq:z_kend}
     \end{align}
     where $J_n(\cdot)$ represents the Bessel function of the first kind of order $n$ and   (\ref{eq:zk_bessel}) follows from the Fourier series expansion of $\exp(j\beta\sin(\omega_0x))$:
    \begin{equation}
	    \exp(j\beta\sin(\omega_0x)) = \sum_{n= -\infty}^{\infty} J_n(\beta)\exp{(j n \omega_0x)}. \label{eq:FSexpansionofexpjsin}
    \end{equation}
    Therefore, the output of the neural network can be written as:
    \begin{align}
        f_\vtheta(r)={\bm w^{(2)}}^\top z^{(1)}  = \sum_{m=0}^{F-1} {\bm w_m^{(2)}} z_m^{(1)} 
        =& \sum_{m=0}^{F-1} {\bm w_m^{(2)}}
        \sum_{s_1, \dots, s_T=-\infty}^{\infty} \left( \prod_{t=0}^{T-1} J_{s_t}(\bW^{(1)}_{m,t}) \right) \sin{ \left( \sum_{t=0}^{T-1} s_t {\omega_t}{r} \right) } \\
        =& \sum_{m=0}^{F-1} 
        \sum_{s_1, \dots, s_T=-\infty}^{\infty} {\bm w_m^{(2)}} \left( \prod_{t=0}^{T-1} J_{s_t}(\bW^{(1)}_{m,t}) \right) \sin{ \left( \sum_{t=0}^{T-1} s_t {\omega_t}{r} \right) }
    \end{align}
    
\end{proof}

\clearpage
\section{Imperfect recovery}
\subsection{Experimental details}

For the experiment in Section 4.1, we train an FFN with four fully connected layers: three hidden layers with dimension $256$ followed by an output layer of dimension $C=3$. All the networks presented in Fig.2 are trained for 2000 iterations with Adam optimizer \cite{kingma2014adam}. The training image has the size of $512\times512$ and we use all the available pixels during training.

\subsection{Additional experiments}
We further demonstrate the imperfect recovery phenomenon with INRs for different networks and configurations. In \cref{fig:siren_imperfect}, we present the results for SIREN~\cite{sitzmann2019siren}, where the first layer of the form $\bm z^{(0)} = \sin \left(\omega_0 (\bm W^{(0)} \bm r + \bm b^{(0)}) \right)$ can be considered as the input mapping $\gamma(\bm r)$. For a fair comparison with FFNs and to better illustrate the strong dependence of the learned representation on the chosen mapping, we do not perform any updates on the parameters of the initial layer, i.e., $\bm W^{(0)}$ and $\bm b^{(0)}$, during training. The initialization of these parameters ensures a similar mapping to that of FFN presented in Figure 2 of the main text, i.e. two single frequency mappings with frequencies $f_0=1$ and $f_0=0.5$ followed by a rich mapping. As for the rest of the architecture, we use the same number of layers and training strategy. As you can see in \cref{fig:siren_imperfect}, initializing the SIREN with this $\gamma(\bm r)$ results also in a set $\mathcal{H}(\bm \Omega)\subseteq\{2\pi k | k\in\mathbb{Z}\}$, which leads to a very imperfect recovery with the reconstruction looking aliased in the spatial domain.

\def \subfigsizeimrecsiren{2.8cm}
\def \textboxsizeone{2.8cm}
\begin{figure}[h!]
\centering
\begin{subfigure}{\subfigsizeimrecsiren}
    \centering 
    \includegraphics[width=\subfigsizeimrecsiren]{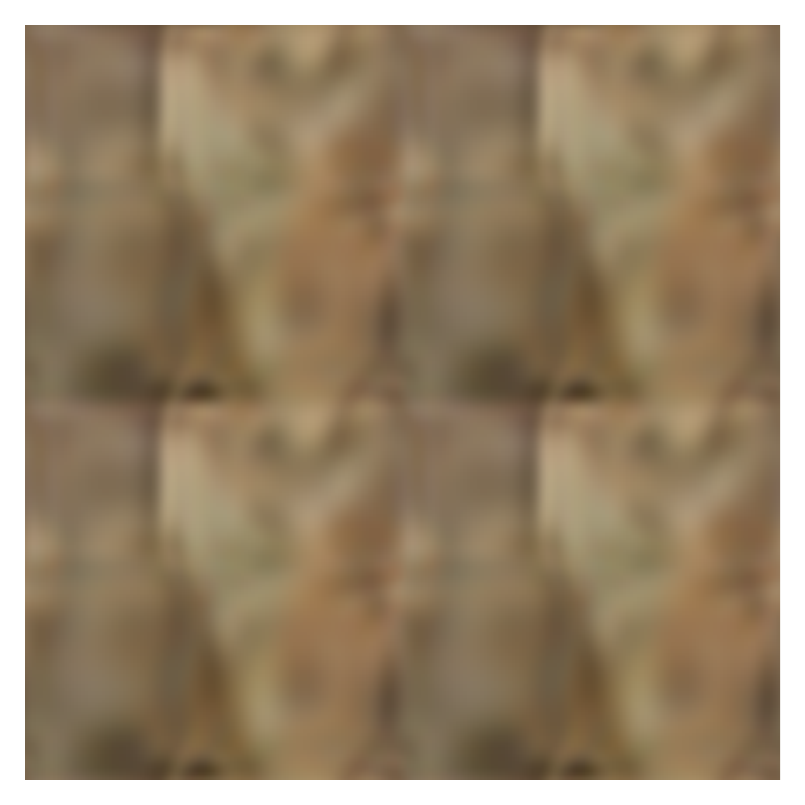} \\
    \includegraphics[width=\subfigsizeimrecsiren]{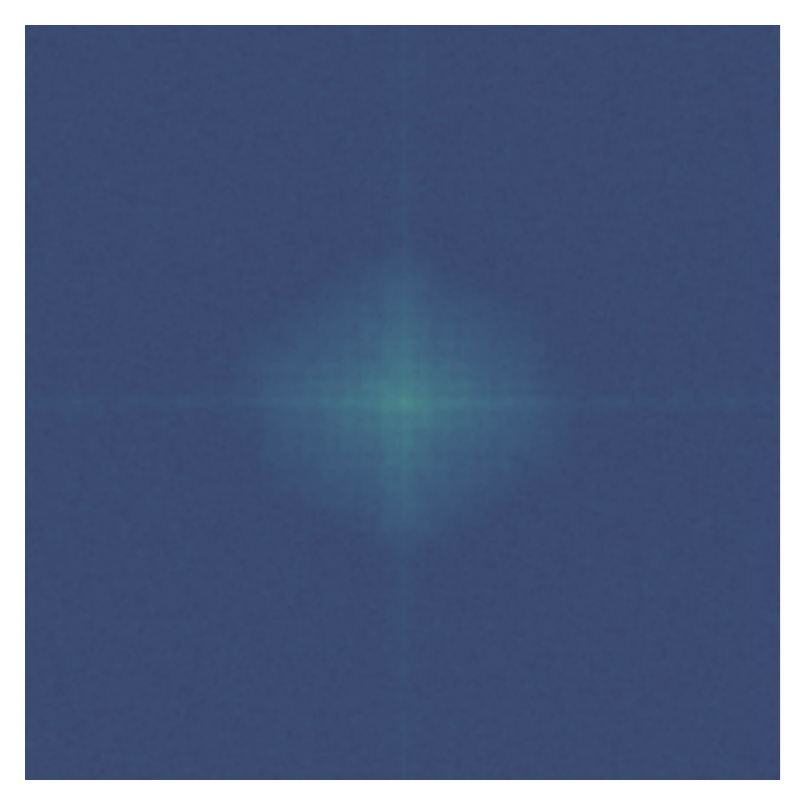} \\
    \includegraphics[width=\subfigsizeimrecsiren]{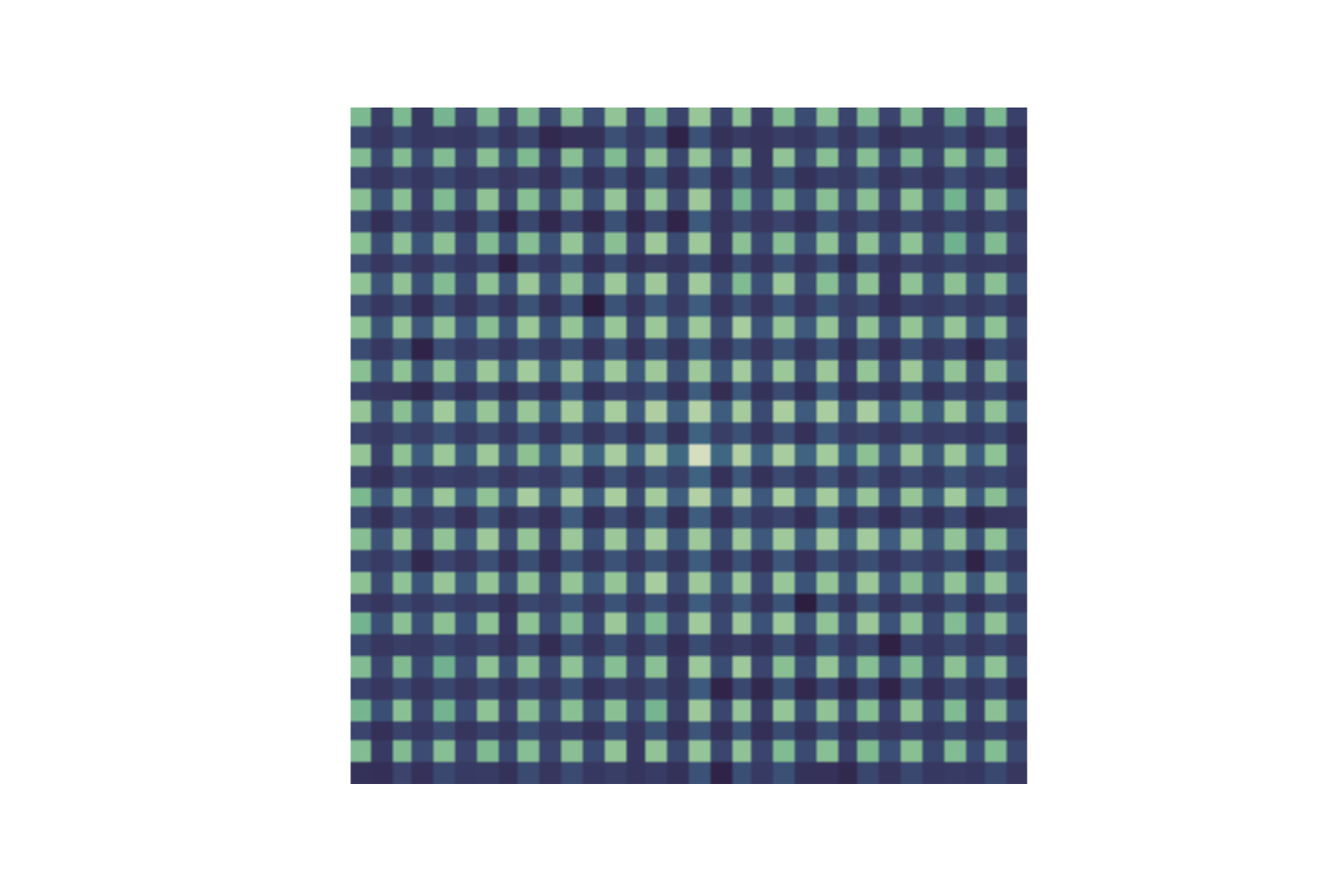} \\
    \parbox{\textboxsizeone}{\centering \small (a) SIREN ($\omega_0=1$) \\ \footnotesize{$\bm W^{(0)} = I$}}
\end{subfigure}
\begin{subfigure}{\subfigsizeimrecsiren}
    \centering 
    \includegraphics[width=\subfigsizeimrecsiren]{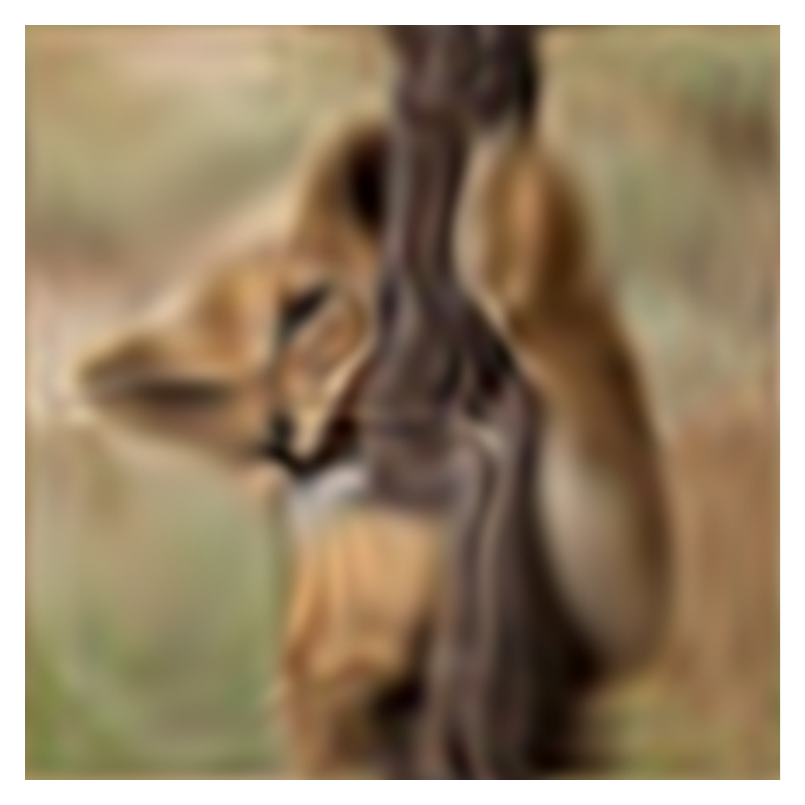} \\
    \includegraphics[width=\subfigsizeimrecsiren]{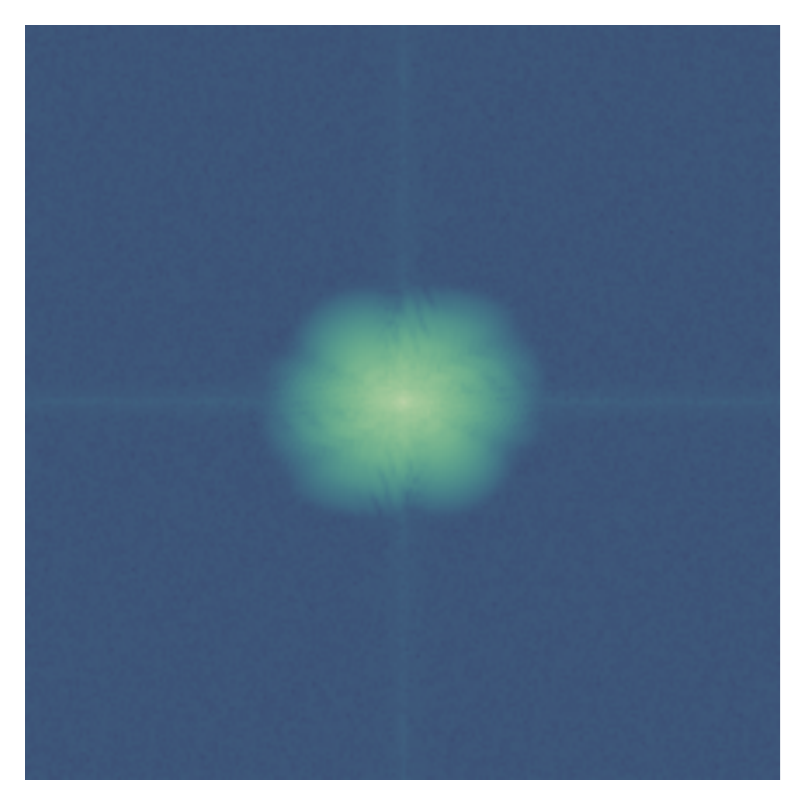} \\
    \includegraphics[width=\subfigsizeimrecsiren]{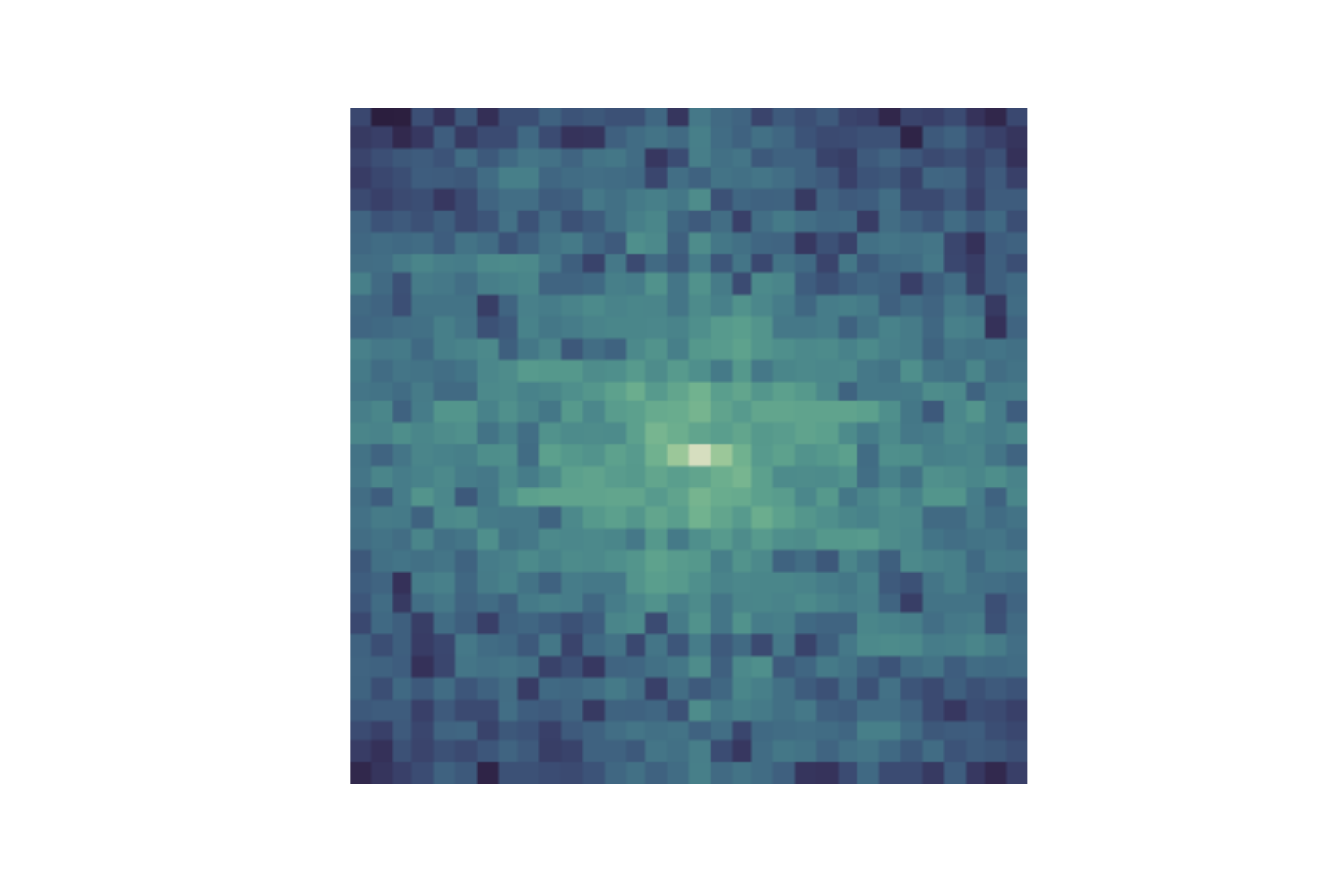} \\
    \parbox{\textboxsizeone}{\centering \small (b) SIREN ($\omega_0=1$) \\ \footnotesize{$\bm W^{(0)} = 0.5 \times I$}}
\end{subfigure}
\begin{subfigure}{\subfigsizeimrecsiren}
    \centering 
    \includegraphics[width=\subfigsizeimrecsiren]{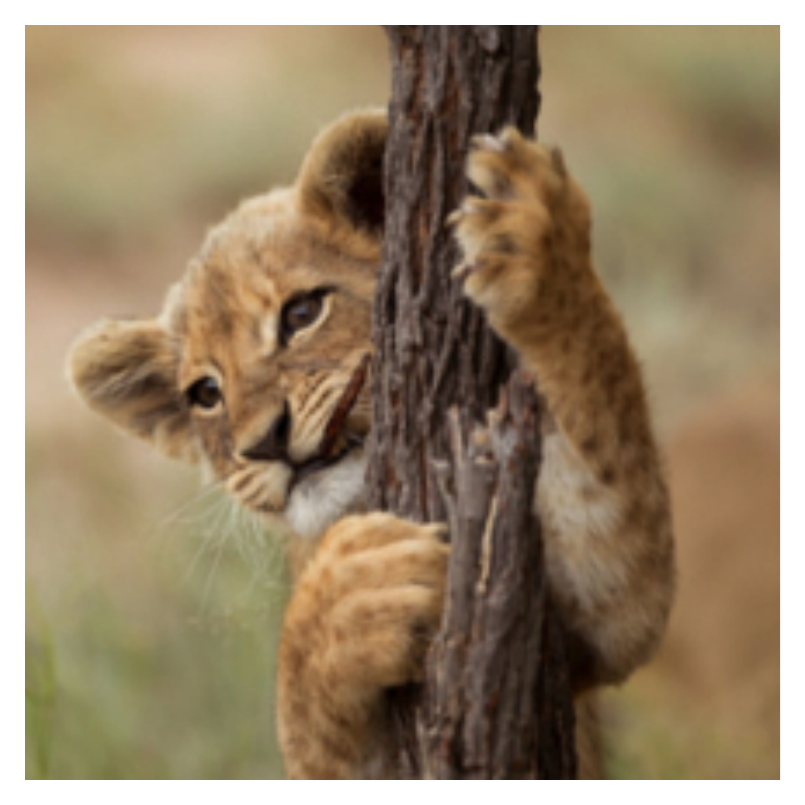} \\
    \includegraphics[width=\subfigsizeimrecsiren]{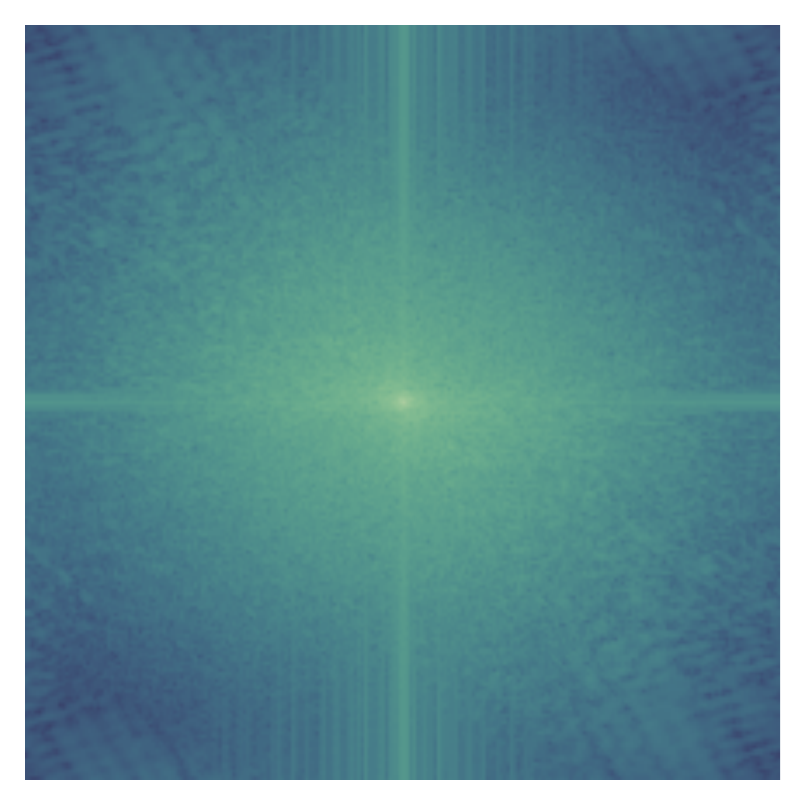} \\
    \includegraphics[width=\subfigsizeimrecsiren]{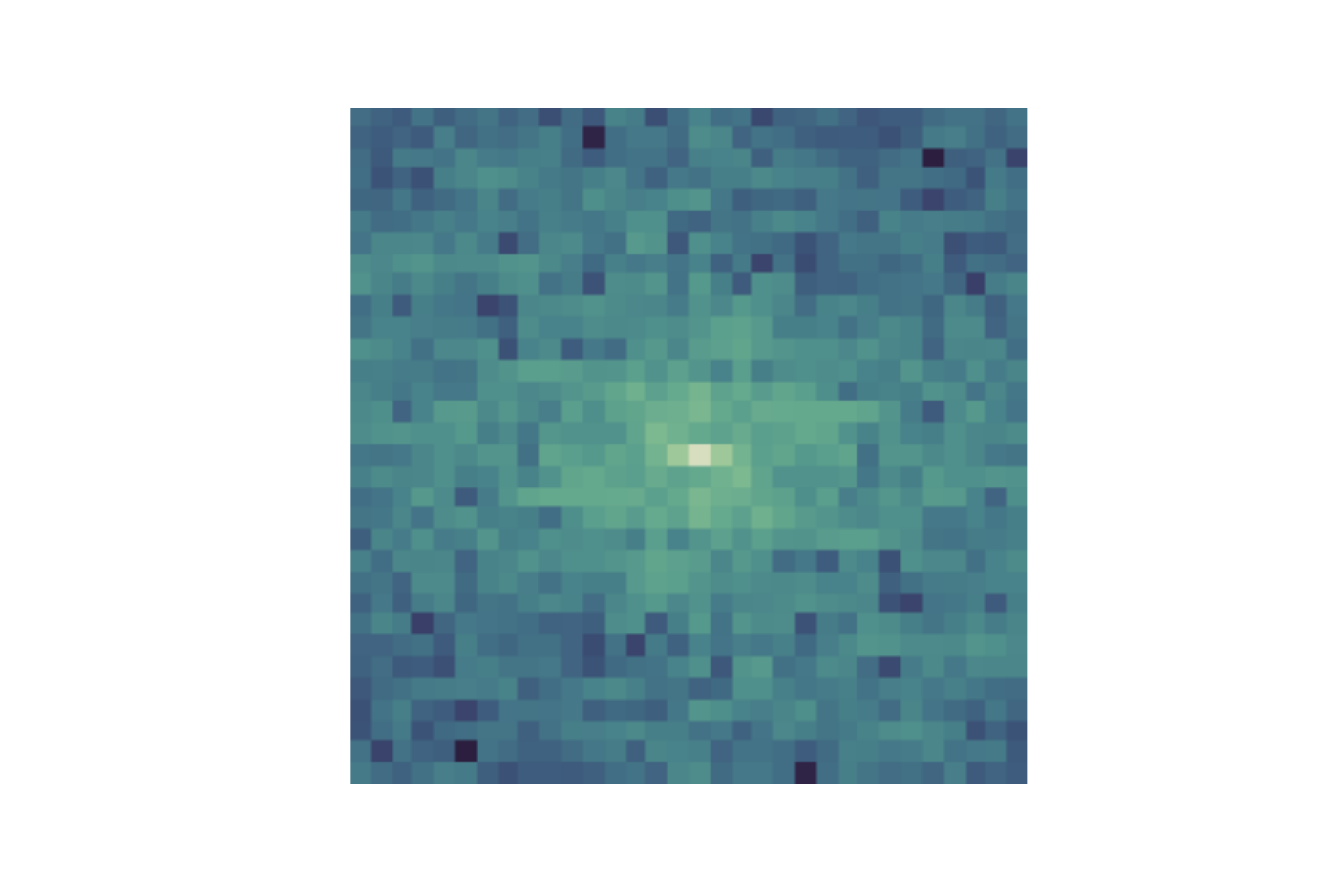} \\
    \parbox{\textboxsizeone}{\centering \small (c) SIREN ($\omega_0=30$) \\ \footnotesize{$\bm W^{(0)} \in \R^{256 \times 2}$}}
\end{subfigure}
\begin{subfigure}{\subfigsizeimrecsiren}
    \centering 
    \includegraphics[width=\subfigsizeimrecsiren]{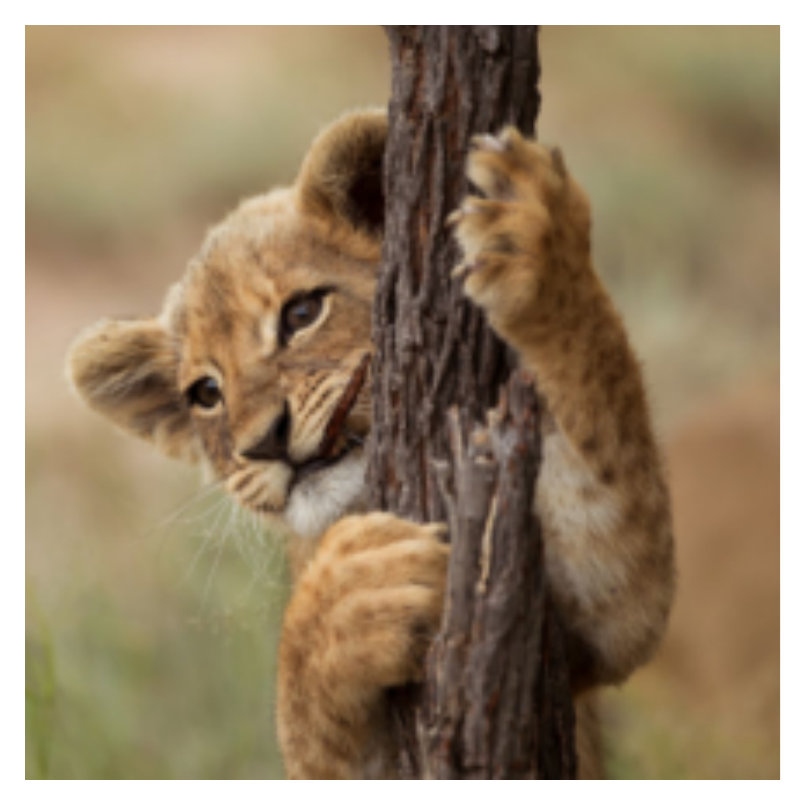} \\
    \includegraphics[width=\subfigsizeimrecsiren]{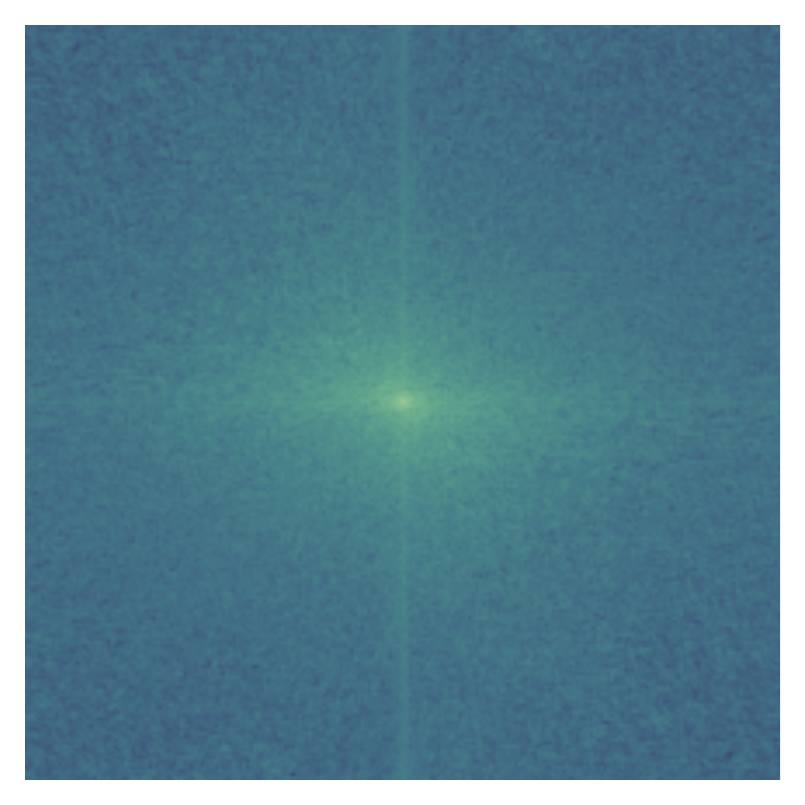} \\
    \includegraphics[width=\subfigsizeimrecsiren]{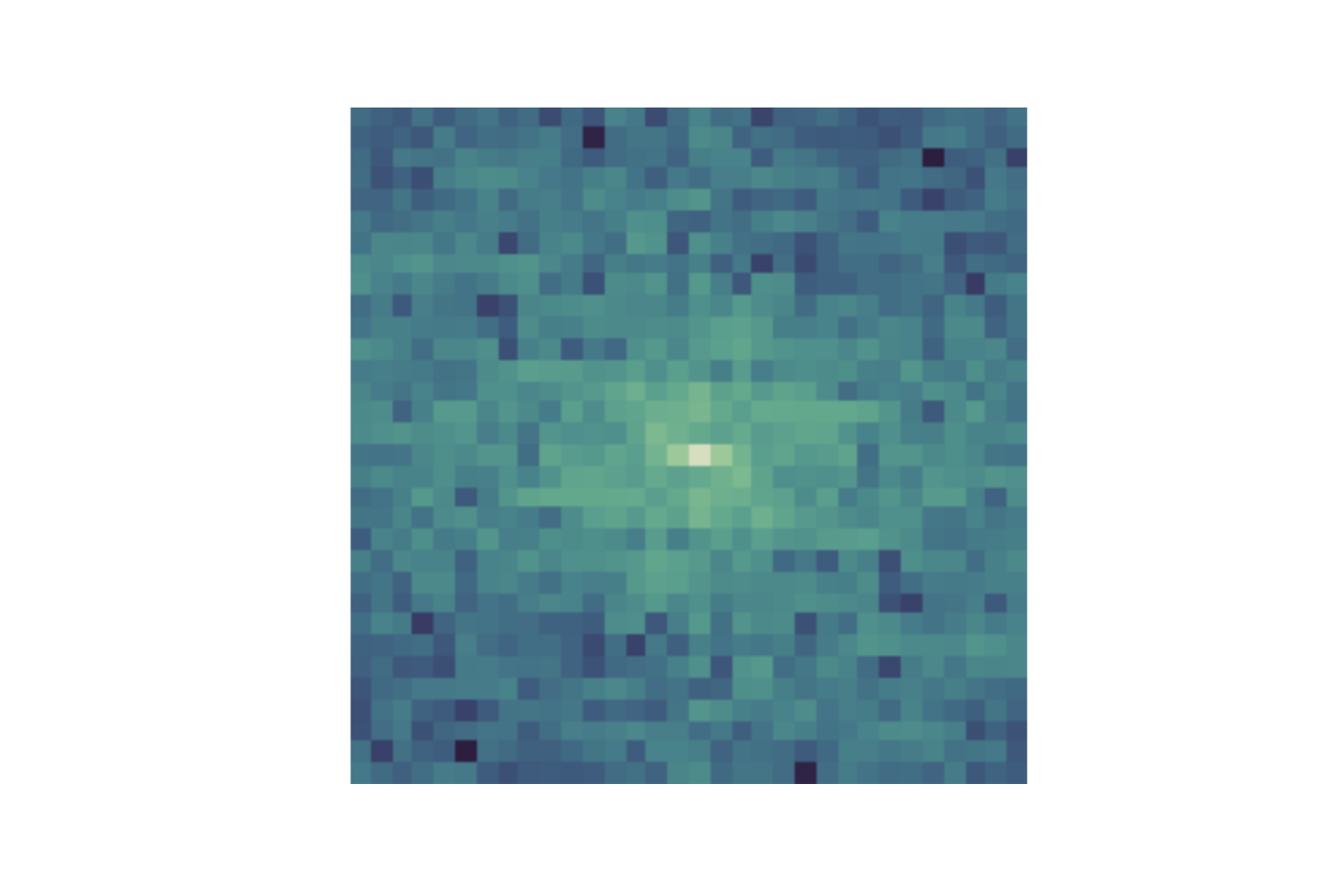} \\
    \parbox{\textboxsizeone}{\centering \small (d) Ground Truth \\ \hfill}
\end{subfigure}
\caption{\textbf{Top row:} Image reconstruction with SIREN \cite{sitzmann2019siren} for different configurations. \textbf{Middle row:} Magnitude of the DFT of the reconstruction. \textbf{Bottom row:} The center crop of size $32 \times 32$ from the magnitude of the DFT of the reconstruction.}
\label{fig:siren_imperfect}
\end{figure}

Note, however, that in the original formulation of SIREN~\cite{sitzmann2019siren}, the parameters of the initial layer are also updated during training. To see the effect that this has in the reconstruction, we repeat the same experiment with a SIREN\footnote{The notation *, e.g., SIREN*, denotes the corresponding INR involving input mapping with learnable parameters.} with trainable $\bm W^{(0)}$ and $\bm b^{(0)}$. Interestingly, as illustrated in \cref{fig:siren_imperfect_trainable}, the aliasing effect is mostly eliminated in this case. This a result of having a dynamic input mapping, which enables spreading the energy over all spectral coefficients rather than only the even ones. Nevertheless, similar to our observations for Fig.2, we see that the reconstruction for low values of $\omega_0$ is also blurry with most of the energy concentrated in the low frequencies. 

\def \subfigsizeimrecsiren{2.8cm}

\begin{figure}
\centering
\begin{subfigure}{\subfigsizeimrecsiren}
    \centering 
    \includegraphics[width=\subfigsizeimrecsiren]{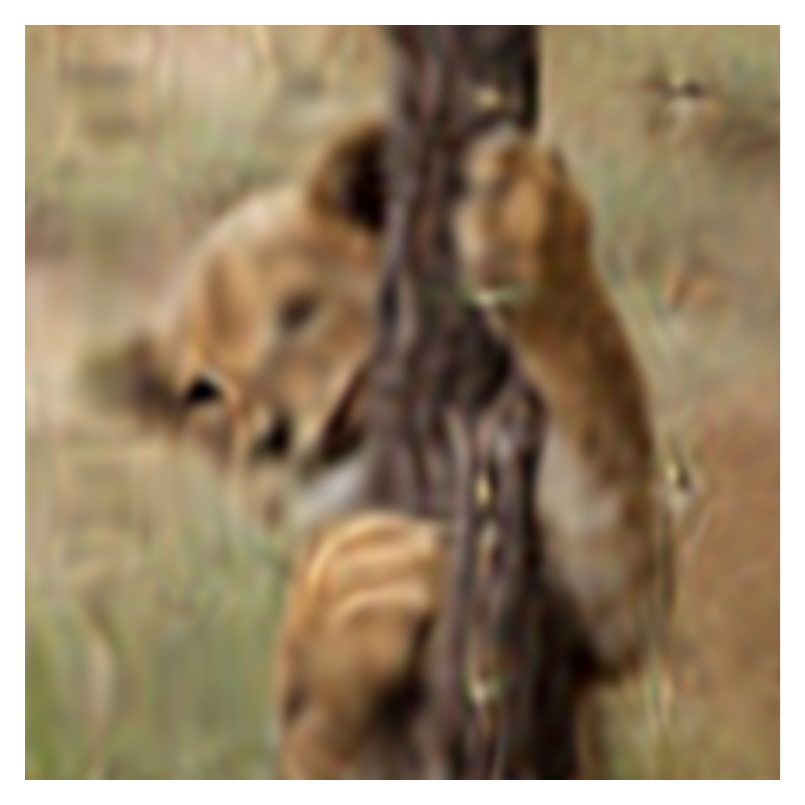} \\
    \includegraphics[width=\subfigsizeimrecsiren]{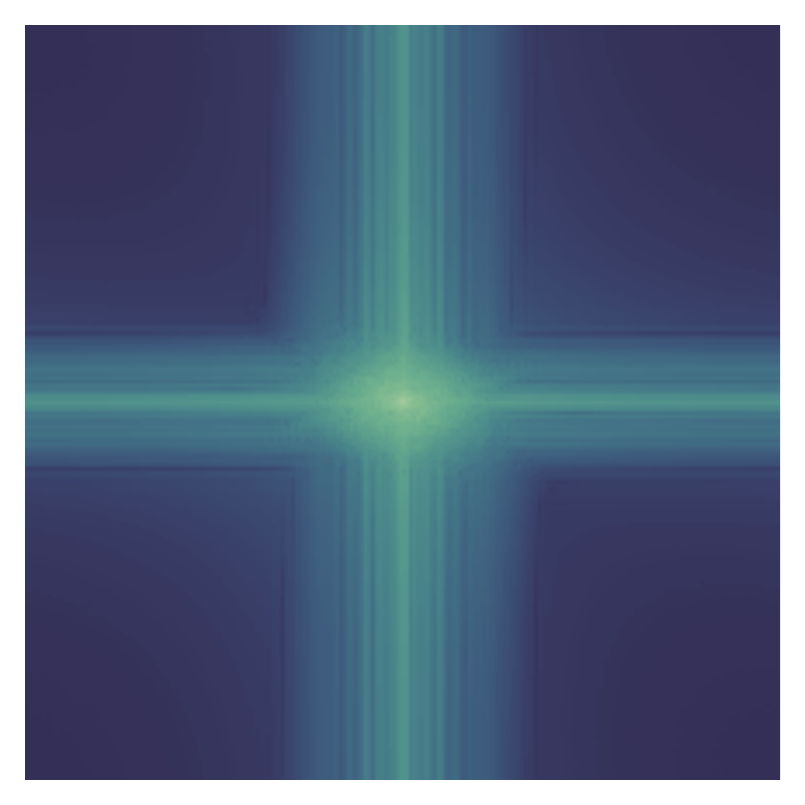} \\
    \includegraphics[width=\subfigsizeimrecsiren]{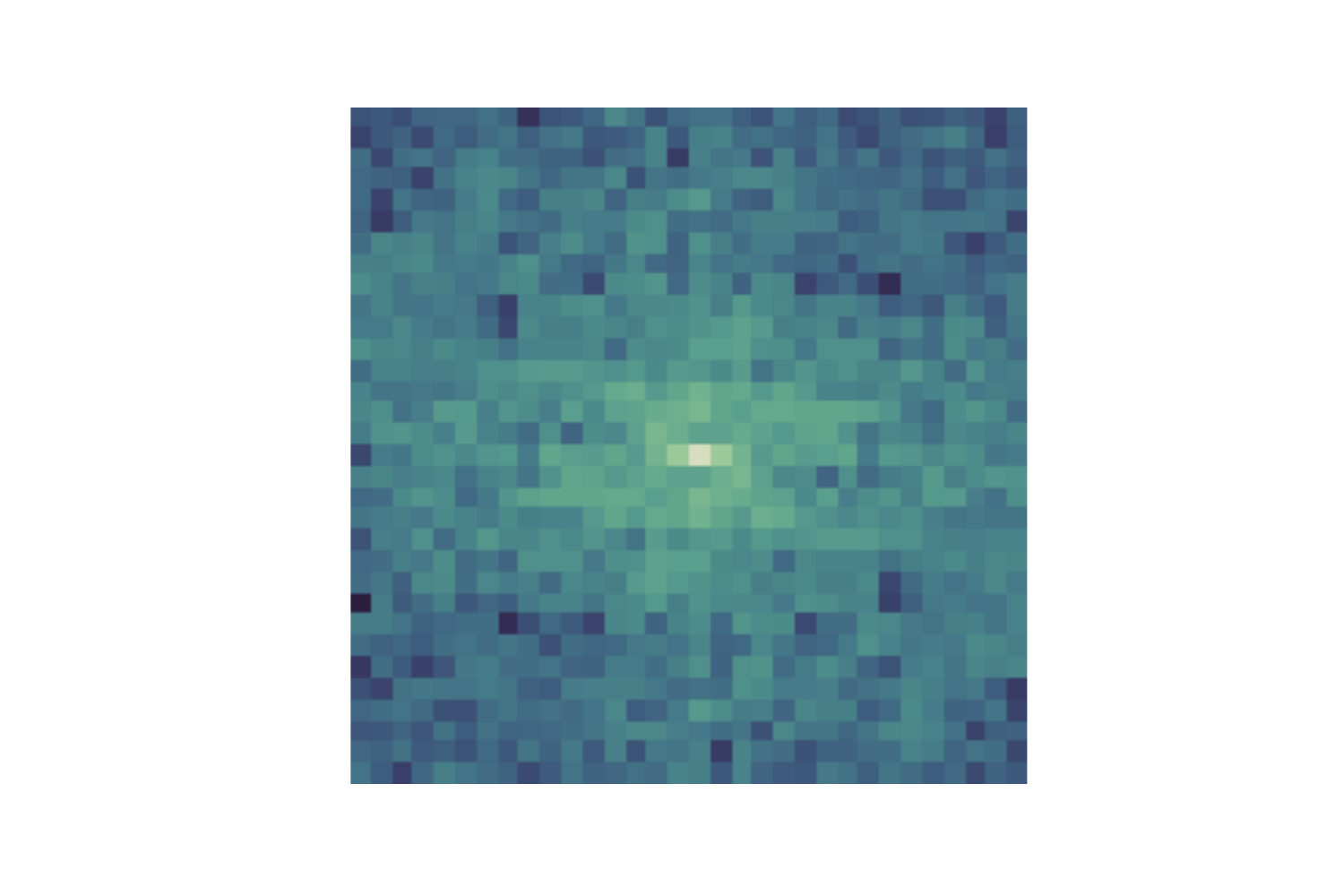} \\
    \parbox{\textboxsizeone}{\centering \small (a) SIREN* ($\omega_0=1$) \\ \footnotesize{$\bm W^{(0)} = I$}}
\end{subfigure}
\begin{subfigure}{\subfigsizeimrecsiren}
    \centering 
    \includegraphics[width=\subfigsizeimrecsiren]{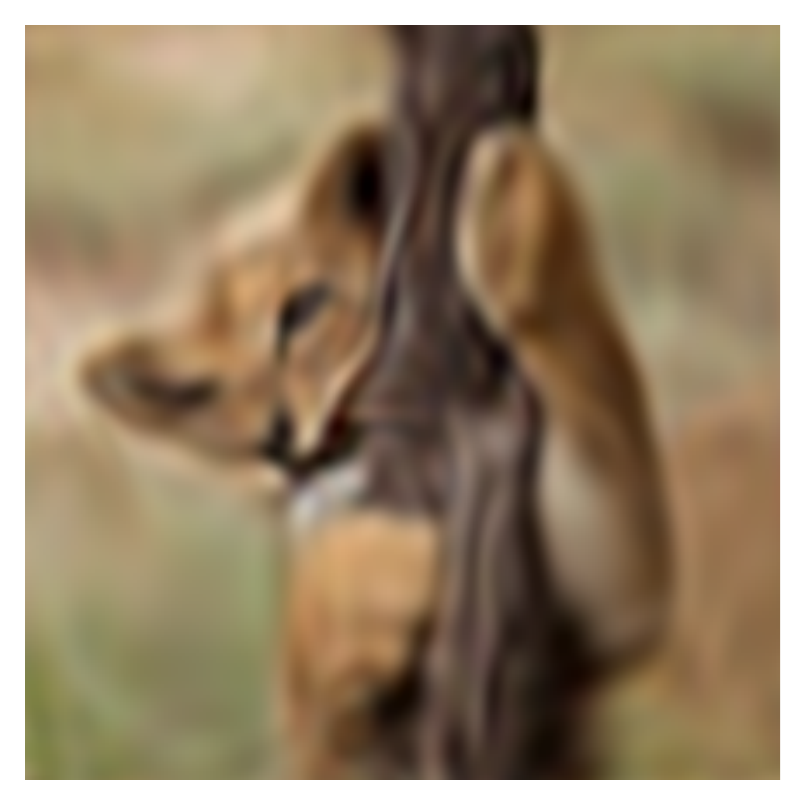} \\
    \includegraphics[width=\subfigsizeimrecsiren]{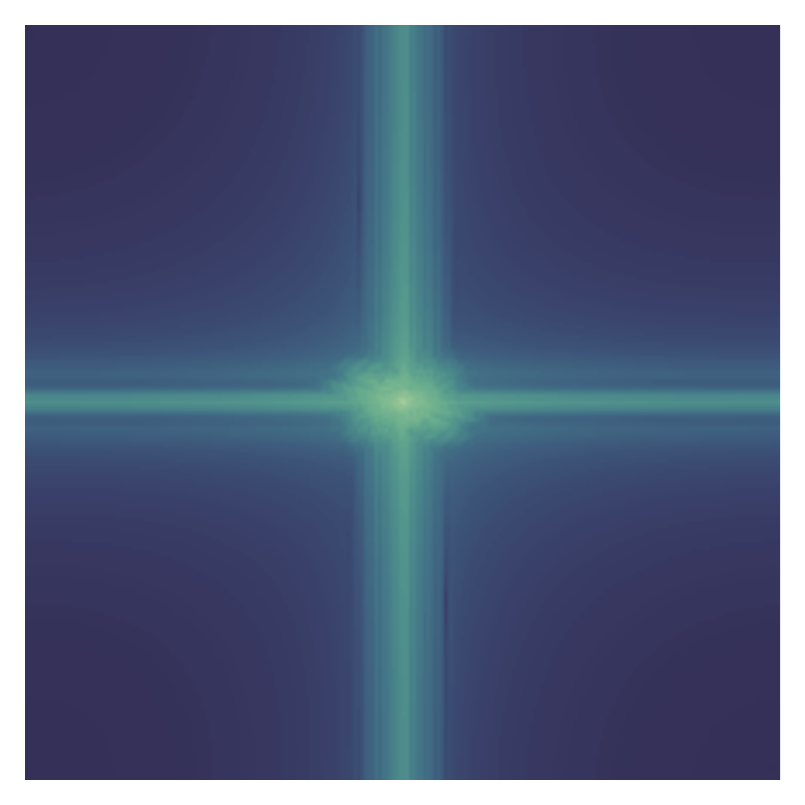} \\
    \includegraphics[width=\subfigsizeimrecsiren]{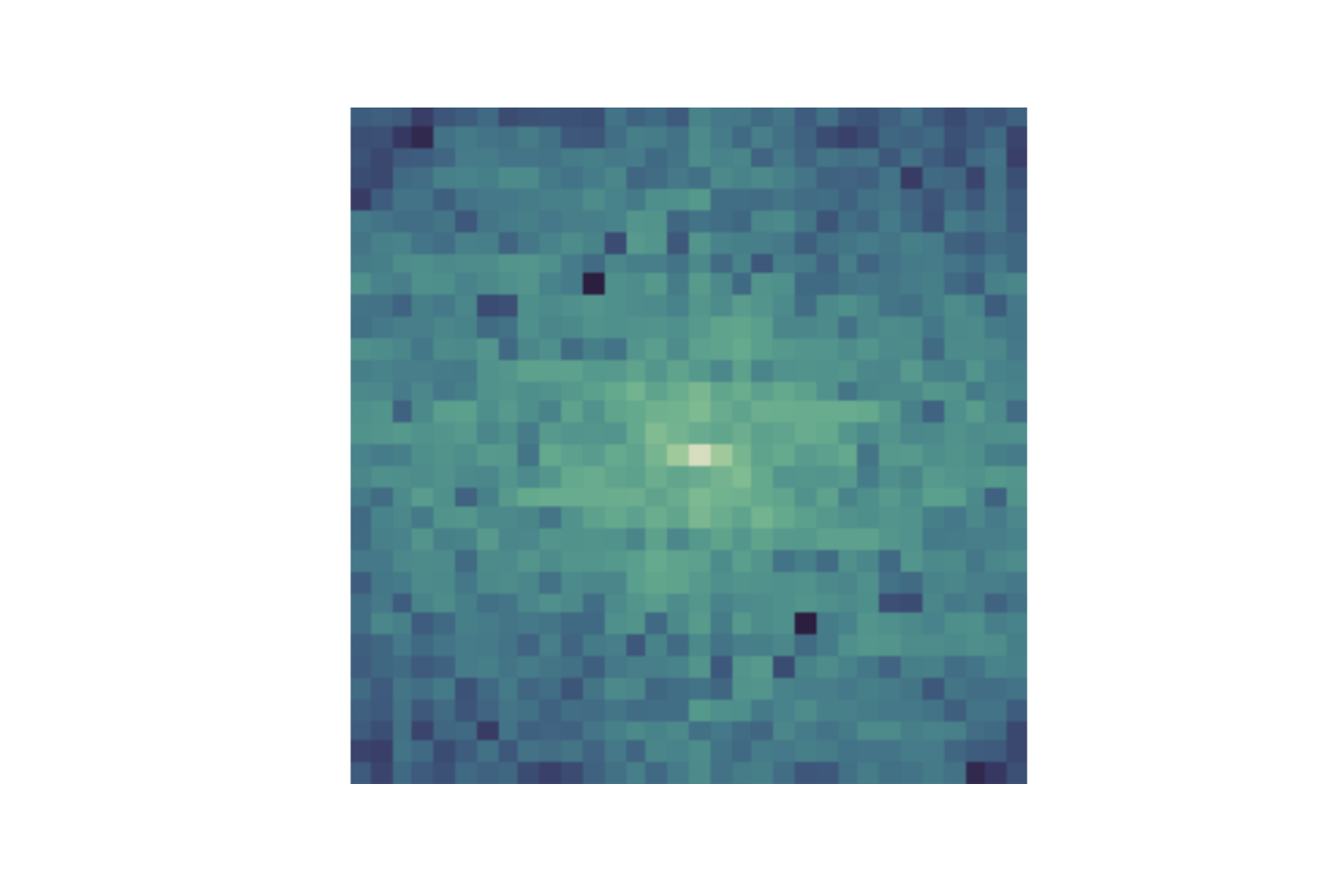} \\
    \parbox{\textboxsizeone}{\centering \small (b) SIREN* ($\omega_0=1$) \\ \footnotesize{$\bm W^{(0)} = 0.5 \times I$}}
\end{subfigure}
\begin{subfigure}{\subfigsizeimrecsiren}
    \centering 
    \includegraphics[width=\subfigsizeimrecsiren]{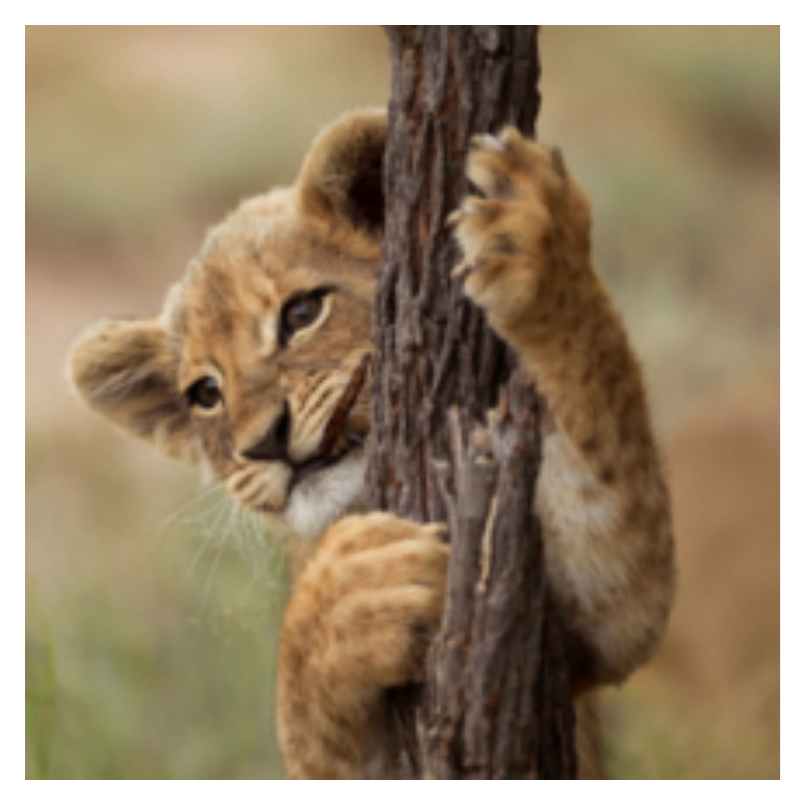} \\
    \includegraphics[width=\subfigsizeimrecsiren]{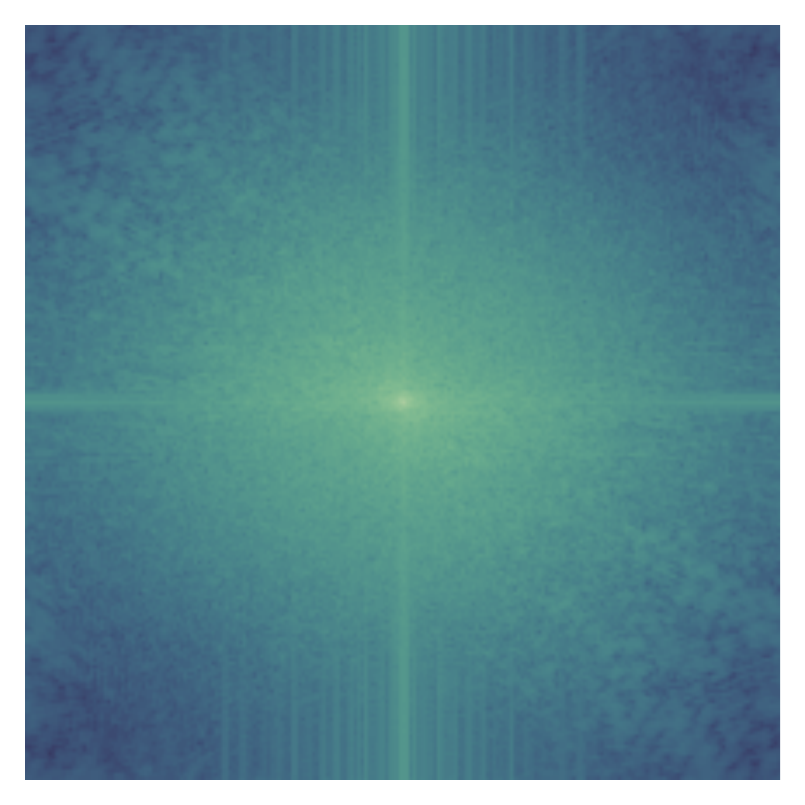} \\
    \includegraphics[width=\subfigsizeimrecsiren]{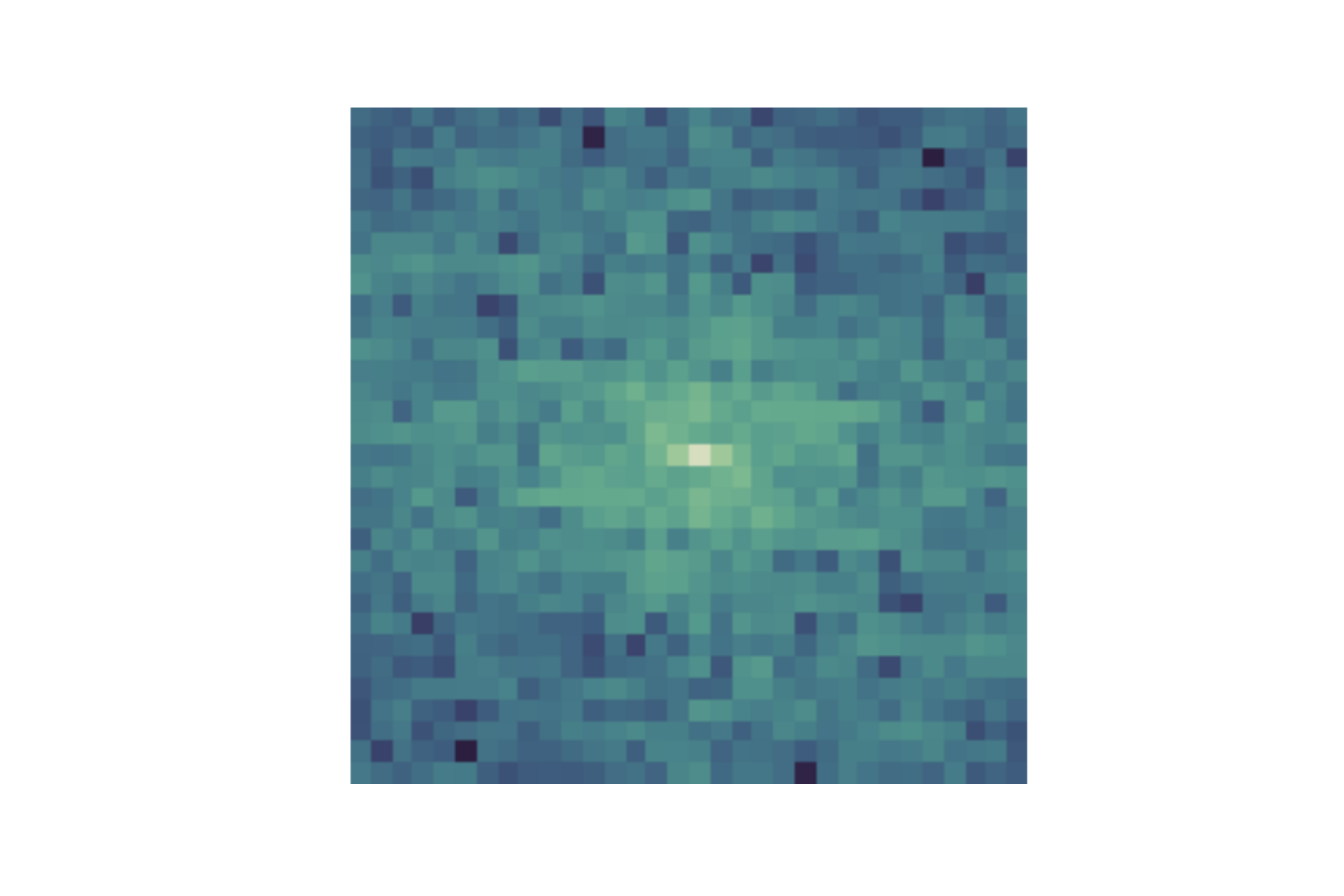} \\
    \parbox{\textboxsizeone}{\centering \small (c) \footnotesize SIREN* ($\omega_0=30$) \\ \footnotesize{$\bm W^{(0)} \in \R^{256 \times 2}$}}
\end{subfigure}
\begin{subfigure}{\subfigsizeimrecsiren}
    \centering 
    \includegraphics[width=\subfigsizeimrecsiren]{supp_figures/lion_512_gt__9.pdf} \\
    \includegraphics[width=\subfigsizeimrecsiren]{supp_figures/lion_512_gt_ft__2.pdf} \\
    \includegraphics[width=\subfigsizeimrecsiren]{supp_figures/gt_ft_zoomed.pdf}  \\
    \parbox{\textboxsizeone}{\centering \small (d) Ground Truth \\ \hfill}
\end{subfigure}

\caption{\textbf{Top row:} Image reconstruction with SIREN \cite{sitzmann2019siren} for different configurations. The notation * indicates that the parameters of the initial layer are subject to gradient updates during training. \textbf{Middle row:} Magnitude of the DFT of the reconstruction. \textbf{Bottom row:} The center crop of size $32 \times 32$ from the magnitude of the DFT of the reconstruction.}
\label{fig:siren_imperfect_trainable}
\end{figure}

For the sake of completeness, we also investigate the effect of having learnable parameters in $\gamma(\bm r)$ for FFNs. Specifically, we repeat this same experiment for the single frequency mapping of the main text, $\gamma(\bm r) = [\cos (2\pi f \bm r), \sin (2\pi f \bm r)]^T$, where we initialize the frequency variable $f$ as $f_0$. Note that $f$, now, is a trainable parameter. We denote this network as SFM\textsuperscript{*}. Similarly, we also train with an FFN\textsuperscript{*} with trainable $\bm\Omega_{i,j}$, randomly initialized as explained in Sec.2. \Cref{fig:ffn_imperfect_trainable} shows the results of these experiments with identical findings as in the main text and the SIRENs presented above.
\def \subfigsizeimrecsiren{2.8cm}

\begin{figure}
\centering
\begin{subfigure}{\subfigsizeimrecsiren}
    \centering 
    \includegraphics[width=\subfigsizeimrecsiren]{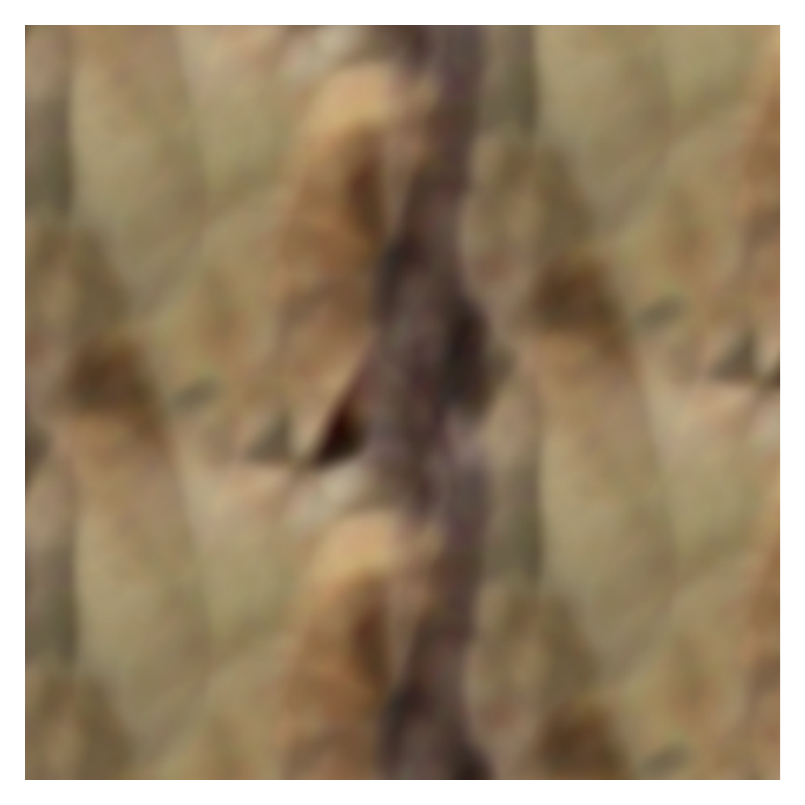} \\
    \includegraphics[width=\subfigsizeimrecsiren]{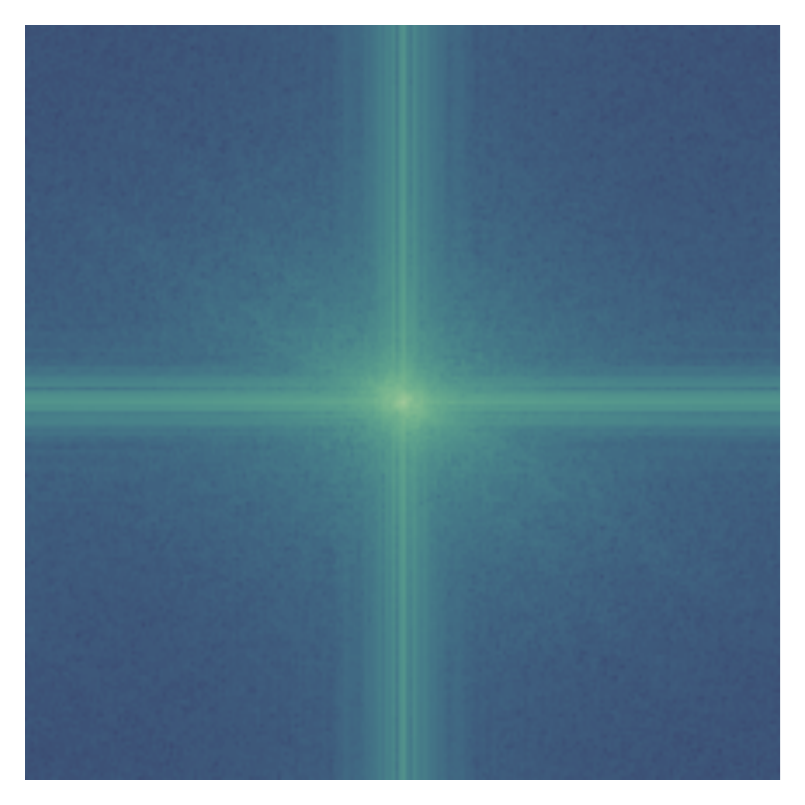} \\
    \includegraphics[width=\subfigsizeimrecsiren]{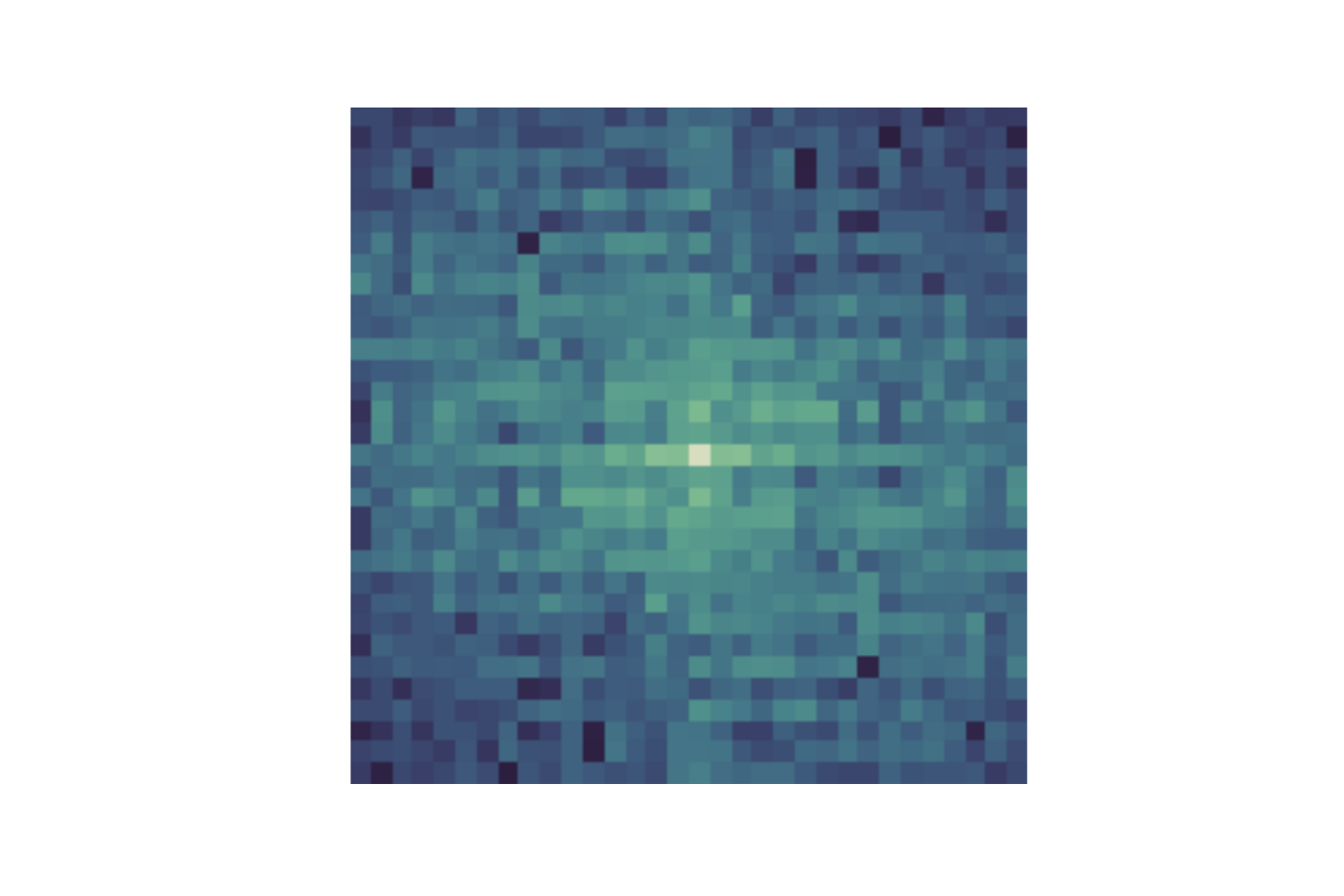} 
    \caption{SFM* ($f_0=1$)}
\end{subfigure}
\begin{subfigure}{\subfigsizeimrecsiren}
    \centering 
    \includegraphics[width=\subfigsizeimrecsiren]{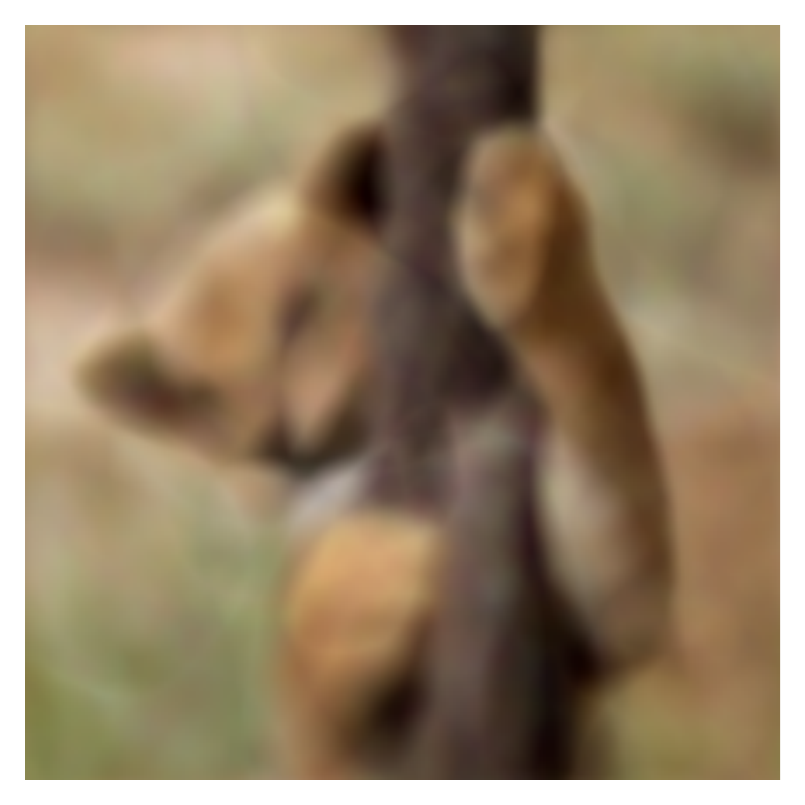} \\
    \includegraphics[width=\subfigsizeimrecsiren]{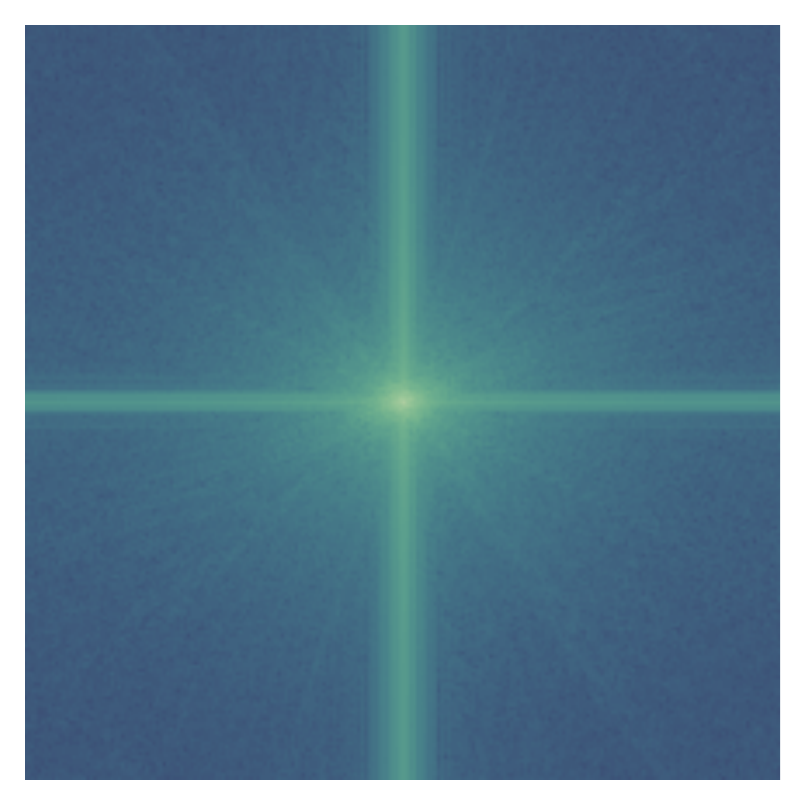} \\
    \includegraphics[width=\subfigsizeimrecsiren]{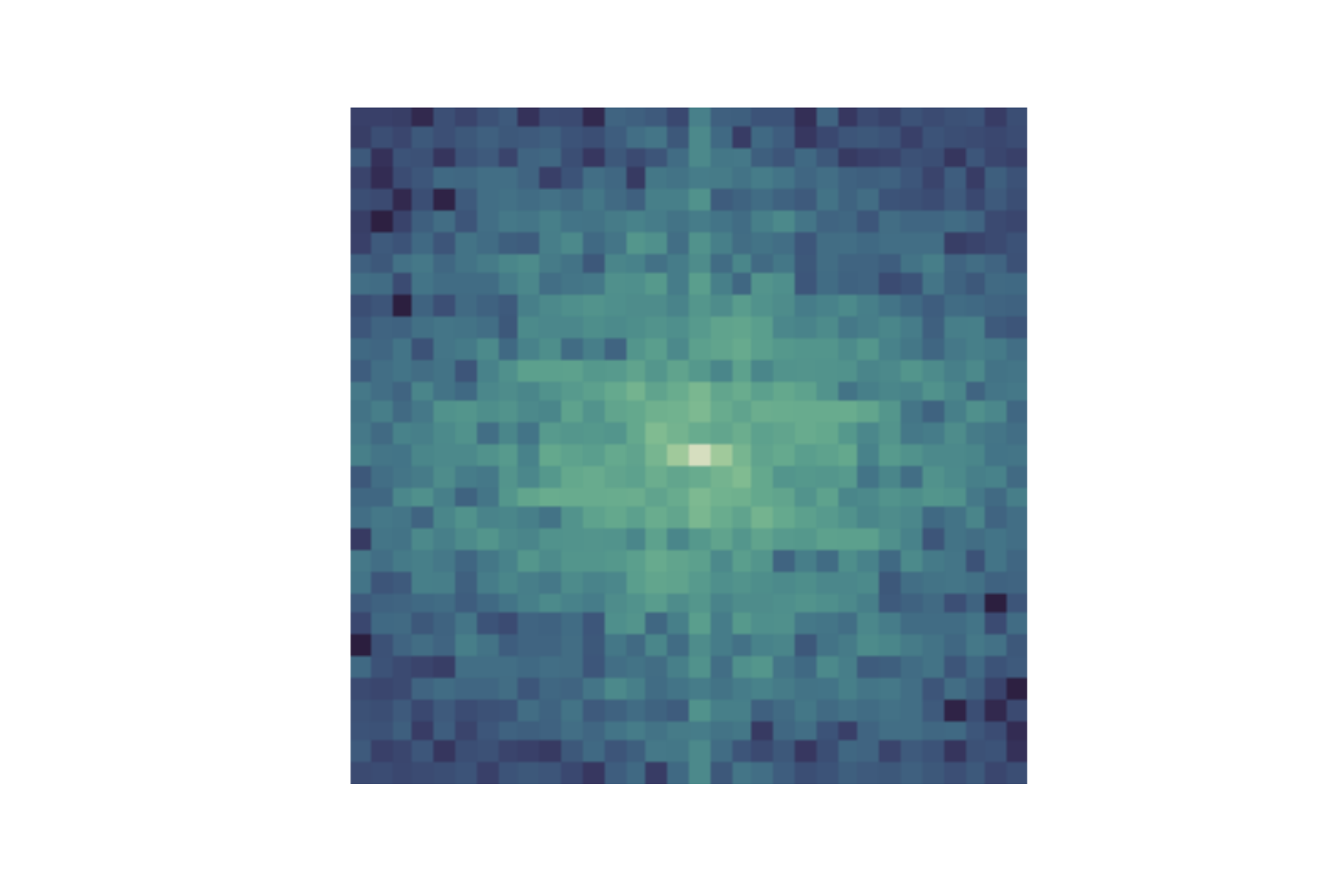} 
    \caption{SFM* ($f_0=0.5$)}
\end{subfigure}
\begin{subfigure}{\subfigsizeimrecsiren}
    \centering 
    \includegraphics[width=\subfigsizeimrecsiren]{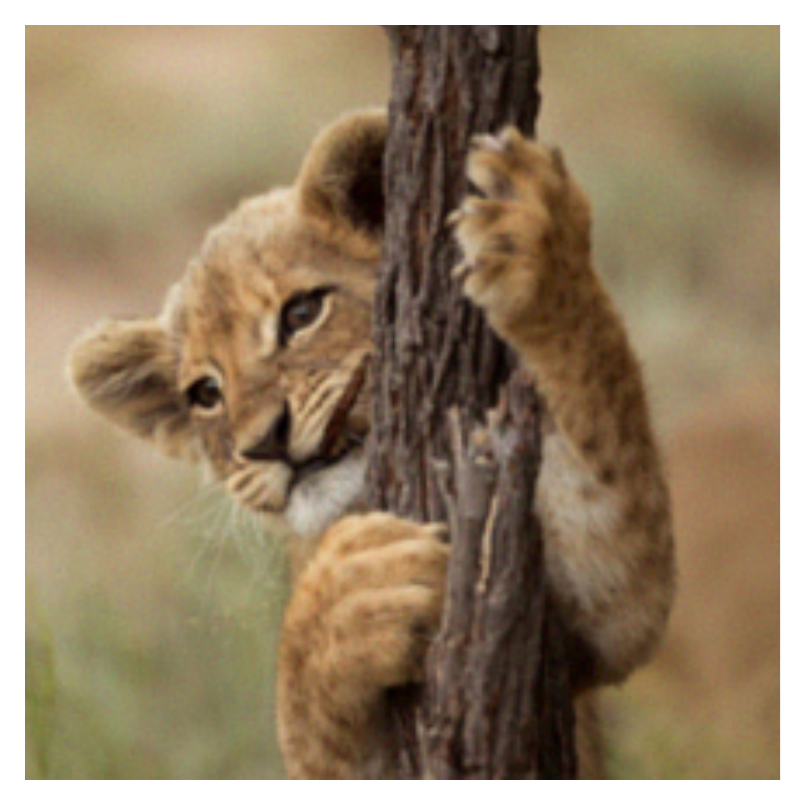} \\
    \includegraphics[width=\subfigsizeimrecsiren]{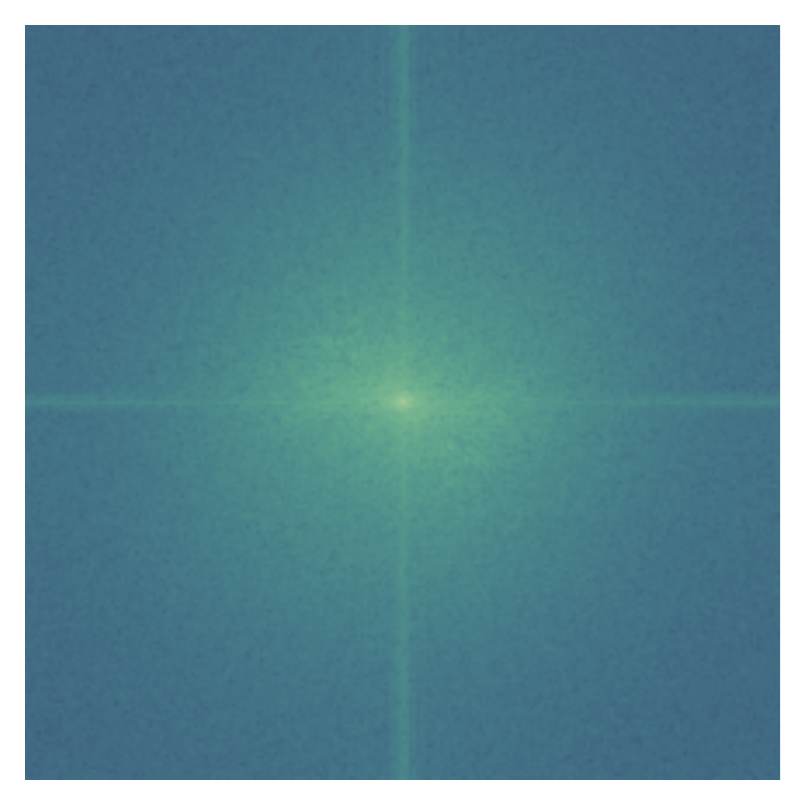} \\
    \includegraphics[width=\subfigsizeimrecsiren]{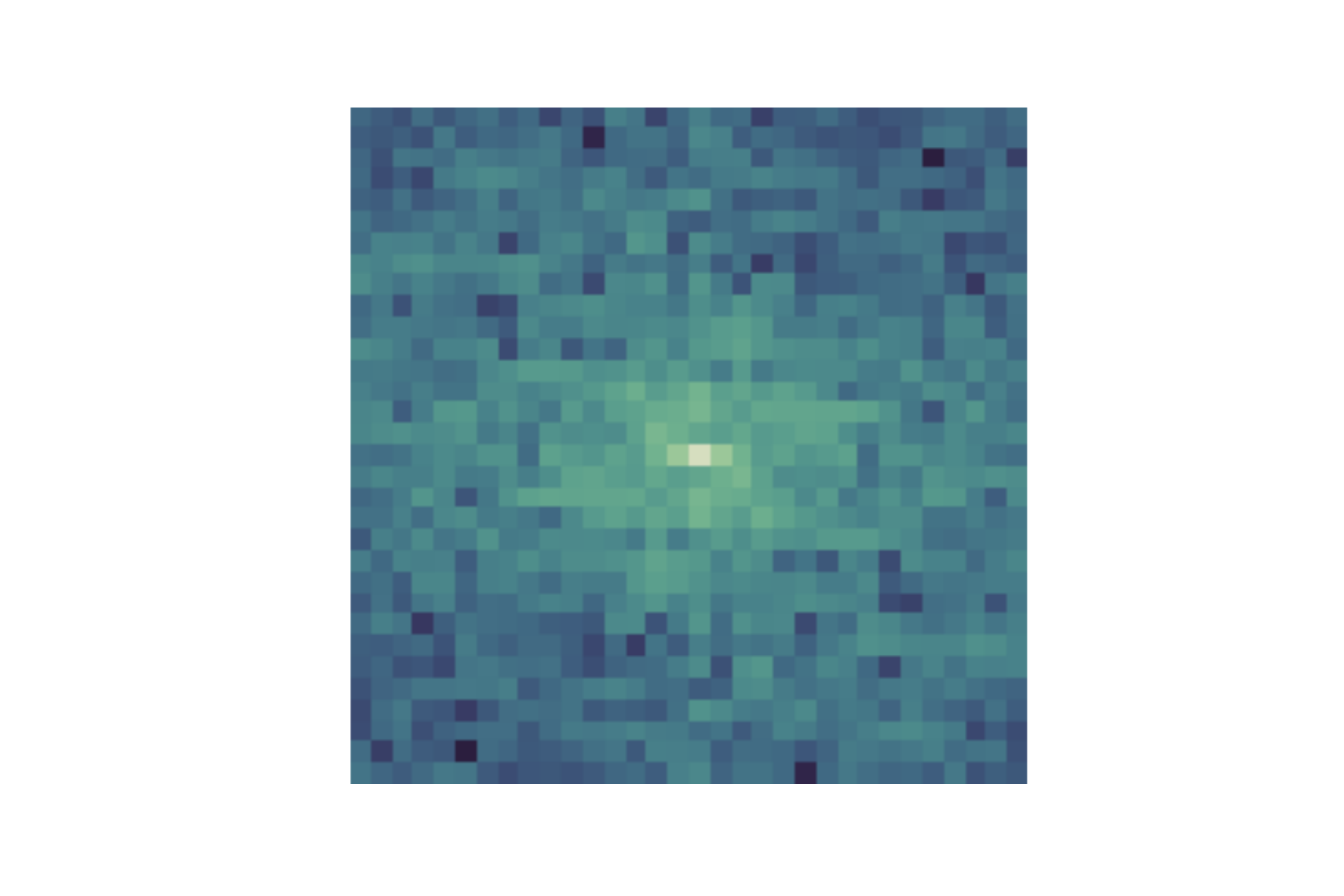} 
    \caption{FFN* ($\sigma=1$)}
\end{subfigure}
\begin{subfigure}{\subfigsizeimrecsiren}
    \centering 
    \includegraphics[width=\subfigsizeimrecsiren]{supp_figures/lion_512_gt__9.pdf} \\
    \includegraphics[width=\subfigsizeimrecsiren]{supp_figures/lion_512_gt_ft__2.pdf} \\
    \includegraphics[width=\subfigsizeimrecsiren]{supp_figures/gt_ft_zoomed.pdf} 
    \caption{Ground Truth}
\end{subfigure}
\caption{\textbf{Top row:} Image reconstruction with FFN \cite{tancik2020fourfeat} for different configurations. SFM stands for single frequency mapping as in Figure 2, i.e., $\gamma(\bm r, f) = [\cos (2\pi f \bm r), \sin (2\pi f \bm r)]^T$ where the frequency variable is initialized as $f_0$ and subject to gradient updates. In general, the notation * indicates learnable input mapping. \textbf{Middle row:} Magnitude of the DFT of the reconstruction. \textbf{Bottom row:} The center crop of size $32 \times 32$ from the magnitude of the DFT of the reconstruction.}
\label{fig:ffn_imperfect_trainable}
\end{figure}

\clearpage
\section{Aliasing}
\subsection{Experimental details}
For the experiment in Section 4.2, we train a SIREN with three fully connected layers: two hidden layers with dimension 128 followed by the output layer of dimension 1. All the networks presented in Fig.3 are trained for 2000 iterations using Adam \cite{kingma2014adam} with a learning rate of $1\times10^{-4}$. For the training data, we generate a single frequency signal $g(r)=\sin(2\pi\cdot23 r)$ on $128$ evenly spaced samples in the range $[0,1]$, i.e., sampled with a frequency of $f_s=128$ and test the learned representation $f_\theta(r)$ with samples from the same signal with $f_s=256$.

\subsection{Additional experiments}
For the sake of completeness, and to show that aliasing is prevalent across INR families, we repeat the same experiment using FFNs~\cite{tancik2020fourfeat} initialized with different $\sigma$. Recall that $\sigma$ determines the standard deviation of the the distibution of $\bm \Omega_{i,j}$ at initialization, i.e., $\bm\Omega_{i,j}\sim\mathcal{N}(0,\sigma^2)$. The results presented in \cref{fig:freq_aliasing_ffn} highlight that the same aliasing phenomenon observed for SIRENs in Fig.3 happens identically for FFNs.

\def \subfigsizealiasing{7cm}
\def \textboxsizealiasingrotated{3.5cm}
\def \textboxsizealiasing{7cm}

\begin{figure}[h!]
\centering
\rotatebox[origin=c]{90}{\makebox[\textboxsizealiasingrotated]{\centering $\sigma=100$}}
\begin{subfigure}{\subfigsizealiasing}
    \centering 
    \includegraphics[width=\subfigsizealiasing]{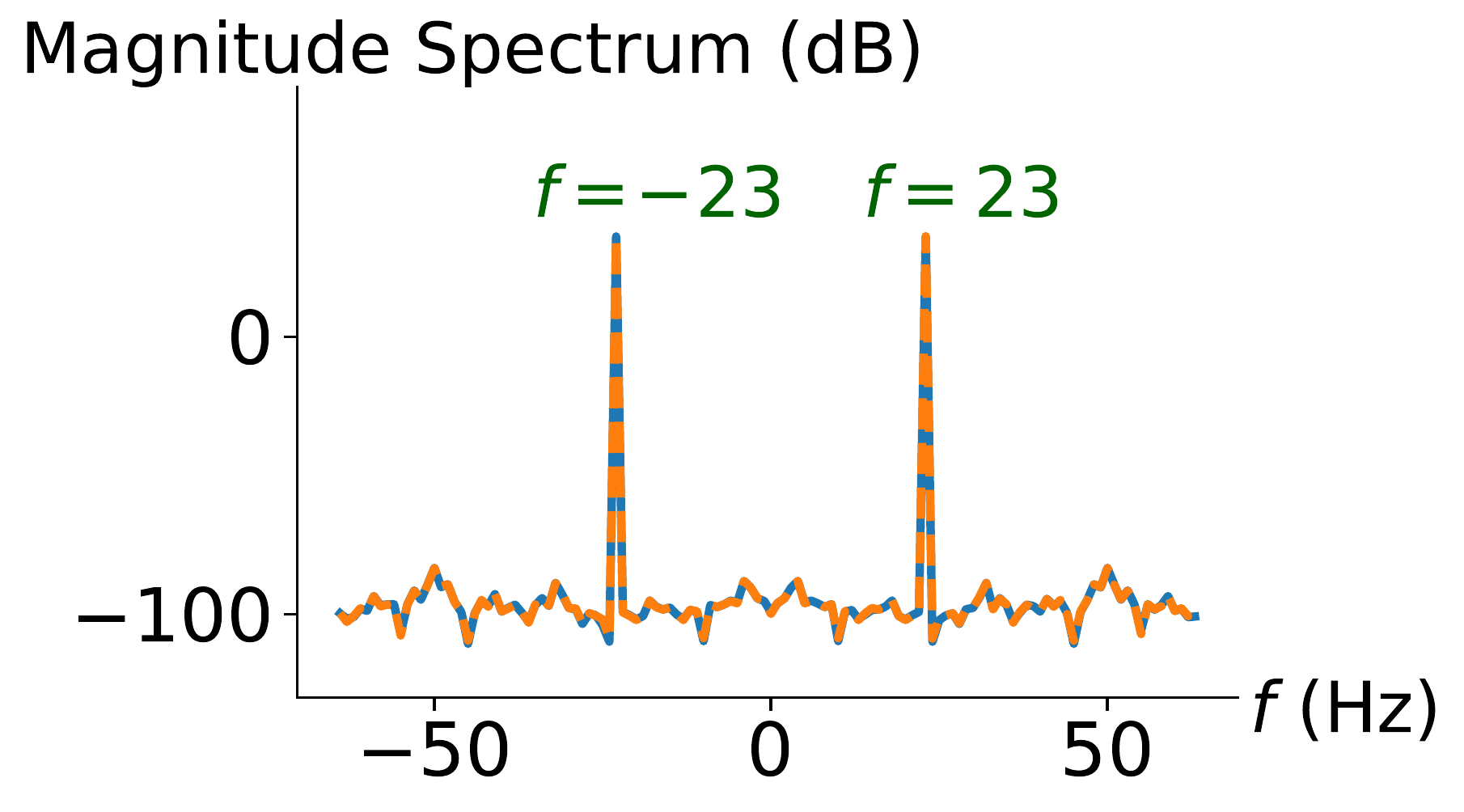}
\end{subfigure}
\begin{subfigure}{\subfigsizealiasing}
    \centering
    \includegraphics[width=\subfigsizealiasing]{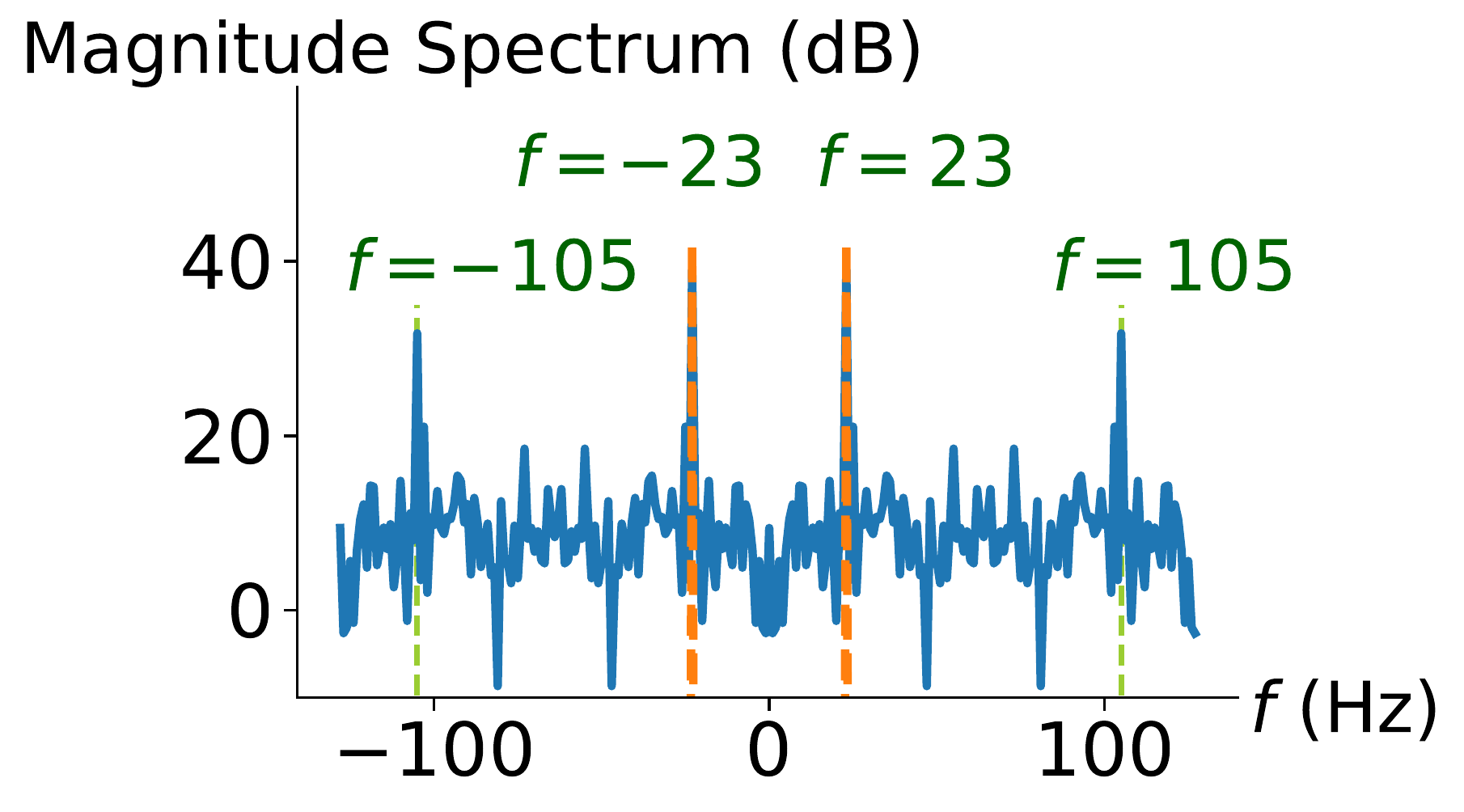}
\end{subfigure}
\vspace{0.5cm}
\\
\rotatebox[origin=c]{90}{\makebox[\textboxsizealiasingrotated]{\centering $\sigma=3$}}
\begin{subfigure}{\subfigsizealiasing}
    \centering 
    \includegraphics[width=\subfigsizealiasing]{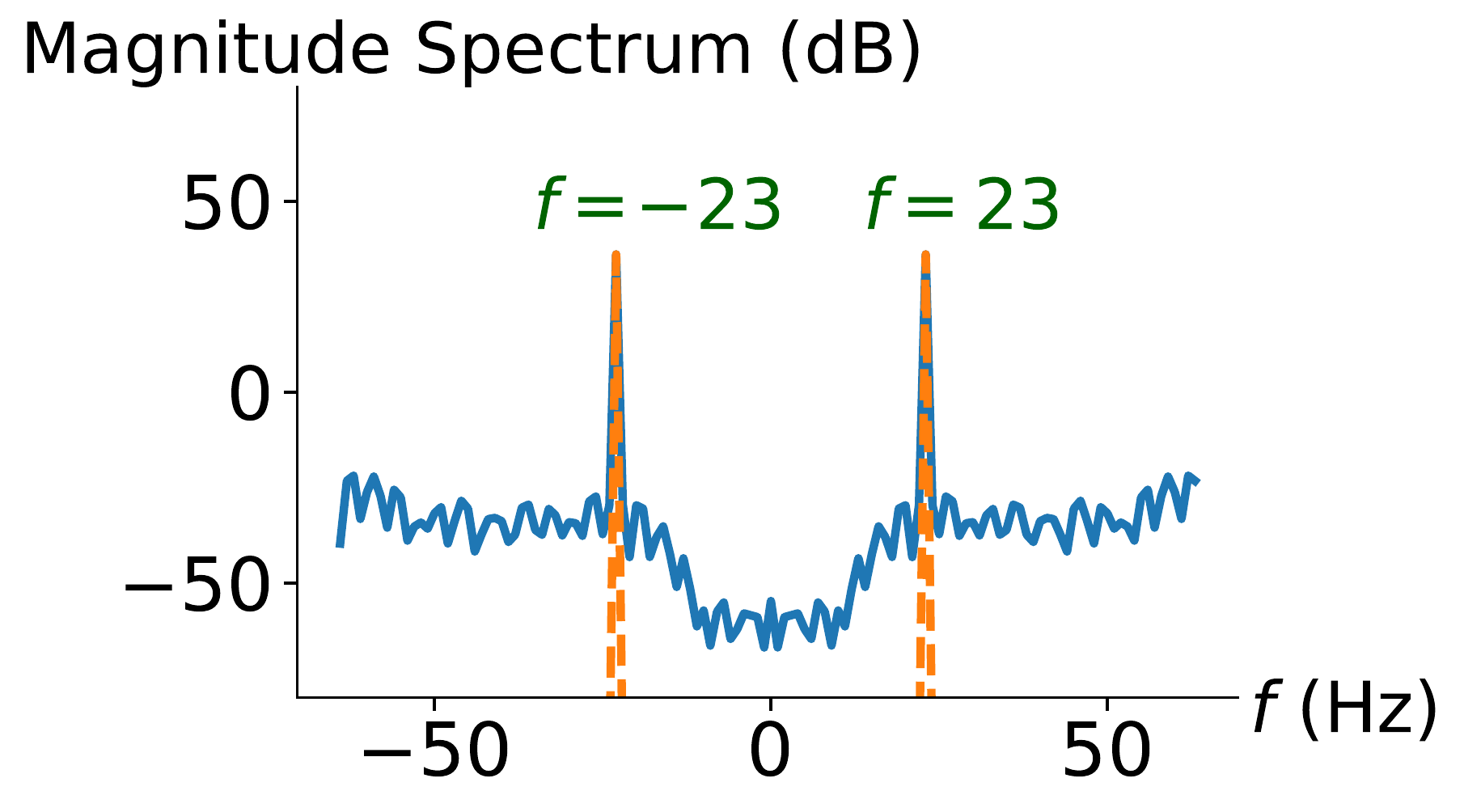}
\end{subfigure}
\begin{subfigure}{\subfigsizealiasing}
    \centering
    \includegraphics[width=\subfigsizealiasing]{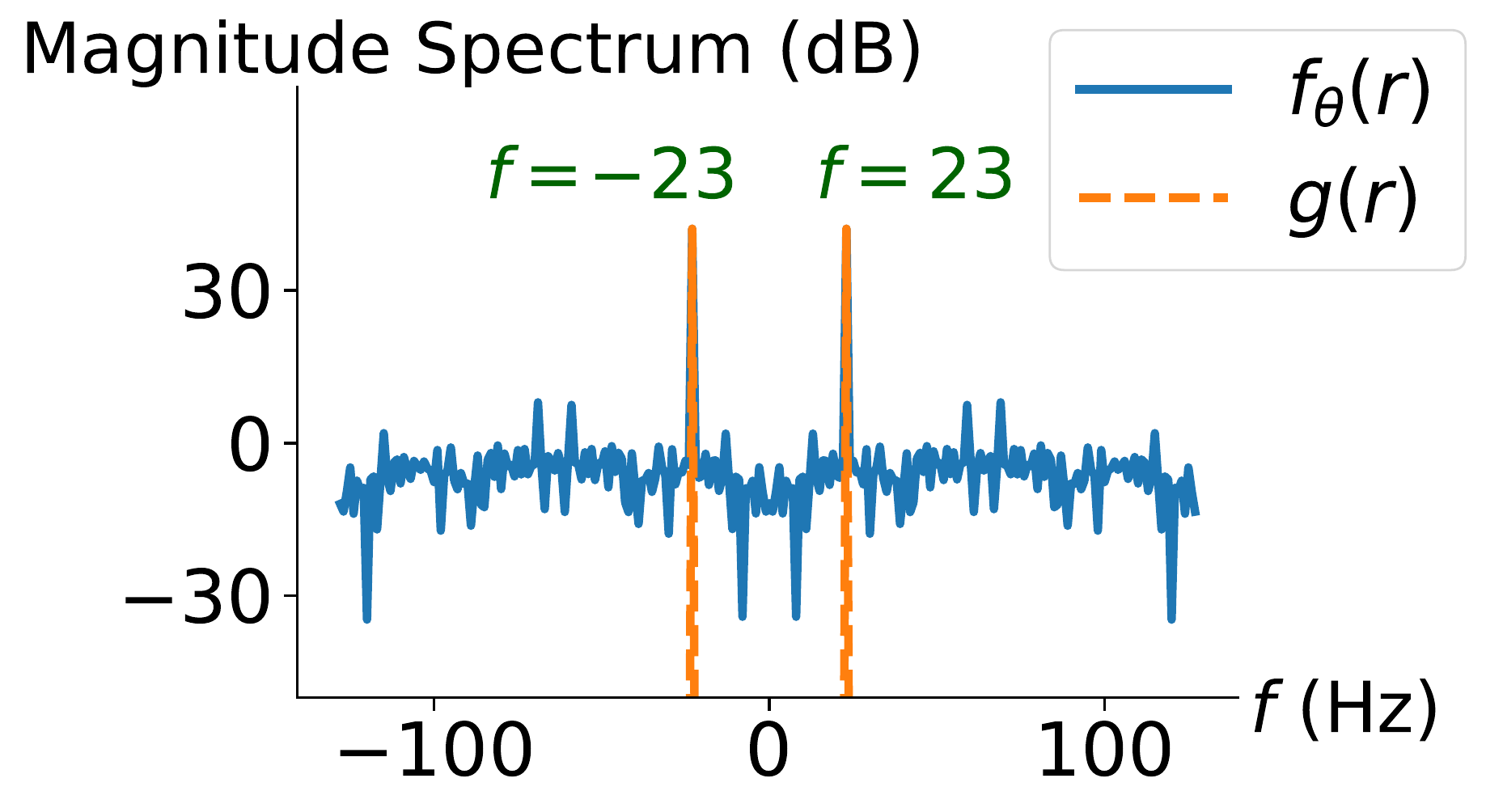}
\end{subfigure}
\vspace{0.25cm}\\
\makebox[\textboxsizealiasing]{\centering Sampling frequency $f_s = 128Hz$}
\makebox[\textboxsizealiasing]{\centering Sampling frequency $f_s = 256Hz$}
\caption{Magnitude of the spectrum of $g(r)=\sin(2\pi\cdot 23r)$ and its FFN~\cite{tancik2020fourfeat} reconstruction trained at $f_s = 128$ Hz. \textbf{Top row} shows $\sigma=100$, and \textbf{bottom row} $\sigma=3$. On the \textbf{left} the signals are sampled at $f_s=128$ Hz and on the \textbf{right} at $f_s=256$ Hz.}
\label{fig:freq_aliasing_ffn}
\end{figure}

\newpage
\section{NTK eigenfunctions as dictionary atoms}

\subsection{Estimation of eigenfunctions of the NTK}
As it is common in the kernel literature, in this work, we use the eigenvectors of the kernel Gram matrix to approximate the eigenfunctions of the NTK. Specifically, unless stated otherwise, in all our experiments we compute the Gram matrix of the NTK at any $\bm \theta_0$ using as samples the coordinates of all the pixels of an image, laid out on a grid of fixed resolution {($64\times64$)}. That is, we compute a Gram matrix of size {$64^2\times 64^2$}. To that end, we use the \texttt{empirical\_kernel\_fn} function from the \texttt{neural\_tangents} library~\cite{novakNeuralTangentsFast} which allows to compute this matrix using a batch implementation. Note that this operation can be computationally intense, scaling quadratically with the number of samples, but also quadratically with the number of outputs. For this reason, we decided to use INRs with a single output and work only with grayscale images. The results, however, are easily extensible to the multi-output setting.

Once we have the Gram matrix, we can perform its eigendecomposition and use the resulting eigenvectors to approximate the values of the eigenfunctions $\phi_j(\bm r)$ at the pixel coordinates. The inner products $\langle \phi_j, g_n\rangle$ are then easily approximated as
\begin{equation}
    \langle \phi_j, g_n\rangle\approx\sum_{i=1}^{64^2}\phi_j(\bm r_i)g_n(\bm r_i).
\end{equation}

\subsection{Training details}\label{sec:ntk_train}
In Section 5.2, we compare the generalization performance of several SIRENs ({four hidden layers with dimension 256 followed by the output layer of dimension 1})  with different initialization strategies. To that end, we train each of these networks to reconstruct 100 validation images from the CelebA dataset using half of the pixels of the images for training. The training pixel locations are random, but we use the same across all validation images. Generalization performance, is then tested using the remaining half of the pixels.  To be consistent with the empirical protocol proposed in \cite{tancik2020meta}, we use full batch Adam \cite{kingma2014adam} with a learning rate of {$10^{-4}$} for the randomly initialized weights and use full batch gradient descent with learning rate {$10^{-2}$} for the meta learned weights, which they reported to be the optimal choice of optimizers and learning rates for each individual case.

\Cref{fig:siren_ntk_train} shows the evolution of the average training and test curves for each of these networks. As we can see, the networks with a better energy concentration in Fig. 4, are the networks that train faster, and reach the best test performances.

\def \subfigsizeperf{8cm}
\begin{figure}[ht!]
    \centering
    \begin{subfigure}{\subfigsizeperf}
        \centering
        \includegraphics[width=\subfigsizeperf]{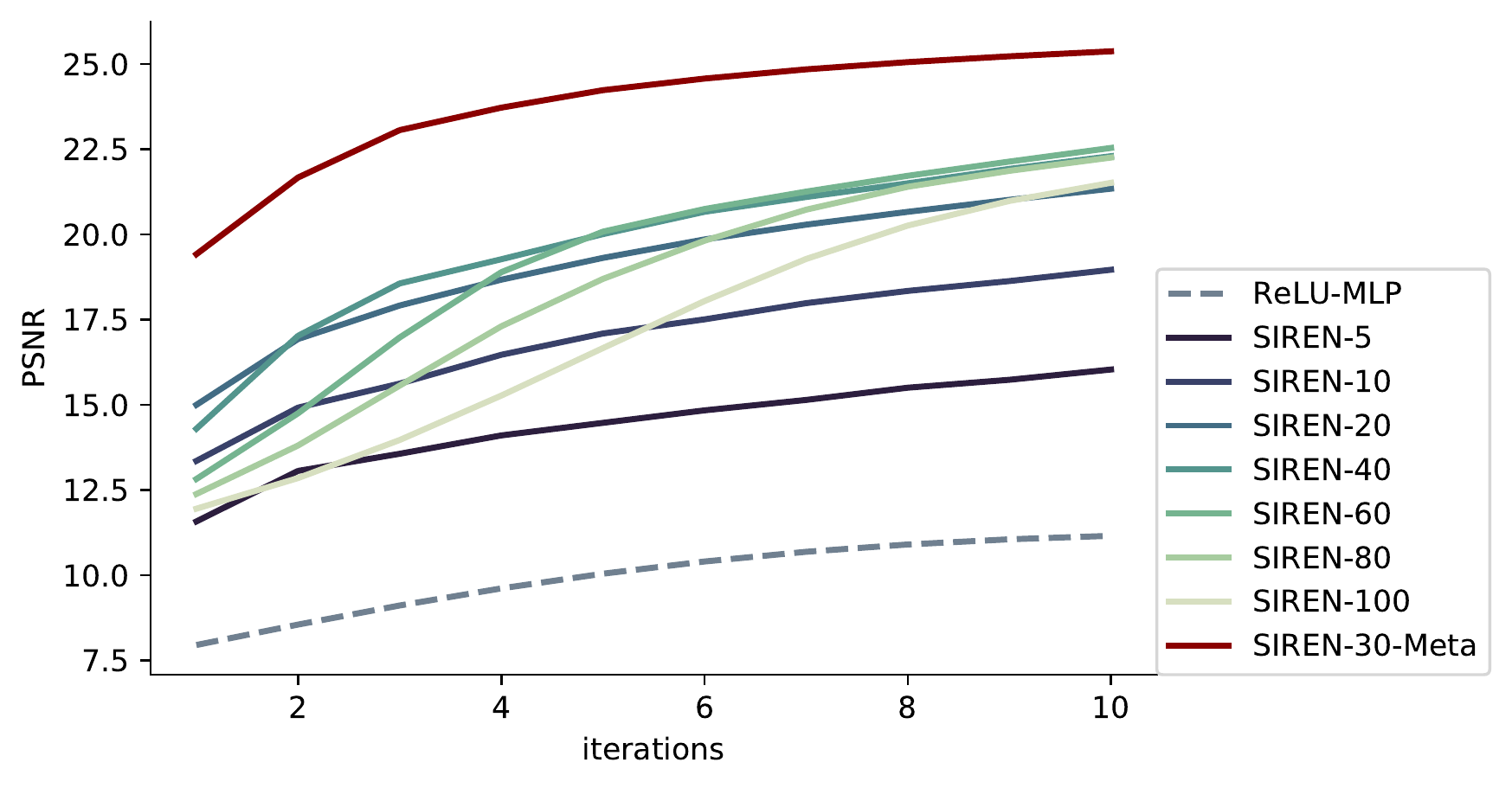}
        \caption{Test}
    \end{subfigure}
    \begin{subfigure}{\subfigsizeperf}
        \centering
        \includegraphics[width=\subfigsizeperf]{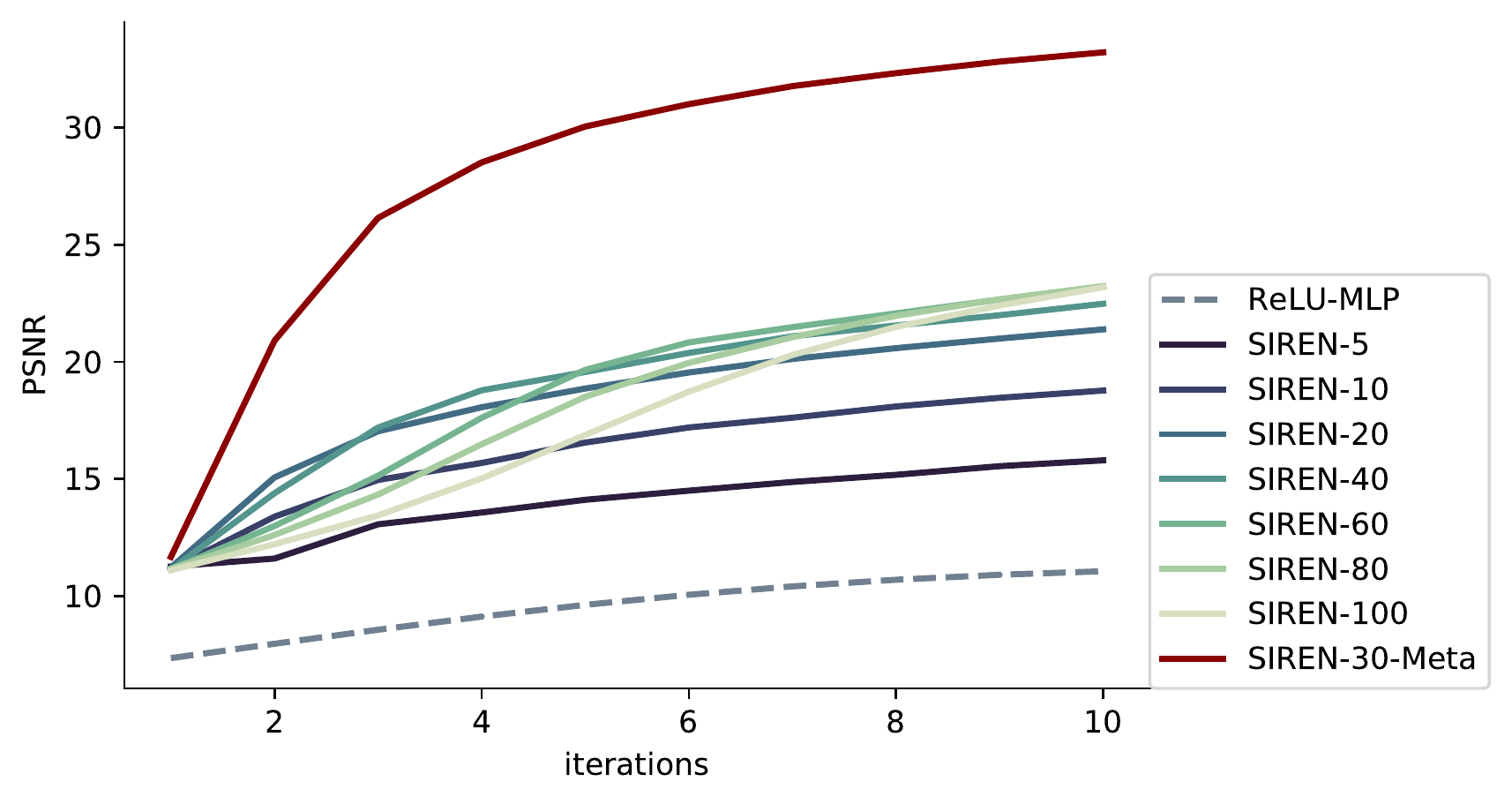}
        \caption{Training}
    \end{subfigure}
    \caption{The evolution of the reconstruction performance (average PSNR) of different representations for the first 10 training iterations when trained on 100 validation images from CelebA dataset. The numbers for different SIREN instances indicate the value of $\omega_0$. Meta indicates the learned initialization, whereas all the other networks are initialized randomly in accordance with the original implementation~\cite{sitzmann2019siren}.
    }
    \label{fig:siren_ntk_train}
\end{figure}

\subsection{Experiments on additional networks}

For completeness, we also provide the results of these experiments using FFNs ({a sinusoidal mapping of size 256 followed by three hidden layers with dimension 256 followed by the output layer of dimension 1}), instead of SIRENs in \cref{fig:ntk_ffns}. Again, we observe that those networks which have an energy profile more concentrated on the largest eigenvalues perform much better.

\begin{figure}[ht!]
    \centering
    \includegraphics[width=0.7\linewidth]{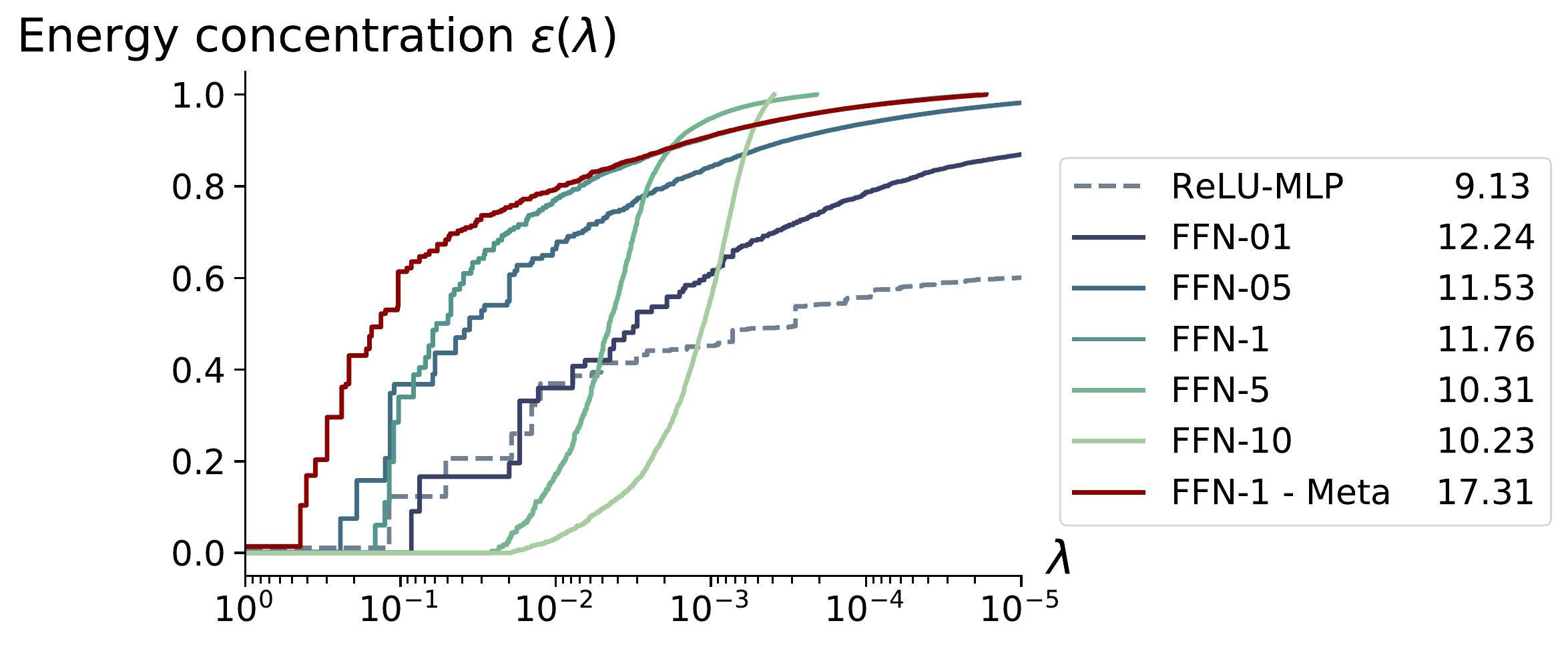}
    \caption{Average energy concentration of $100$ validation images from CelebA on subspaces spanned by the eigenfunctions of the empirical NTK associated to eigenvalues greater than a given threshold. Legend shows the average test PSNR after training to reconstruct those images from 50\% randomly selected pixels for 3 iterations. The value following FFN specifies the $\sigma$ parameter of the given network. }
    \label{fig:ntk_ffns}
\end{figure}


\section{Meta-learning experiment}

\subsection{Experimental details}\label{sec:ntk_setting}

Our meta-learning experiments consist of two phases: A first pre-training phase, in which we use MAML, {to meta-learn a good initialization from $5,000$ training images from CelebA using a learning rate of $10^{-5}$ as indicated in \cite{tancik2020meta} and $5000$ meta-iterations. In our experiments, we use a SIREN with four hidden layers with dimension 256 followed by the output layer of dimension 1, randomly initialized prior to meta-learning using $\omega_0=30$.} After pretraining, we finetune the networks starting at the meta-learned weights using the training protocol described in \cref{sec:ntk_train}. 

To estimate the eigenfunctions of the NTK at the meta-learned weights we use the experimental setting described in \cref{sec:ntk_setting}. 

\subsection{Experiments with an additional meta-learning algorithm}
We repeat the same experiments by replacing the meta-learning algorithm with Reptile. \cref{fig:energy_reptile} shows the resulting energy concentration plot and \cref{fig:eigvec_singletask} shows the eigenfunctions of the NTK at the meta-learned weights using Reptile. As we can see, the results agree completely with those found using MAML. This suggests that the reshaping effect of meta-learning on the NTK is a general phenomenon, which might be induced using multiple algorithms.

\begin{figure}[ht!]
    \centering
    \includegraphics[width=0.7\linewidth]{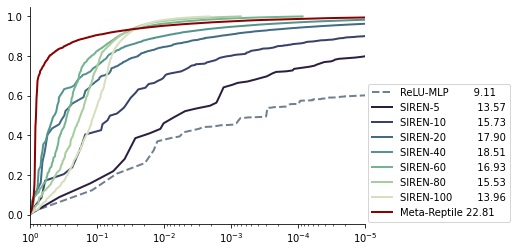}
    \vspace{-1em}
    \caption{Average energy concentration of $100$ validation images from CelebA on subspaces spanned by the eigenfunctions of the empirical NTK associated to eigenvalues greater than a given threshold. Legend shows the average test PSNR after training to reconstruct those images from 50\% randomly selected pixels. The meta-learned weights are computed using Reptile.}
    \vspace{-1em}
    \label{fig:energy_reptile}
\end{figure}

\subsection{Experiments on additional networks}

For completeness, we also provide the results of these experiments using an FFN ({a sinusoidal mapping of size 256 followed by three hidden layers with dimension 256 followed by the output layer of dimension 1) instead of a SIREN. Prior to meta-learning the input mapping is initialized using $\sigma=1$.} \cref{fig:ntk_ffns} shows these additional results, where we see that meta-learning does also improve the energy concentration of the validation images on the principal eigenspace of the NTK for the FFNs. Performance does also improve significantly in this case.

\subsection{Comparison with standard training}

Finally, as a baseline, we provide a comparison between standard training and meta-learning, both in terms of the dynamics of the NTK, and fine-tuning performance. Specifically, we repeated the same experimental pipeline described before, where instead of pretraining using MAML, {we pretrain a SIREN using Adam} to reconstruct one training image in CelebA. 

\begin{figure}[ht!]
    \centering
    \includegraphics[width=0.7\linewidth]{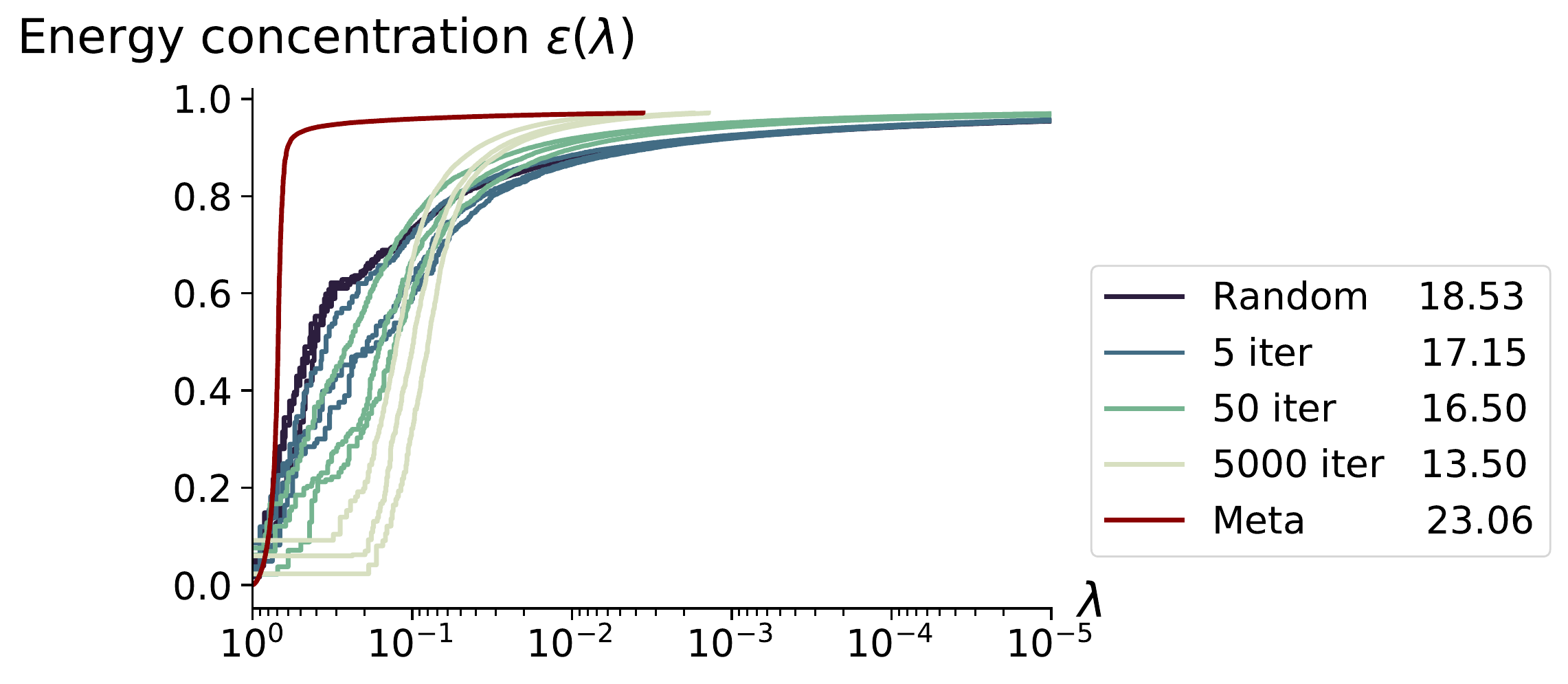}
    \caption{Average energy concentration of $100$ validation images from CelebA on subspaces spanned by the eigenfunctions of the empirical NTK associated to eigenvalues greater than a given threshold. Legend shows the average test PSNR after training to reconstruct those images from 50\% randomly selected pixels. The number of iterations specify those of pretraining on the single task.}
    \label{fig:ntk_single_task}
\end{figure}
As we can see in \cref{fig:ntk_single_task} and \cref{fig:eigvec_singletask} has a signficant impact on the NTK. As described before in \cite{kopitkovNeuralSpectrum2020}, training a neural network to reconstruct a signal transforms the eigenfunctions of the final NTK such that they look like the target signals. Surprisingly, though, we are seeing that for this particular set of tasks, fine-tuning from those initializations results in worse performance than starting from scratch, i.e., there is a negative transfer between different tasks.

\begin{figure}[ht!]
    \centering
    \begin{subfigure}{\subfigsizeperf}
        \centering
        \includegraphics[width=\subfigsizeperf]{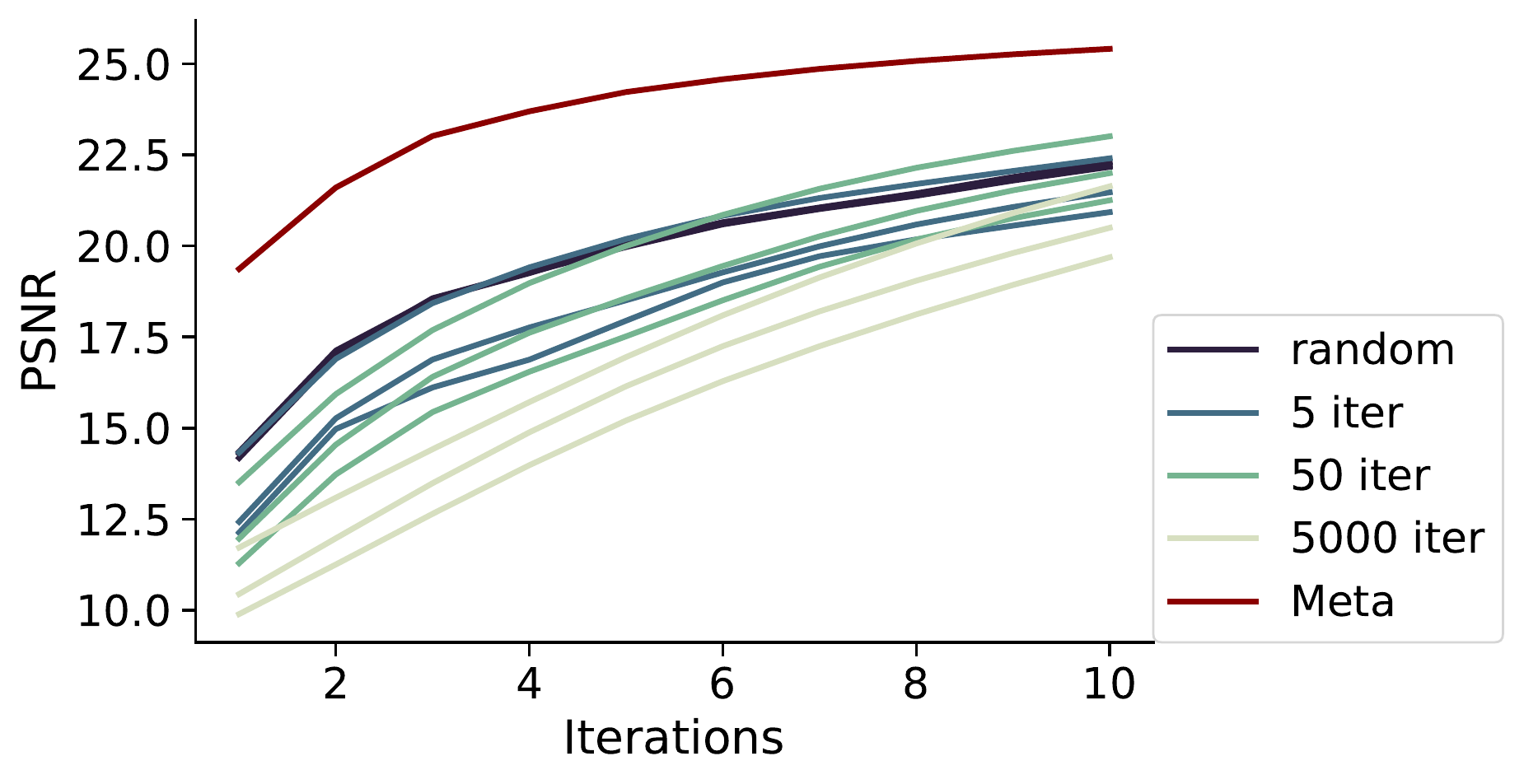}
        \caption{Test}
    \end{subfigure}
    \begin{subfigure}{\subfigsizeperf}
        \centering
        \includegraphics[width=\subfigsizeperf]{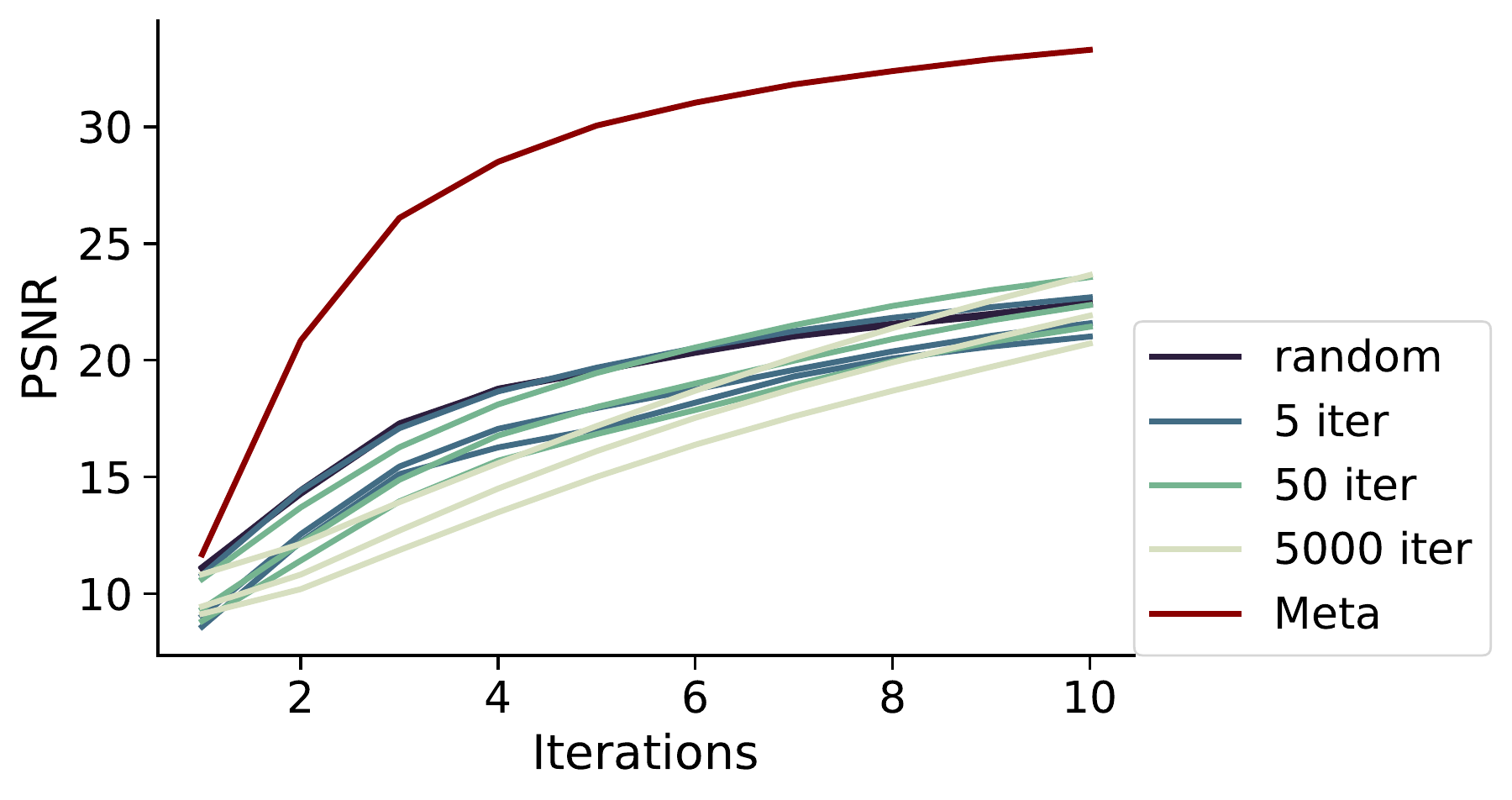}
        \caption{Training}
    \end{subfigure}
    \caption{The evolution of the reconstruction performance (PSNR) of a SIREN, where the parameters of the network are (i) initialized randomly, (ii) pretrained on a single image for different number of training iterations, and (iii) meta-learned. Note that three different realizations of the networks with pretrained weights on a single image are provided for each different number of training iterations.}
    \label{fig:train_single_task}
\end{figure}

We can understand this phenomenon using the signal dictionary analogy, which is summarized by \cref{fig:ntk_single_task}. As we can see, the more we pretrain the networks to reconstruct a specific training image, the more the energy profile of the validation set shifts to the smallest eigenvalues of the pretrained kernel, i.e, the learned dictionary has a worse compression performance than the randomly initialized NTK. As a result, as illustrated in \cref{fig:train_single_task}, we see that the more pretrained networks take longer to converge when fine-tuned, and reach worse generalization performances.

Inspecting the eigenfunctions of these different networks, and interpreting them as dictionary atoms, gives us one last piece of intuition to understand the dynamics of meta-learning. Indeed, as we can see in \cref{fig:eigvec_singletask} the eigenfunctions of the standardly trained networks do really look like the target signals, having very specific crisp edges and textures, which might not directly appear on other images. That is, the dictionaries of these networks are overfitted to their specific pretraining task. On the other hand, meta-learning can leverage much more data to construct a dictionary whose atoms can be easily composed to create face images. We believe that understanding the dynamics of this process will be a very fruitful avenue for future research.

\section{Performance on NTK eigenfunctions}
In this section we empirically demonstrate that it is easier for a network to learn the NTK eigenfunctions corresponding to larger eigenvalues. \Cref{fig:eigenfunctions} shows the PSNR reached after training the network for 2000 iterations with the target set to the NTK eigenfunction at initialization with the corresponding index. Even though the meta-learned weights are used in this experiment, this phenomenon applies to other networks and weights as well and is previously discussed in \cite{ortiz2021can}.

\begin{figure}[ht!]
    \centering\includegraphics[width=0.6\linewidth]{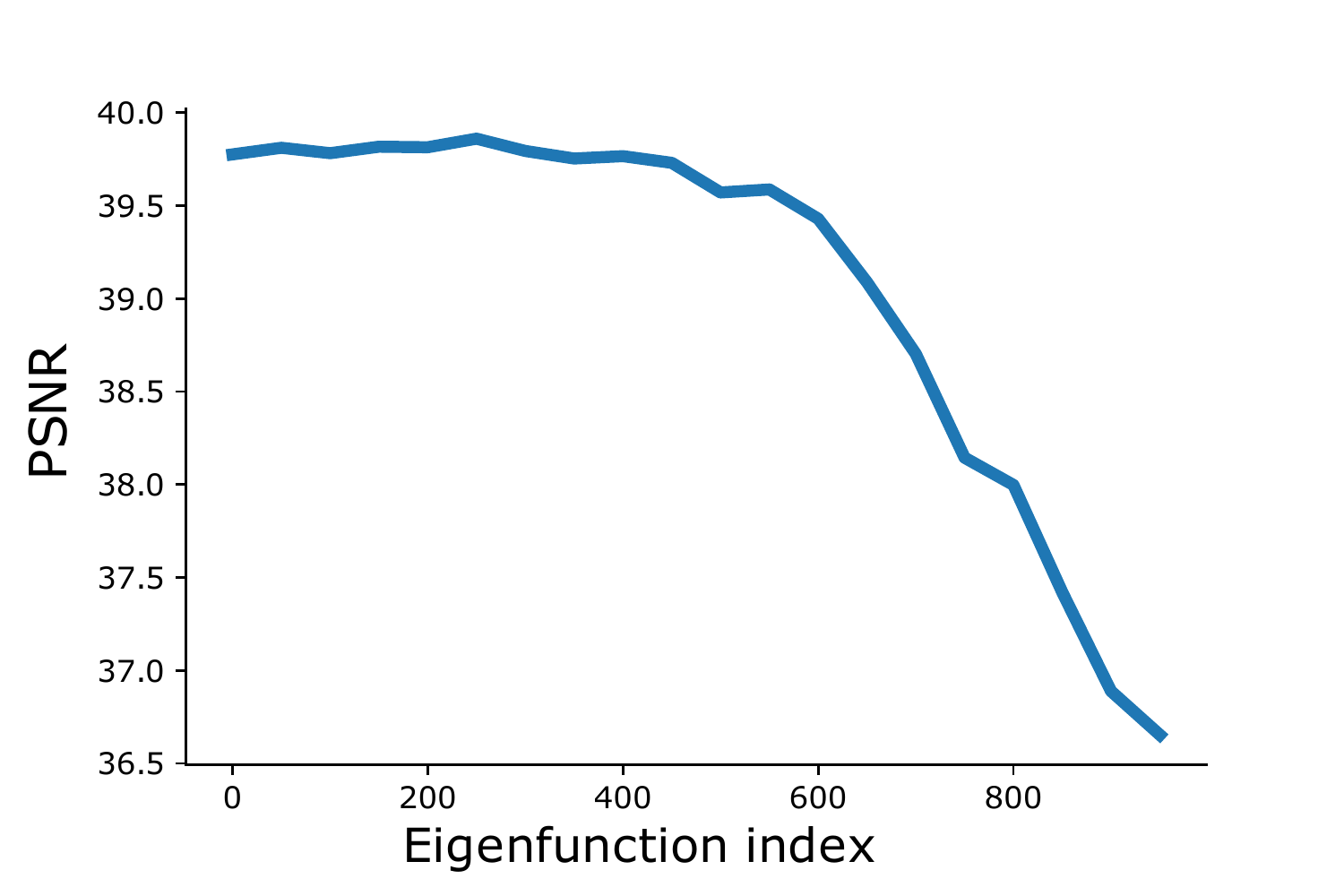}
    \caption{PSNR performance of a SIREN trained on increasing eigenfunctions of its NTK at initialization.}
    \label{fig:eigenfunctions}
\end{figure}

\def \subfigsizeone{1.4cm}
\def \textboxsizeone{1.9cm}

\begin{figure*}[t!]
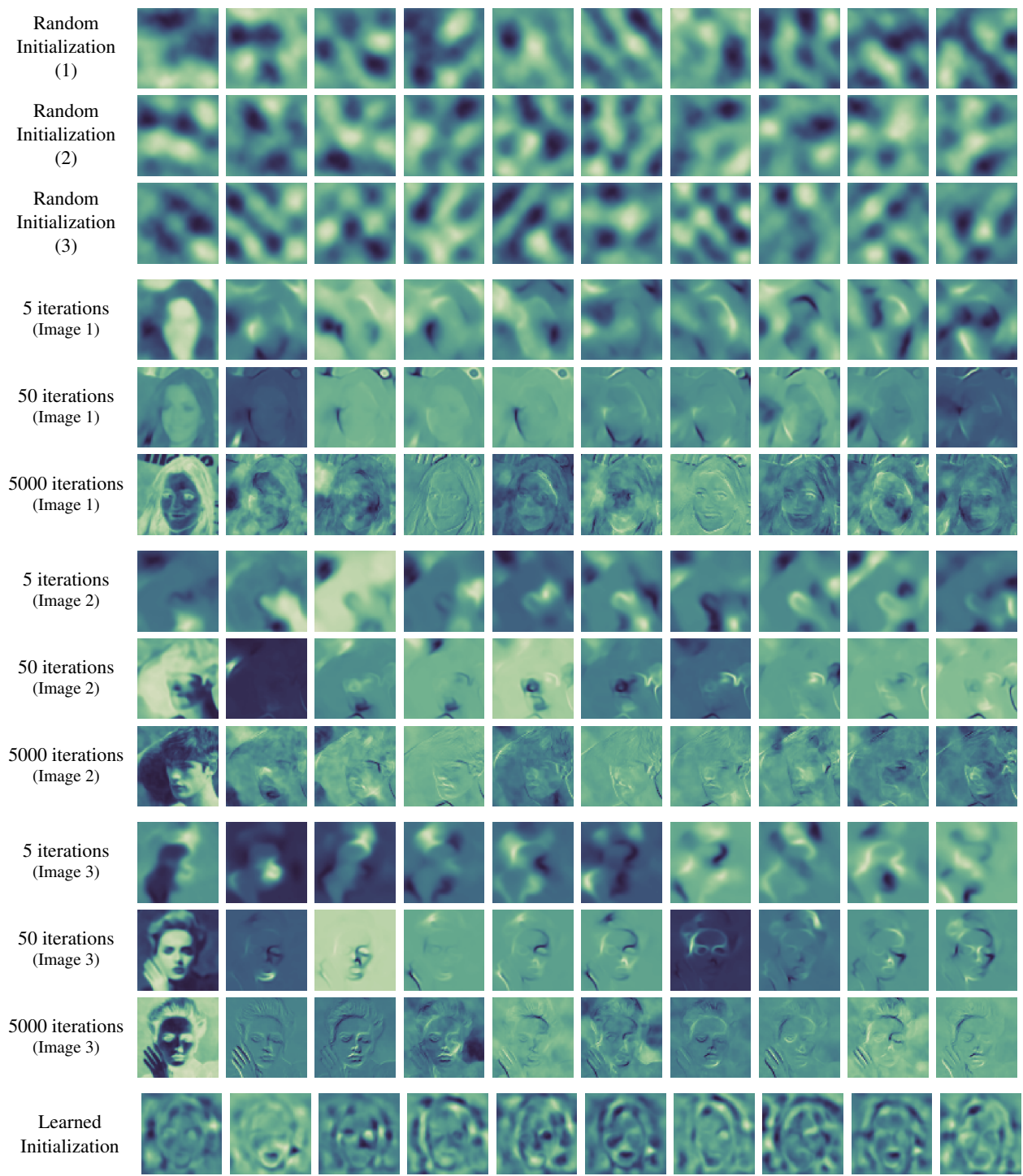

    \centering
    \parbox{\textboxsizeone}{\centering \small{Random Initialization \\(1)}}
    \foreach \i in {0,...,9} {
        \begin{subfigure}{\subfigsizeone}\includegraphics[width=\linewidth]{supp_eig/r1eigvec\i}
        \end{subfigure}
        \hspace{-6pt}
    }\\
    \parbox{\textboxsizeone}{\centering \small{Random Initialization \\(2)}}
    \foreach \i in {0,...,9} {
        \begin{subfigure}{\subfigsizeone}\includegraphics[width=\linewidth]{supp_eig/r2eigvec\i}
        \end{subfigure}
        \hspace{-6pt}
    }\\    
    \parbox{\textboxsizeone}{\centering \small{Random Initialization \\(3)}}
    \foreach \i in {0,...,9} {
        \begin{subfigure}{\subfigsizeone}\includegraphics[width=\linewidth]{supp_eig/r3eigvec\i}
        \end{subfigure}
        \hspace{-6pt}
    }\\    
    \vspace{4pt}
    \parbox{\textboxsizeone}{\centering \small{5 iterations} \\ \footnotesize{(Image 1)}}
    \foreach \i in {0,...,9} {
        \begin{subfigure}{\subfigsizeone}\includegraphics[width=\linewidth]{supp_eig/im1_it5eigvec\i}
        \end{subfigure}
        \hspace{-6pt}
    }\\
    \parbox{\textboxsizeone}{\centering \small{50 iterations} \\ \footnotesize{(Image 1)}}
    \foreach \i in {0,...,9} {
        \begin{subfigure}{\subfigsizeone}\includegraphics[width=\linewidth]{supp_eig/im1_it50eigvec\i}
        \end{subfigure}
        \hspace{-6pt}
    }\\    
    \parbox{\textboxsizeone}{\centering \small{5000 iterations} \\ \footnotesize{(Image 1)}}
    \foreach \i in {0,...,9} {
        \begin{subfigure}{\subfigsizeone}\includegraphics[width=\linewidth]{supp_eig/im1_it5000eigvec\i}
        \end{subfigure}
        \hspace{-6pt}
    }\\    
    \vspace{4pt}
    \parbox{\textboxsizeone}{\centering \small{5 iterations} \\ \footnotesize{(Image 2)}}
    \foreach \i in {0,...,9} {
        \begin{subfigure}{\subfigsizeone}\includegraphics[width=\linewidth]{supp_eig/im2_it5eigvec\i}
        \end{subfigure}
        \hspace{-6pt}
    }\\
    \parbox{\textboxsizeone}{\centering \small{50 iterations} \\ \footnotesize{(Image 2)}}
    \foreach \i in {0,...,9} {
        \begin{subfigure}{\subfigsizeone}\includegraphics[width=\linewidth]{supp_eig/im2_it50eigvec\i}
        \end{subfigure}
        \hspace{-6pt}
    }\\    
    \parbox{\textboxsizeone}{\centering \small{5000 iterations} \\ \footnotesize{(Image 2)}}
    \foreach \i in {0,...,9} {
        \begin{subfigure}{\subfigsizeone}\includegraphics[width=\linewidth]{supp_eig/im2_it5000eigvec\i}
        \end{subfigure}
        \hspace{-6pt}
    }\\ 
    \vspace{4pt}
    \parbox{\textboxsizeone}{\centering \small{5 iterations} \\ \footnotesize{(Image 3)}}
    \foreach \i in {0,...,9} {
        \begin{subfigure}{\subfigsizeone}\includegraphics[width=\linewidth]{supp_eig/im3_it5eigvec\i}
        \end{subfigure}
        \hspace{-6pt}
    }\\
    \parbox{\textboxsizeone}{\centering \small{50 iterations} \\ \footnotesize{(Image 3)}}
    \foreach \i in {0,...,9} {
        \begin{subfigure}{\subfigsizeone}\includegraphics[width=\linewidth]{supp_eig/im3_it50eigvec\i}
        \end{subfigure}
        \hspace{-6pt}
    }\\    
    \parbox{\textboxsizeone}{\centering \small{5000 iterations} \\ \footnotesize{(Image 3)}}
    \foreach \i in {0,...,9} {
        \begin{subfigure}{\subfigsizeone}\includegraphics[width=\linewidth]{supp_eig/im3_it5000eigvec\i}
        \end{subfigure}
        \hspace{-6pt}
    }\\  
    \vspace{4pt}
    \parbox{\textboxsizeone}{\centering \small{Learned Initialization}}
    \foreach \i in {0,...,9} {
        \begin{subfigure}{\subfigsizeone}\includegraphics[width=\linewidth]{figures/meta_64_30_eigvec\i}
        \end{subfigure}
        \hspace{-6pt}
    }
\caption{First ten eigenfunctions corresponding to the largest eigenvalues of the empirical NTK of SIREN~\cite{sitzmann2019siren} at initialization for different cases. The first three rows represent different realizations of random initialization. Rows 4-5-6 illustrate the eigenfunctions of a SIREN after meta-learning on a randomly chosen (single) training image from the CelebA dataset~\cite{liu2015faceattributes} with corresponding number of meta-learning updates on learnable parameters \cite{tancik2020meta}. Rows 7-8-9 and 10-11-12 are the outcomes of the same experiment for different choice of the training image used for meta-learning. The bottom row depicts the eigenfunctions of the learned initialization when we use $5,000$ training images from the CelebA dataset for meta-learning.}
\label{fig:eigvec_singletask}
\end{figure*}

\end{document}